%% file: Manuscript.tex
\documentclass[lettersize,journal]{IEEEtran}

\usepackage{amsmath,amsfonts}
\usepackage{hyperref}
\usepackage{tikz}
\usetikzlibrary{patterns}
\usepackage{array}
\usepackage{textcomp}
\usepackage{lscape}
\usepackage{stfloats}
\usepackage{url} 
\usepackage{pgfplots}
\usepackage{verbatim}
\usepackage{tikz}
\usepackage{subcaption}
\usepackage{graphicx}
\usepackage{amssymb}
\usepackage{multirow}
\usepackage{cuted}
\usepackage{caption}
\pgfplotsset{compat=1.18} 
\usepackage[linesnumbered,ruled,vlined]{algorithm2e}
\usepackage{rotating}
\usepackage{bbding}

\usepackage{placeins}
\usepackage{amsthm}
\newtheorem{theorem}{Theorem}    

\newtheorem{proposition}[theorem]{Proposition}

 \usepackage[dvipsnames]{xcolor}

\begin{document}
\let\WriteBookmarks\relax
\def\floatpagepagefraction{1}
\def\textpagefraction{.001}

\title{Indian Wedding System Optimization (IWSO): A Novel Socially Inspired Metaheuristic  with Operational Design and Analysis} 

\author{Deepika Saxena, Kishu Gupta, Jitendra Kumar, Jatinder Kumar, Sakshi Patni, Vinaytosh~Mishra,\\ Niharika Singh, and Ashutosh Kumar Singh
\thanks{Deepika Saxena is with Depatment of Computer Science and Engineering, University of Aizu, Aizuwakamatsu, Fukushima, Japan and Department of Computer Science, the VIZJA University, Warsaw, 01-043, Poland (email: 13deepikasaxena@gmail.com)}
\thanks{Kishu Gupta is with the Department of Computer Science and Engineering, National Sun Yat-sen
University, Kaohsiung, Taiwan and Department of Computer Science, the VIZJA University, Warsaw, 01-043, Poland
(email: kishuguptares@gmail.com)}
\thanks{Jitendra Kumar is with Department of Mathematics, Bioinformatics and Computer Applications, Maulana Azad National Institute of Technology Bhopal, India (email: jitendrakumar@ieee.org)}
\thanks{Jatinder Kumar is with Department of Computer Applications,
	National Institute of Technology, Kurukshetra, India, and Department of Electrical, Electronic and Computer Engineering, University of Pretoria, Hatfield 0028, Pretoria, South Africa (email: jatinder\_61900097@nitkkr.ac.in, jatinderkumar2851@gmail.com)}
\thanks{Sakshi Patni is with Department of Computer Applications, Panipat Institute of Engineering and Technology, Haryana, India (email: sakshichhabra555@gmail.com)}
\thanks{Vinaytosh Mishra is with Thumbay College of  Management and AI in Healthcare, Gulf Medical University, Ajman, United Arab Emirates (email: vinaytosh@gmail.com)}
\thanks{Niharika Singh is with Department of Computer Science, University of Helsinki, Finland (email: niharika.singh@helsinki.fi)}
\thanks{Ashutosh Kumar Singh is with Department of Computer Science and Engineering, Indian Institute of Information Technology Bhopal, India and Department of Computer Science, the VIZJA University, Warsaw, 01-043, Poland (email: ashutosh@iiitbhopal.ac.in)}

}

\maketitle
\begin{abstract}
This paper presents a novel population-based metaheuristic, \textit{Indian Wedding System Optimization} (IWSO), inspired by the socio-cultural dynamics of traditional Indian weddings. IWSO models the matchmaking process driven by collaboration among families, candidates, and matchmakers as a guided, selective search framework for solving complex optimization problems. The algorithm introduces two key innovations: (i) a \textit{matchmaker-guided influence strategy}, where elite solutions direct the evolution of weaker candidates, enhancing convergence without external parameters; and (ii) an \textit{adaptive elimination and reinitialization} mechanism that maintains diversity and prevents premature convergence by replacing underperforming individuals. IWSO employs a weighted multi-objective fitness function and analytically derived time and space complexity, benchmarked against existing optimization approaches such as Genetic Algorithm (GA), Partical Swarm Optimization (PSO), Differential Evolution (DE), Cuckoo Search (CS), etc. Extensive experiments on benchmark high-dimensional and multimodal test functions demonstrate  superior performance of IWSO in terms of convergence speed, solution quality, and robustness.
\end{abstract}

\begin{IEEEkeywords}
	Adaptive elimination, Convergence, Indian wedding, Metaheuristic optimization, Matchmaker influence
\end{IEEEkeywords}

\section{Introduction}\label{sec1}
Metaheuristic optimization algorithms have become indispensable for addressing complex, nonlinear, and high-dimensional problems across engineering, economics, and computational intelligence domains \cite{yang2013swarm,SEC1-IKEGUCHI2025102060}. They are particularly suited for scenarios where the objective function is non-differentiable, multimodal, noisy, or dynamically constrained \cite{xs2010nature,tan2021evolutionary}. Their strength lies in balancing global exploration with local exploitation, often guided by natural, physical, or social metaphors. Swarm intelligence-based algorithms, such as ant colony optimization (ACO), PSO, and artificial bee colony (ABC), exemplify this balance through decentralized cooperation, self-organization, and collective learning, enabling efficient navigation of complex search spaces \cite{mavrovouniotis2020ant,zhang2011evolutionary}.

Despite their success, most classical metaheuristics encounter persistent challenges, including premature convergence, loss of population diversity, and limited adaptability to dynamic environments. These limitations arise primarily from fixed search parameters and rigid population dynamics, which constrain their performance in large-scale or time-varying optimization problems. Recent research has therefore emphasized adaptive, hybrid, and socially inspired frameworks capable of maintaining diversity and learning dynamically from evolving populations.

In this context, we propose a novel population-based metaheuristic algorithm, termed the \textit{Indian Wedding System Optimization} (IWSO). The algorithm draws inspiration from the hierarchical, collaborative, and adaptive nature of traditional Indian weddings, where numerous stakeholders, families, partners, and vendors coordinate simultaneously under multiple, often conflicting objectives \cite{uberoi2000family,SEC2-ABDELHAFEZ2025102110}. The dynamic interplay of negotiation, distributed decision-making, and feedback-driven adaptation in such social systems provides a compelling metaphor for efficient global optimization. The IWSO algorithm models this socio-cooperative process to construct a robust, adaptive search framework. Its design incorporates three key mechanisms that collectively enable dynamic balance between exploration and exploitation:

\begin{itemize}
 \item \textbf{Decision propagation via matchmakers:} Elite solutions influence the population through structured information sharing, mirroring matchmakers that mediate preferences between stakeholders. This mechanism encourages cooperative learning and knowledge transfer, improving convergence stability and reducing the risk of premature convergence. 
 
    \item \textbf{Adaptive elimination:} Candidate solutions are evaluated using a joint fitness–diversity criterion. Underperforming individuals are selectively removed to preserve quality and prevent stagnation, analogous to prioritizing feasible and high-impact decisions in complex event coordination.
    
    \item \textbf{Guided reinitialization:} Eliminated candidates are replaced by new ones generated from elite solutions combined with controlled stochastic perturbations. This strategy enhances exploration, maintains diversity, and accelerates convergence toward the global optimum.

\end{itemize}

Together, these mechanisms establish a self-organizing and adaptive optimization process capable of dynamically responding to changing problem landscapes. By embedding social hierarchy, cooperative learning, and adaptive feedback into its core design, IWSO achieves sustained population diversity and robust convergence even in multimodal, constrained, and high-dimensional search spaces. Unlike conventional nature-inspired algorithms that rely solely on physical or biological analogies, IWSO introduces a socially grounded optimization paradigm, bridging human-inspired coordination and computational intelligence.
\subsection{Related Work}

Swarm intelligence and evolutionary algorithms have demonstrated strong capability in solving complex, nonlinear, and high-dimensional optimization problems through adaptive, population-based search mechanisms. Despite their effectiveness, many suffer from premature convergence, loss of diversity, and heavy parameter dependence, limiting scalability and adaptability in dynamic environments.

Genetic Algorithms (GA) were among the earliest evolutionary techniques, relying on biologically inspired operators such as selection, crossover, and mutation to iteratively refine solutions \cite{back1997handbook}. Although GA demonstrates strong global search ability, it commonly suffers from premature convergence in multimodal landscapes due to loss of population diversity \cite{man1996genetic}. Adaptive and hybrid variants alleviate these limitations through dynamic parameter tuning and local refinement strategies. Particle Swarm Optimization (PSO), inspired by collective swarm dynamics, offers rapid convergence and a simple update rule \cite{kennedy1995particle}. However, PSO tends to stagnate in high-dimensional or rugged search spaces as diversity collapses. Research efforts—including adaptive inertia control, quantum-behaved models, and fully informed update schemes—aim to mitigate this limitation \cite{shi1998modified}. The Lévy Flight Triangle Walk Dung Beetle Optimization (LTDBO) algorithm enhances the original DBO framework by integrating Logistic–cubic mapping, triangle-based foraging, and adaptive Lévy-flight mechanisms to improve exploration–exploitation balance and prevent entrapment in local optima \cite{fan2025novel}. Similarly, the Geometric Octal Zones Distance Estimation (GOZDE) algorithm introduces an eight-zone population structure and a median-based information-sharing mechanism to strengthen global exploration and adaptive guidance across the search space \cite{kuyu2022gozde}.

Differential Evolution (DE) effectively addresses continuous optimization using vector differentials to drive solution evolution \cite{storn1997differential}. While DE performs well across diverse domains, static mutation and crossover rates often hinder adaptability \cite{das2010differential}. In bio-inspired families, Ant Colony Optimization (ACO) and Artificial Bee Colony (ABC) exploit cooperative agent behaviors for solution construction \cite{dorigo1996ant, karaboga2007powerful}. ACO excels in discrete routing and scheduling tasks but incurs high computational cost due to pheromone updates \cite{dorigo2007ant}, whereas ABC maintains balanced exploration but may converge slowly in constrained or complex landscapes. More recent nature-inspired optimizers such as Grey Wolf Optimizer (GWO) \cite{mirjalili2014grey} and Whale Optimization Algorithm (WOA) \cite{mirjalili2016whale} model social hierarchies and hunting patterns to dynamically balance exploration and exploitation, yet they still lack strong adaptability in dynamic or time-varying environments \cite{alshammari2025nature}. Other classical techniques—including Cuckoo Search (CS) \cite{yang2009cuckoo}, Harmony Search (HS) \cite{geem2001new}, Simulated Annealing (SA) \cite{kirkpatrick1983optimization}, and Black Hole Optimization (BHO) \cite{hatamlou2013black}—improve solution diversity or local refinement but often suffer from slow convergence or heavy reliance on fixed control parameters.

Despite significant diversity in design principles, most classical metaheuristics share intrinsic limitations that hinder their effectiveness in large-scale, real-time, and dynamically evolving optimization environments. Their dependence on fixed or weakly adaptive control mechanisms—such as PSO’s inertia weight \cite{kennedy1995particle}, DE’s mutation and crossover rates \cite{storn1997differential}, GWO’s hierarchical coefficients \cite{mirjalili2014grey}, ACO’s pheromone parameters \cite{dorigo2007ant}, and CS’s static Lévy step sizes \cite{yang2009cuckoo}—limits search responsiveness under shifting conditions. Consequently, these algorithms frequently experience premature convergence, loss of diversity, limited scalability, and poor adaptability in non-stationary or uncertain environments \cite{leung1997degree, asim2020review}.

To address these shortcomings, hybrid, adaptive, and socially driven metaheuristics have emerged, emphasizing dynamic control, population restructuring, and context-aware behavioral updates. The proposed \textit{IWSO} algorithm advances this paradigm by modeling socio-cultural interaction to enable structured cooperation and adaptive search refinement. Through mechanisms such as hierarchical role assignment, adaptive elimination, guided reinitialization, and decision propagation, IWSO preserves population diversity, mitigates stagnation, and enables fluid transitions between exploration and exploitation. This adaptive socio-inspired framework allows \textit{IWSO} to remain robust in high-dimensional landscapes, responsive to changing objectives, and scalable for real-time complex optimization tasks.

\begin{table*}
\caption{Comparative Taxonomy of IWSO versus state-of-the-art metaheuristic optimization algorithms}
\tiny
\label{tab:metaheuristic_comparison}
\begin{tabular}{|p{1cm}|p{1.55cm}|p{1.75cm}|p{1.5cm}|p{1.5cm}|p{1.5cm}|p{1.25cm}|p{1.25cm}|p{1.25cm}|p{1.25cm}|}
\hline
\textbf{Algorithm} & \textbf{Exploration vs Exploitation} & \textbf{Key Limitations} & \textbf{Adaptability} & \textbf{Applicability} & \textbf{Pros} & \textbf{Technical Stability} & \textbf{Hyperpar- ameter Sensitivity} & \textbf{Scalability to High Dimensions} & \textbf{Dynamic/Real-Time Adaptation} \\
\hline
Genetic Algorithm (GA) & Strong exploration, moderate exploitation & Premature convergence, loss of diversity in later stages & Moderate; adaptive operators improve performance & Combinatorial and continuous optimization & Robust global search, flexible representation, easy hybridization & Moderate; dependent on selection/cros- sover/mutation & High; population size, mutation, crossover rates & Moderate; performance drops in very high dimensions & Low; adaptation requires additional mechanisms \\
\hline
Particle Swarm Optimization (PSO) & Moderate exploration, strong exploitation & Premature convergence in high-dimensional or rugged landscapes & Low; standard PSO struggles with dynamic changes & Continuous optimization, engineering design problems & Simple, fast convergence, low memory requirement & Moderate; stable in low-dimensional spaces & Moderate; inertia weight and acceleration coefficients critical & Moderate; variants improve performance in high dimensions & Low; standard PSO poorly adapts to dynamic changes \\
\hline
Differential Evolution (DE) & Strong exploration, moderate exploitation & Loss of diversity, limited local refinement & Moderate; adaptive variants improve performance & Continuous and real-parameter optimization & Simple implementation, effective global search, scalable & High; robust in continuous spaces & Moderate to high; mutation factor and crossover rate sensitive & High; good scalability & Low; standard DE lacks dynamic adaptation \\
\hline
Ant Colony Optimization (ACO) & Strong exploitation, moderate exploration & Computationally expensive, slow convergence in large problems & Low; pheromone adaptation limited & Combinatorial problems: TSP, routing, scheduling & Effective for discrete problems, positive feedback enhances solution quality & Moderate; sensitive to pheromone evaporation & High; pheromone update, evaporation, initial settings critical & Low; scalability issues in large-scale problems & Low; static pheromone updates limit adaptability \\
\hline
Artificial Bee Colony (ABC) & Balanced exploration and exploitation & Slow convergence, accuracy degradation in complex constrained spaces & Moderate; adaptive scout strategies & Continuous and combinatorial optimization & Simple, robust, population-based, memory-efficient & High; stable under moderate problem complexity & Moderate; colony size and limit parameter tuning required & Moderate; can degrade in very high dimensions & Moderate; can adapt via scout behavior \\
\hline
Grey Wolf Optimizer (GWO) & Balanced exploration and exploitation & Limited adaptability to dynamic/non-stationary environments & Low; static hierarchy limits flexibility & Continuous and engineering problems & Efficient, fast convergence, simple structure & Moderate; can stagnate in multimodal problems & Low to moderate; population size and alpha/beta/gamma roles & Moderate; performance declines in very high dimensions & Low; hierarchy is static \\
\hline
Whale Optimization Algorithm (WOA) & Strong exploitation, moderate exploration & Poor performance in dynamic problems, limited adaptability & Low & Continuous optimization, engineering design & Simple, few parameters, effective in static environments & Moderate; sensitive to initialization & Low; few hyperparameters & Moderate; high-dimensional performance limited & Low; lacks dynamic adaptation \\
\hline
Cuckoo Search (CS) & Strong exploration, weak exploitation & Slow convergence in constrained spaces, computational overhead in large-scale problems & Moderate; adaptive Lévy flights enhance performance & Continuous optimization, multimodal problems & Escapes local optima effectively, global search robustness & Moderate; stochastic nature affects repeatability & High; step size and probability of discovery critical & Moderate; global Lévy flights help high dimensions & Moderate; adaptive Lévy flights provide limited dynamic adaptation \\
\hline
Harmony Search (HS) & Moderate exploration, weak exploitation & Requires problem-specific tuning, lacks intensification & Low to moderate & Engineering design, combinatorial problems & Memory-based learning, flexible, easy hybridization & High; stable under moderate problem complexity & High; harmony memory size, pitch adjustment, HMCR sensitive & Low; struggles in very high dimensions & Low; lacks dynamic adaptation \\
\hline
Simulated Annealing (SA) & Strong exploration, weak exploitation & Slow convergence, single-solution approach, low parallelism & Low; temperature schedule static unless adaptive & Combinatorial and continuous optimization & Escapes local optima, simple implementation & High; stable if annealing schedule carefully designed & Moderate; initial temperature, cooling rate critical & Low; single-solution limits scalability & Low; adaptation requires complex temperature schemes \\
\hline
Black Hole Optimization (BHO) & Strong exploration, moderate exploitation & Slow convergence near optimum, lack of adaptive refinement & Low; no dynamic adaptation & Continuous and multimodal optimization & Simple, intuitive, population-based & Moderate; sensitive to black hole radius parameter & Moderate; population size, gravitational parameter sensitive & Moderate; performance may degrade in high dimensions & Low; lacks real-time adaptation \\
\hline

\textbf{Indian Wedding System Optimization (IWSO)} & Balanced and adaptive; exploration via candidate diversity, exploitation via matchmaker influence & Computational cost higher than simple heuristics; requires careful role and parameter design & High; adaptive elimination and guided reinitialization improve flexibility & High-dimensional, multimodal, constrained and dynamic optimization & Adaptive, socially structured, prevents stagnation, interpretable socio-cultural analogy & High; diversity maintained through elimination and reinitialization & Low to moderate; less sensitive due to adaptive mechanisms & High; scales well by preserving diversity and adaptive convergence & High; dynamic elimination and role adaptation support real-time scenarios \\
\hline
\end{tabular}%
\end{table*}
\normalsize

Table~\ref{tab:metaheuristic_comparison} summarizes a comparative taxonomy of IWSO against major state-of-the-art metaheuristic algorithms. The comparison covers key algorithmic dimensions such as exploration–exploitation balance, adaptability, scalability, and real-time responsiveness. Classical algorithms like GA, PSO, and DE show strong but isolated performance in specific dimensions, often constrained by static design and high parameter sensitivity. In contrast, IWSO exhibits balanced adaptability through its socially structured population dynamics, enabling sustained diversity, dynamic learning, and high technical stability across problem scales. The taxonomy clearly highlights IWSO’s superior capacity to operate efficiently under dynamic, multimodal, and high-dimensional optimization landscapes, emphasizing its potential as a robust alternative to conventional metaheuristics.

\subsection{Motivation for IWSO}
The motivation for the IWSO algorithm stems from the observation that Indian weddings naturally encapsulate a robust, adaptive, and multi-agent decision-making process. These ceremonies feature dynamic role assignment (e.g. matchmakers, parents, planners), priority-driven selection, reallocation of tasks, and negotiation-based resolution of conflicting objectives all of which mirror the demands of high-dimensional, constrained, and multi-objective optimization problems. Specifically, IWSO introduces two key innovations over traditional metaheuristics: \textit{Matchmaker-Guided Influence}  and \textit{Adaptive Elimination and Reinitialization}.  The concept of Matchmaker-guided influence is inspired by the real-world role of matchmakers in influencing and coordinating between families, IWSO incorporates a guided selection mechanism where elite agents help direct sub-optimal agents toward promising regions in the search space, enhancing convergence. Adaptive elimination and reinitialization facilitate dynamic elimination of underperforming agents (e.g., unsuccessful candidates) and reinitialize them using heuristic cues derived from successful agents and the problem landscape. This maintains diversity and prevents stagnation. By blending socio-cultural dynamics with algorithmic efficiency, IWSO fills a conceptual and performance-oriented gap in existing metaheuristics. It is particularly well-suited for scenarios requiring fast convergence, robustness to stagnation, and adaptable single-objective and multi-objective handling traits that are increasingly critical in real-world large-scale optimization tasks.

\subsection{Key Contributions}

The key contributions of this research paper are:

\begin{itemize}
	\item A novel metaphor-driven optimization framework "\textit{IWSO}" is proposed that maps the core components of Indian weddings (e.g., matchmakers, suitors, elimination rituals) to algorithmic operations for search, selection, and adaptation.
	
	\item Integration of a \textit{matchmaker-guided influence strategy}, where elite solutions guide weaker solutions, improving convergence and search space exploration without external control parameters.
	
	\item Introduction of an \textit{adaptive elimination and reinitialization} mechanism that prevents premature convergence and maintains population diversity by dynamically replacing poorly performing individuals.
	
	\item Comprehensive theoretical analysis and empirical evaluation demonstrating IWSO’s computational efficiency, scalability, and superiority on several benchmark functions; and compared with traditional state-of-the-art (SoTA) swarm-intelligence algorithms.
\end{itemize}

\subsection{Paper Outline}
The rest of the paper is organized as follows. Section \ref{iwsodes} provides an overview of the IWSO approach, along with i) computational intelligence analogy, ii) operational design, iii) algorithm and complexity analysis, and iv) convergence analysis of the proposed IWSO approach. Section \ref{performance evaluation} describes the step-by-step experimental details and results covering the i) efficiency analysis, ii) sensitivity analysis, and iii) comparative analysis, followed by a compelling conclusion in Section \ref{conclusion}.

\section{Indian Wedding System Optimization}\label{iwsodes}
\subsection{Computational Intelligence Analogy}

The IWSO algorithm is conceptually inspired by the structured social dynamics of Indian wedding matchmaking. In this analogy, each candidate solution represents a potential bride or groom in a multidimensional search space, where attributes correspond to compatibility criteria and fitness objectives. The optimization process mirrors the matchmaking lifecycle, balancing exploration (evaluating diverse prospects) and exploitation (refining the best options) through iterative evaluations and adaptive decision-making.

Key metaphorical components of the algorithm include: (i) \textit{Candidates}, modeled as solution vectors that evolve within defined bounds; (ii) \textit{Matchmaker Influence}, represented as a time-dependent heuristic that initially promotes broad exploration and progressively shifts toward intensified search around high-quality solutions; and (iii) \textit{Selection and Elimination}, where poorly performing or redundant individuals are probabilistically removed and replaced to maintain diversity and prevent early stagnation. Random perturbations emulate inherent human uncertainties in matchmaking preferences, while population mean–based reinitialization reflects cultural tendencies toward collective reasoning without losing individuality.

Through this abstraction, IWSO embeds social logic into its search dynamics, achieving a structured balance between global search diversity and local refinement. This cognitively intuitive mechanism supports efficient convergence toward optimal solutions while remaining resilient to premature convergence and suboptimal search bias.

\subsection{Operational Design}

As illustrated in Fig. \ref{fig:1}, the operational design of IWSO follows seven consecutive steps: initialization, fitness evaluation, and best candidate identification,   iterative refinement guided by  matchmaker influence, adaptive elimination and reinitialization, selection of best solution, and termination within a coherent algorithmic cycle aimed at converging to the most admissible solution.
\begin{figure*}[!htbp]
	\centering
	\includegraphics[width=0.9\linewidth]{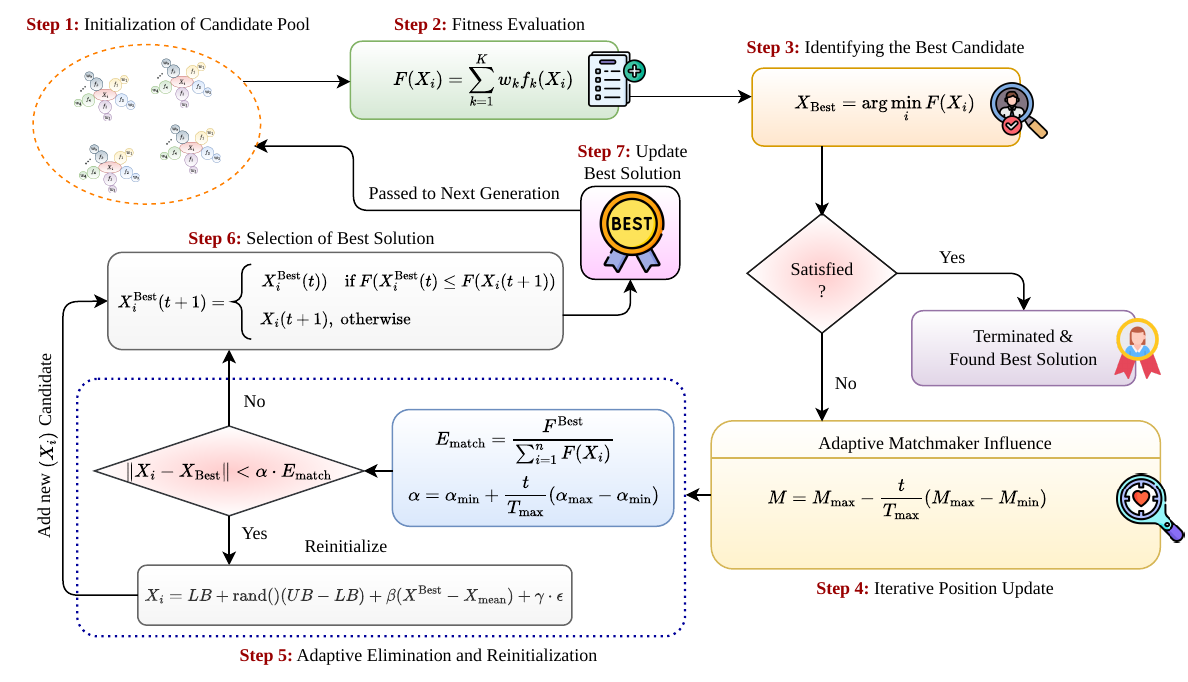}
	\caption{Schematic Overview of  Indian Wedding System Optimization Approach}
	\label{fig:1}
\end{figure*}
\subsubsection{Initialization}
A population of $n$ candidates is generated randomly within a defined search space using Eq. \eqref{e1}:
\begin{equation} \label{e1}
	X_i = LB + \text{rand()} \times (UB - LB), \quad i = 1, 2, \dots, n
\end{equation}

where, each candidate solution $X_i$ symbolizes an individual participant in the matchmaking process, such as a potential groom or bride. The search space is defined by the lower and upper bounds, denoted $LB$ and $UB$, respectively, which constrain the feasible region for candidate attributes. The term $\text{rand()}$ refers to a uniformly distributed random number in the interval $[0,1]$, ensuring stochastic diversity during the initialization and exploration phases of the algorithm. This step ensures diversity in the initial population, similar to how families consider multiple potential matches before shortlisting.

\subsubsection{Fitness Evaluation}
Each candidate's suitability is evaluated using a composite weighted sum approach using Eq. \eqref{e2}:
\begin{equation} \label{e2}
	F(X_i) = \sum_{k=1}^{K} w_k f_k(X_i), \quad \sum_{k=1}^{K} w_k = 1
\end{equation}
where, each candidate $X_i$ is evaluated across multiple criteria using a set of objective functions $f_k(X_i)$, where $f_k(X_i)$ quantifies the performance or suitability of the $i^{th}$ candidate with respect to the $k^{th}$ evaluation metric, such as education, family status, financial stability, or personality traits. The importance of each criterion is captured by a corresponding weight $w_k$, which reflects its relative significance in the overall decision-making process. These weights are normalized such that $\sum_{k=1}^{K} w_k = 1$, ensuring a balanced aggregation of all considered factors in the composite fitness function. This step models how matchmaking decisions consider multiple factors, such as compatibility, background, and financial stability.

\subsubsection{Identifying the best candidate}
The best candidate (optimal match) is selected using Eq. \eqref{e3}:
\begin{equation} \label{e3}
	X^{\text{Best}} = \arg\min_{i} F(X_i), \quad F^{\text{Best}} = \min_{i} F(X_i)
\end{equation}
where, $X^{\text{Best}}$ denotes the current best candidate in the population i.e., the one exhibiting the most favorable combination of attributes as determined by the fitness function. Consequently, $F^{\text{Best}} = \min_{i} F(X_i)$ represents the best fitness value achieved in all candidates, serving as a benchmark to evaluate the quality of other potential solutions within the search space.
This mimics the process where families shortlist the most promising match based on overall compatibility.

\subsubsection{Computation of matchmaker influence  (Iterative candidate refinement)}
To refine the choices, the position of each candidate is updated using an adaptive matchmaker factor $M$ using Eq. \eqref{e4}:
\begin{equation} \label{e4}
	M = M_{\max} - \frac{t}{T_{\max}} (M_{\max} - M_{\min})
\end{equation}

where, $M_{\max}$ and $M_{\min}$ represent the upper and lower bounds of the matchmaker's influence, respectively, effectively controlling the degree of stochastic exploration throughout the optimization process. The variable $t$ denotes the current iteration index, while $T_{\max}$ specifies the total number of allowed iterations. As the algorithm progresses, the influence of random perturbations is gradually reduced, enabling a smooth transition from exploration to exploitation.

As iterations progress, $M$ decreases, mimicking how families become more selective over time.
Specifically, the Adaptive Matchmaker Factor ($M$)  balances exploration and exploitation by adjusting randomness over iterations. The initial randomness ($M_{\max}$) ranges from 1.2 to 2.0, promoting broad search, while the final refinement ($M_{\min}$) ranges from 0.05 to 0.3, ensuring convergence. The iteration limit ($T_{\max}$) varies between 50 and 500, depending on problem complexity. For high-dimensional searches, higher $M_{\max}$ (1.8-2.0) and lower $M_{\min}$ (0.05-0.1) allow extensive exploration. Simpler problems use lower values (e.g., $M_{\max} = 1.2-1.5$). Proper tuning enhances optimization efficiency, balancing diverse candidate selection and precision matching, similar to refining marriage proposals in an Indian wedding system to find the best match. Each candidate updates its position using Eq. \eqref{e5}:
\begin{equation} \label{e5}
	X_i(t+1) = X_i(t) + r_1 (X^{\text{Best}} - X_i) + M  \epsilon
\end{equation}
where, \( r_1 \in [0,1] \) is a control parameter that governs the degree of attraction toward the best candidate, effectively balancing exploration and exploitation. The term \( \epsilon \in [-1,1] \) introduces random perturbations into the candidate's position update, enhancing stochastic behavior and helping the algorithm avoid local optima.
This represents how candidates refine their choices based on feedback from the matchmaking process.

\subsubsection{Adaptive Elimination and Reinitialization}
Candidates who are too similar to the best match but not optimal are eliminated using an expected match ($E_{\text{match}}$) threshold using Eq. \eqref{e6}:
\begin{equation} \label{e6}
	E_{\text{match}} = \frac{F^{\text{Best}}}{\sum\limits_{i=1}^{n} F(X_i)}
\end{equation}
and the elimination factor ($\alpha$) using Eq. \eqref{e7}:
\begin{equation} \label{e7}
	\alpha = \alpha_{\min} + \frac{t}{T_{\max}} (\alpha_{\max} - \alpha_{\min})
\end{equation}
If a candidate is too close to $X^{\text{Best}}$ without being optimal, it is reinitialized using Eq. \eqref{e8}:
\begin{equation} \label{e8}
	X_i = LB + \text{rand()}(UB - LB) + \beta (X^{\text{Best}} - X_{\text{mean}}) + \gamma  \epsilon
\end{equation}
where,  \( X_{\text{mean}} = \frac{1}{n} \sum\limits_{i=1}^{n} X_i \) denotes the population mean, which helps preserve diversity within the candidate pool during reinitialization. The parameter \( \beta \in [0.1, 0.5] \) introduces a controlled bias towards the global best solution, thereby enhancing the exploitation capability of the algorithm. Meanwhile, \( \gamma \in [0.2, 0.8] \) scales the stochastic component \( \epsilon \sim \mathcal{N}(0,1) \), introducing randomness that prevents premature convergence and encourages exploration of the solution space.
This process mimics how, in an Indian wedding system, families reject unsuitable matches and explore for better alternatives.
\subsubsection{Fitness Evaluation and Selection of the best solution}

The fitness values of the updated candidate solutions in the new population are computed using Eq.~\eqref{e3}. The selection of the best candidate at each step is performed according to Eq.~\eqref{selection}:

\begin{equation}
	X^{\text{Best}}_{i}(t+1) =
	\begin{cases} 
		X_i^{\text{Best}}(t), & \text{if } F(X_i^{\text{Best}}(t)) \leq F(X_i(t+1)) \\
		X_i(t+1), & \text{otherwise}
	\end{cases}
	\label{selection}
\end{equation}

At the end of each iteration, the global best solution across the entire population is updated using Eq. \eqref{e9}:
\begin{equation} \label{e9}
	X^{\text{Best}} = \arg\min_{i} F(X_i), 
	\quad F^{\text{Best}} = \min_{i} F(X_i)
\end{equation}

This process is repeated iteratively until the predefined maximum number of iterations is reached, ultimately yielding the final best-matching solution.

\subsubsection{Termination}

The process repeats until the maximum iteration count $T_{\max}$ is reached or a convergence criterion is met, with $X^{\mathrm{Best}}$ returned as the algorithm’s optimal solution.

\subsection{Algorithm and Complexity Analysis}
\begin{algorithm}[!htbp]
	\caption{Indian Wedding System Optimization} 
	\DontPrintSemicolon
	\KwIn{Population size $n$, bounds $[LB, UB]$, max iterations $T_{\max}$}
	\KwOut{Best solution $X^{\text{Best}}$ and fitness $F^{\text{Best}}$}
	
	\tcc{\bf Initialization:}
	Initialize random population $X_i = LB + \text{rand()} \times (UB - LB), \quad i=1, \dots, n$\
	\tcc{\textbf{Fitness evaluation:}}
	Compute fitness for each $X_i$ using weighted sum:  
	$F(X_i) = \sum_{k=1}^{K} w_k f_k(X_i), \quad \sum_{k=1}^{K} w_k = 1$\
	\\
	Identify the best solution and best fitness value:  
	$X^{\text{Best}} = \arg\min_{i} F(X_i), \quad F^{\text{Best}} = \min_{i} F(X_i)$ \
	
	\For{$t = 1$ to $T_{\max}$}{
		\tcc{\bf Matchmaker influenced-based candidate refinement:}
		$M = M_{\max} - \frac{t}{T_{\max}} (M_{\max} - M_{\min})$\
		\For{each candidate $X_i$}{
			$X_i(t+1) = X_i(t) + r_1 (X^{\text{Best}} - X_i) + M  \epsilon$
		}
		
		\tcc{\bf Adaptive elimination:}
		$E_{\text{match}} = \frac{F^{\text{Best}}}{\sum_{i=1}^{n} F(X_i)}$\
		$\alpha = \alpha_{\min} + \frac{t}{T_{\max}} (\alpha_{\max} - \alpha_{\min})$\
		\For{each $X_i$}{
			\If{$\| X_i - X^{\text{Best}} \| < \alpha  E_{\text{match}}$}{
				\tcc{\textbf{Reinitialization:}}
				$X_i = LB + \text{rand()}(UB - LB) + \beta (X^{\text{Best}} - X_{\text{mean}}) + \gamma  \epsilon$
			}
		}
		\tcc{\textbf{Fitness evaluation and selection of  best solution:}}
		
		Evaluation of fitness of new population using Eq. \eqref{e3} and selection of best candidate solution using Eq. \eqref{selection}
	}
	\KwRet{$X^{\text{Best}}$, $F^{\text{Best}}$}
\end{algorithm}

In essence, the operational design of the IWSO embeds sociocultural principles of the Indian wedding within a mathematically grounded evolutionary framework, enabling robust and interpretable optimization performance.

\subsubsection*{Complexity Analysis}

Let $n$ be the population size, $T_{\max}$ the maximum number of iterations, $D$ the problem dimensionality, and $K$ the number of objectives in the weighted fitness function. During initialization, generating $n$ individuals of dimension $D$ requires $\mathcal{O}(nD)$ time, computing their fitness values requires $\mathcal{O}(nK)$, and identifying the best solution takes $\mathcal{O}(n)$, resulting in an overall initialization complexity of $\mathcal{O}(nD + nK)$. 
Within the main loop, repeated for $T_{\max}$ iterations, each iteration involves updating positions ($\mathcal{O}(nD)$), evaluating fitness ($\mathcal{O}(nK)$), computing elimination metrics ($\mathcal{O}(n)$), checking reinitialization conditions and computing distances ($\mathcal{O}(nD)$), and possibly reinitializing candidates ($\mathcal{O}(nD)$), along with updating the best solution ($\mathcal{O}(n)$). The total complexity per iteration is thus $\mathcal{O}(nD + nK + n)$, leading to an overall loop complexity of $\mathcal{O}(T_{\max}  n(D + K))$.
Returning the best solution at the end requires constant time, i.e. $\mathcal{O}(1)$. Therefore, the overall time complexity of the IWSO algorithm is $\mathcal{O}(T_{\max}  n(D + K))$. In terms of space, storing the population takes $\mathcal{O}(nD)$, fitness values require $\mathcal{O}(n)$, and maintaining temporary variables and the best solution adds $\mathcal{O}(D)$, resulting in a total space complexity of $\mathcal{O}(nD)$. Hence, IWSO exhibits linear scalability with respect to the population size $n$, dimensionality $D$, and objective count $K$, making it computationally efficient for high-dimensional, multi-objective optimization problems.

\subsection{Convergence Analysis of IWSO}

The operational design of IWSO consists of seven consecutive steps: initialization (Eq.~\eqref{e1}), fitness evaluation (Eq.~\eqref{e2}), identification of the best candidate (Eq.~\eqref{e3}), iterative refinement guided by the matchmaker influence (Eqs.~\eqref{e4}-\eqref{e5}), the adaptive elimination and the reinitialization (Eqs.~\eqref{e6}-\eqref{e8}), selection of the best solution (Eq.~\eqref{selection}), and termination (Eq.~\eqref{e9}). We now formalize the convergence properties based on these operational steps.

\begin{theorem}[Convergence of IWSO]
\label{thm:iwsoconvergence}
Let $\mathbf{X}(t)$ = \{$X_1(t)$, $\dots$, $X_n(t)$\} denote the IWSO population at iteration $t$ in a compact search space $\Omega \subset \mathbb{R}^D$, initialized via Eq.~\eqref{e1}. Assume the fitness function $F:\Omega \to \mathbb{R}$ is measurable and bounded below by $F^\star$. Suppose the stochastic perturbations $\epsilon$ in the update (Eq.~\eqref{e5}) and reinitialization (Eq.~\eqref{e8}) admit a probability density $\rho(x)$ w.r.t.\ the Lebesgue measure $\lambda$ on $\Omega$ with $\rho(x) \ge \rho_{\min} > 0$ for all $x \in \Omega$. Further, assume the elitist selection rule (Eq.~\eqref{selection}) is applied at each iteration.

Then the global best fitness sequence $F^{\mathrm{Best}}(t)$ generated by Eq.~\eqref{e9} converges almost surely to the global minimum $F^\star$, i.e.,
$F^{\mathrm{Best}}(t) \xrightarrow{a.s.} F^\star \quad \text{as } t \to \infty,$
and for any $\varepsilon > 0$, at least one candidate enters the $\varepsilon$-neighborhood of the global optimum infinitely often with probability one.
\end{theorem}

\begin{proof}
Consider the measurable space $(\Omega^n, \mathcal{B}(\Omega^n), \mathbb{P})$, where $\Omega^n$ is the $n$-fold Cartesian product of the search space, $\mathcal{B}(\Omega^n)$ is the Borel $\sigma$-algebra, and $\mathbb{P}$ is the probability measure induced by the stochastic updates in Eq.~\eqref{e5} and Eq.~\eqref{e8}.

\paragraph{Step 1: Monotonicity via Elitist Selection} 
By Eq.~\eqref{selection}, the global best solution is updated only if a candidate improves upon the previous best. Hence, $F^{\mathrm{Best}}(t)$ is a non-increasing sequence bounded below by $F^\star$ and therefore convergent almost surely to some limit $F_\infty \ge F^\star$.

\paragraph{Step 2: Markov Chain Formulation} 
The IWSO updates (Eq.~\eqref{e5} and Eq.~\eqref{e8}) define a time-homogeneous Markov chain $\{\mathbf{X}(t)\}_{t\ge 0}$ with transition kernel
\[
P(\mathbf{X}(t+1) \in A \mid \mathbf{X}(t)) \quad \text{for all measurable } A \subseteq \Omega^n.
\]
Since the perturbation terms in Eq.~\eqref{e5} and reinitialization in Eq.~\eqref{e8} have strictly positive density $\rho_{\min}$ over the compact space $\Omega$, the chain is $\lambda$-irreducible: for any measurable set $B \subseteq \Omega^n$ with $\lambda(B)>0$, 
\[
P(\mathbf{X}(t+1) \in B \mid \mathbf{X}(t)) \ge (\rho_{\min})^n \lambda(B) > 0.
\]

\paragraph{Step 3: Hitting Global Minima}
Let $B_\varepsilon = \{x \in \Omega : F(x) \le F^\star + \varepsilon\}$. Irreducibility implies that at least one candidate enters $B_\varepsilon$ infinitely often. Once inside $B_\varepsilon$, Eq.~\eqref{selection} ensures $F^{\mathrm{Best}}(t+1) \le F^{\mathrm{Best}}(t)$, preserving or improving the global best.

\paragraph{Step 4: Limit Argument} 
Since $F^{\mathrm{Best}}(t)$ is non-increasing and $B_\varepsilon$ is visited infinitely often, it follows that 
\[
F_\infty = \lim_{t \to \infty} F^{\mathrm{Best}}(t) = F^\star \quad \text{almost surely}.
\]
\end{proof}

\begin{proposition}[Finite-Time $\varepsilon$-Optimal Probability Bound in IWSO]
\label{prop:iwsoprob}
Assume additionally that $F$ is $L$-Lipschitz continuous and stochastic reinitialization in Eq.~\eqref{e8} satisfies $\rho(x) \ge \rho_{\min} > 0$. Define 
\[
B_\varepsilon = \{x \in \Omega : F(x) \le F^\star + \varepsilon\}, \quad p_\varepsilon := \rho_{\min} \lambda(B_\varepsilon) > 0.
\]
Then, after $T$ iterations with $n$ candidates and sufficient stochastic perturbations via Eq.~\eqref{e5} or reinitializations via Eq.~\eqref{e8}, the probability that at least one candidate enters $B_\varepsilon$ is bounded below by
\[
\mathbb{P}\Big(\exists\, t\le T,\, i\in\{1,\dots,n\}: X_i(t) \in B_\varepsilon\Big) \ge 1 - (1-p_\varepsilon)^{nT}.
\]
\end{proposition}

\begin{proof}
At each iteration, each candidate’s update (Eq.~\eqref{e5}) or reinitialization (Eq.~\eqref{e8}) has probability at least $p_\varepsilon$ to fall within $B_\varepsilon$. Across $n$ candidates and $T$ iterations, the probability that no candidate enters $B_\varepsilon$ is at most $(1-p_\varepsilon)^{nT}$. Taking the complement yields the stated lower bound.
\end{proof}

\begin{itemize}
    \item Theorem~\ref{thm:iwsoconvergence} establishes almost-sure convergence of IWSO explicitly via its operational steps (Eqs.~\eqref{e1}, \eqref{e5}, \eqref{e8}, \eqref{selection}, \eqref{e9}).  
    \item Proposition~\ref{prop:iwsoprob} provides a finite-time probabilistic guarantee for reaching an $\varepsilon$-optimal solution, directly tied to the stochastic components of IWSO (Eq.~\eqref{e5} and Eq.~\eqref{e8}).  
    \item The combination of matchmaker influence, adaptive elimination, and reinitialization ensures irreducibility and positive probability of reaching global optima.
\end{itemize}

\section{Performance Evaluation} \label{performance evaluation}
This section describes the benchmark test problems employed to assess the performance of the proposed IWSO algorithm in solving optimization tasks and determining its effectiveness in identifying optimal solutions. Details of the benchmark functions, control parameters, and the stopping criteria utilized during the evaluation process are also provided. Furthermore, statistical analyses and sensitivity tests are conducted to validate the robustness of the algorithm. Further, its performance is compared against a set of well-established, diverse population-based and socio-inspired optimization algorithms, including GA \cite{back1997handbook}, ABC \cite{karaboga2007powerful}, PSO \cite{kennedy1995particle}, ACO \cite{dorigo1996ant}, DE \cite{storn1997differential}, WOA \cite{mirjalili2016whale}, HS \cite{geem2001new}, SA \cite{kirkpatrick1983optimization}, BHO \cite{hatamlou2013black}, and CS \cite{yang2009cuckoo} to demonstrate the efficacy of IWSO. 

\subsection{Experimental Details}
The experimental evaluations were conducted on a computing system equipped with an Intel\textsuperscript{\textregistered} Core\textsuperscript{TM} i7-8650U CPU operating at a base clock frequency of 2.20\,GHz. This computational system utilizes a 64-bit version of Ubuntu 25.04 LTS OS. The system is equipped with 16 GB of RAM. The software environment to execute the experimental work of the proposed algorithm was implemented using Python version 3.12.7. The simulation environment is configured with the following specifications:

\textit{Benchmark Test Functions}: To assess the relative performance of the proposed algorithm against established classical optimization techniques, an empirical study was conducted using a comprehensive set of 23 benchmark functions ($f_1$-$f_{23}$). The details of these benchmark functions are summarized in Table \ref{tab:1}. 
\setcounter{table}{1}
\begin{table}[!htbp]
	\centering
	\caption{Benchmark Test Functions}
	\label{tab:1}
	\resizebox{0.99\linewidth}{!}{
		\tiny
		\begin{tabular}{l|l|c|c|l|l|c}
			\hline
            \multirow{2}{*}{\textbf{F No.}} &\multirow{2}{*}{\textbf{Function Name}} &  \multicolumn{2}{c|}{\textbf{Type}} & \multicolumn{2}{c|}{\textbf{Search Space}} &  \multirow{2}{*}{\textbf{Dim}} \\
			 \cline{3-6}& &\textbf{Model} &\textbf{Category}  &\textbf{LB}  &\textbf{UB} & \\ \hline 
			$f_1$ & ackley & M&N &  32.768 &32.768 &30 \\ \hline
			$f_2$ & ackley\_2 & M&N & 32.768&32.768 &2 \\ \hline
			$f_3$ &booth &U&N &-10 &10&2 \\ \hline
			$f_4$ &cosine matrix &M&N &-10& 10&30 \\ \hline
			$f_5$ &dixon-price &U&N &-10 &10&30\\ \hline
			$f_6$ &foxholes &M&N &-65.536 &65.536&2 \\ \hline
			$f_7$ &griewank &M&N & -600 &600 &30\\ \hline
			$f_8$ &levy function &M&N &-10 &10&30\\ \hline
			$f_9$ & michalewicz &M&N &0 &$\pi$&10\\ \hline
			$f_{10}$ &multimodal sphere &M&S &-10 &10&30 \\ \hline
			$f_{11}$ &noisy quadratic &U&S & -10&10&30 \\ \hline
			$f_{12}$ &noisy sphere &U&S &-10 &10 &30\\ \hline
			$f_{13}$ &powell sum &U&S & -1&1&30\\ \hline
			$f_{14}$ &rastrigin &M&S &-5.12 &5.12 &30 \\ \hline
			$f_{15}$ &rastrigin\_2 &M&S &-5.12 &5.12&30 \\ \hline
			$f_{16}$ &rosenbrock &U&N & -5 &10 &30\\ \hline
			$f_{17}$ &salomon &M&S &-100&100 &30\\ \hline
			$f_{18}$ &schwefel &M&S & -500 &500 &30\\ \hline
			$f_{19}$ &sine wave &M&S & -$\pi$ &$\pi$ &30\\ \hline
			$f_{20}$ &sphere function &U&S &-5.12 &5.12 &30 \\ \hline
			$f_{21}$ &three hump camel &U&N &-5 &5 &2 \\ \hline
			$f_{22}$ &xin-she yang 4 &M&N &-10 &10 &30 \\ \hline
			$f_{23}$ &zakharov &U&N & -5 &10 &30 \\ \hline
	\end{tabular}
    }
	
	\footnotesize{U: Uni-modal, M: Multi-modal, S: Separable, N: Non-separable, Dim: dimension, LB: Lower Bound, UB: Upper-Bound}
\end{table}
\par
This test suite encompasses a diverse range of optimization problems with varying levels of complexity, categorized as Unimodal (U), Multimodal (M), Separable (S), and Non-Separable (N) functions. Unimodal functions are primarily employed to evaluate the exploitation capability of the algorithm, as they contain a single global optimum. In contrast, Multimodal functions are designed to assess the exploration ability, testing the algorithm’s effectiveness in avoiding premature convergence and escaping local optima. Additionally, the separability characteristic represents the dependency structure among variables; Separable functions consist of independent variables and are generally easier to optimize, whereas Non-Separable functions involve interdependent variables, posing a greater optimization challenge. A balanced selection of these benchmark functions ensures a rigorous and comprehensive evaluation of the algorithm’s performance in terms of convergence speed, exploration and exploitation balance, and the ability to locate global optimal solutions. 

\begin{table}[!htbp]
	\centering
	\caption{Control Parameters of IWSO and Existing Models}
	\label{tab:2}
		\resizebox{0.49\textwidth}{!}{
			\small
	\begin{tabular}{p{1.5cm}|p{7cm}}\hline
		\textbf{Algorithm} &\textbf{Parameter and Value} \\ \hline
		GA \cite{back1997handbook}	&	pop\_size = 30, crossover\_rate = 0.8, mutation\_rate = 0.1, generations = 50	\\	\hline
		ABC \cite{karaboga2007powerful}	&	colony\_size=30,  generations = 50		\\	\hline
		PSO	\cite{kennedy1995particle}&	swarm\_size = 30,  velocity\_max = 2, generations = 50, w = 0.5       \# inertia weight, c1 = 1.5      \# cognitive component, c2 = 1.5      \# social component	\\	\hline
		ACO	\cite{dorigo1996ant}&	num\_ants = 30, generations = 50,  evaporation\_rate= 0.5  \# evaporation rate of pheromone, Q = 100            \# learning rate  alpha = 1.0, beta=2.0	\\	\hline
		WOA \cite{mirjalili2016whale}	&	encircling prey (p $<$ 0.5), pop\_size = 30,  generations = 50	\\	\hline
		CS \cite{yang2009cuckoo}	&	n\_nests = 30, pa = 0.25  \# discovery rate\\	\hline
		SA \cite{kirkpatrick1983optimization}	&	initial\_temp = 1000, final\_temp = 1e-3, alpha = 0.95 \# cooling rate\\	\hline
		BHO \cite{hatamlou2013black}	&	num\_stars = 30, generations = 50	\\	\hline
		HS \cite{geem2001new}	&	Iterations = 50, agents=30, step\_size = 0.1\\	\hline
		DE \cite{storn1997differential}	&	pop\_size = 30, f = 0.5  \# mutation factor, cr = 0.9,      \# crossover probability, generations = 50	\\	\hline
		\textbf{IWSO} (Proposed) & pop\_size = 30, matchmaker factor ($M_{Max}$) = [1.2-2.0], matchmaker factor ($M_{Min}$) = [0.05-0.3], elimination factor ($\alpha$) = [0.5, 1.5], $T_{\max}$ = 50 \\ \hline
	\end{tabular}}   
\end{table}
\textit{Control Parameters}: 
The control parameter settings utilized for the proposed IWSO algorithm, as well as for the comparative algorithms, are detailed in Table \ref{tab:2}. All algorithms, including IWSO, were executed on the set of benchmark functions $f_1$-$f_{23}$ with search space for all variables bounded within the range as defined in Table \ref{tab:1}, across 30 independent runs to ensure statistical significance. Each run was conducted with a iterations ($T_{\max}$=50) and a population size (n)=30.

\textit{Stopping Criteria}: The algorithm is considered to have converged or terminated when any of the following criteria are satisfied:
\begin{itemize}
	\item The algorithm terminates when either the maximum number of iterations or the maximum allowed number of function evaluations is reached.
	\item The algorithm halts if no substantial improvement in solution quality is detected over a predefined number of successive iterations.
	\item The algorithm stops immediately upon discovering an optimal or sufficiently satisfactory solution that meets the predefined criteria.
\end{itemize}
\subsection{Results}
The optimization results are thoroughly analyzed to perform a comprehensive \textit{efficiency analysis} (Section \ref{pfa}) and \textit{sensitivity analysis} (Section~\ref{sen}) of the proposed IWSO algorithm across a diverse set of benchmark test functions. 
\subsubsection{Efficiency Analysis}\label{pfa}
 
Fig.~\ref{fig:cp} depicts the effect of the \textit{expected match} ($E_{Match}$) threshold, a key control parameter in the proposed IWSO algorithm that governs adaptive elimination and reinitialization to balance between exploration and exploitation. The $E_{Match}$ guides candidate solutions toward elite individuals over 100 iterations across 23 benchmark functions. Maintaining a consistent $E_{Match}$ in early iterations promotes broad stochastic exploration, preventing premature convergence. As iterations progress, the gradual decay of $E_{Match}$ enhances local refinement by amplifying the influence of the current best solution. This dynamic modulation accelerates convergence while preserving solution diversity during the initial stages. Moreover, by removing low-performing or redundant individuals based on relative fitness and similarity to the best solution, the mechanism sustains diversity and mitigates the risk of entrapment in local minima. In effect, it emulates a selection pressure that favors exploration in early stages and exploitation as convergence proceeds, enabling the IWSO algorithm to adapt its search behavior dynamically and achieve improved optimization performance across multi-modal landscapes.

\input{Ematch}

 Fig.~\ref{fig:mp} presents a comprehensive analysis of the IWSO algorithm's convergence and divergence behavior, emphasizing its ability to strike an effective balance between \textit{diversification} (global exploration) and \textit{intensification} (local exploitation), a crucial aspect for avoiding premature convergence. The convergence and divergence profiles over benchmark functions $f_1$ to $f_{23}$ are illustrated in Figs.~\ref{fig:mp}(a) to Fig.~\ref{fig:mp}(w), respectively. These functions represent diverse optimization landscapes, including uni-modal/multi-modal and separable/non-separable scenarios, offering rigorous testing conditions.

\input{con-div11}
 \setcounter{figure}{3}
 \input{con-div22}

The best fitness values (y-axis) over 50 epochs in an increasing manner, reveal that IWSO initiates with wide exploration, enabling the algorithm to escape shallow local optima. As the iterations progress, the adaptive matchmaker-guided influence gradually intensifies the search around elite solutions, leading to refined exploitation. This is evident from the steady and the smooth convergence curves across all test cases. The  robustness of the algorithm stems from its dual adaptation mechanisms: matchmaker-based guidance and elimination-driven diversity preservation, which together facilitate reliable navigation of rugged fitness landscapes. These trajectories confirm the ability of IWSO to maintain a dynamic search strategy that adapts to the landscape structure, enabling fast convergence to high-quality global or near-global solutions.

Further, the analysis illustrated in Fig. \ref{fig:violin} offers an in-depth exploration of the IWSO algorithm's performance across a diverse range of optimization landscapes. This examination encompasses various models, including both uni-modal and multi-modal scenarios, as well as distinct categories like separable and non-separable functions. Our rigorous testing was conducted over 30 runs, ensuring the reliability of the results. In Fig. \ref{fig:violina}, we present the IWSO algorithm's impressive performance in uni-modal scenarios, featuring separable functions ($f_{11}$, $f_{20}$) and non-separable functions ($f_3$, $f_{21}$). Conversely, Fig. \ref{fig:violinb} demonstrates the algorithm's superior effectiveness in multi-modal contexts, showcasing separable functions ($f_{14}$, $f_{16}$) alongside non-separable functions ($f_8$, $f_{22}$). The outcomes clearly highlight the remarkable capability of the proposed approach to efficiently pinpoint the elite candidate with minimal variation, even in the case of challenging complex non-separable functions across both uni-modal and multi-modal scenarios. This evidence emphasizes the strength of the IWSO algorithm as a leading solution in the field of optimization. 
\begin{figure*}[!htbp]
	\centering
	\begin{subfigure}{0.49\textwidth}
		\centering
		\includegraphics[width=0.99\linewidth]{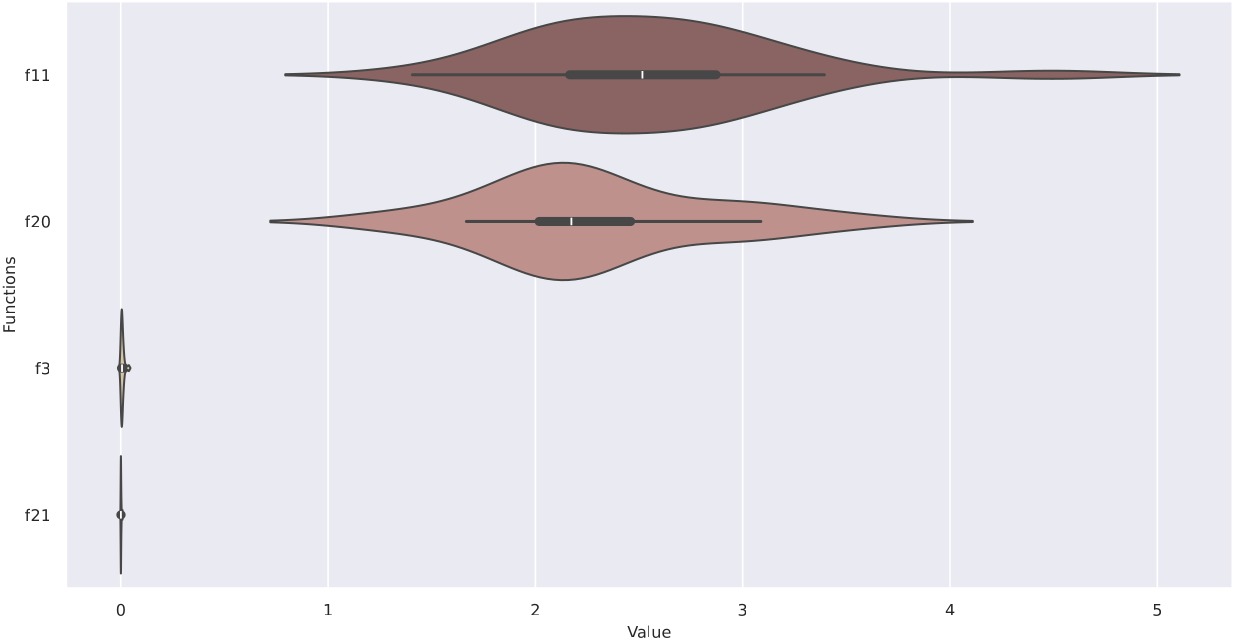}
		\caption{Uni-modal functions}
		\label{fig:violina}
	\end{subfigure}
	\hfill
	\begin{subfigure}{0.49\textwidth}
		\centering
		\includegraphics[width=0.99\linewidth]{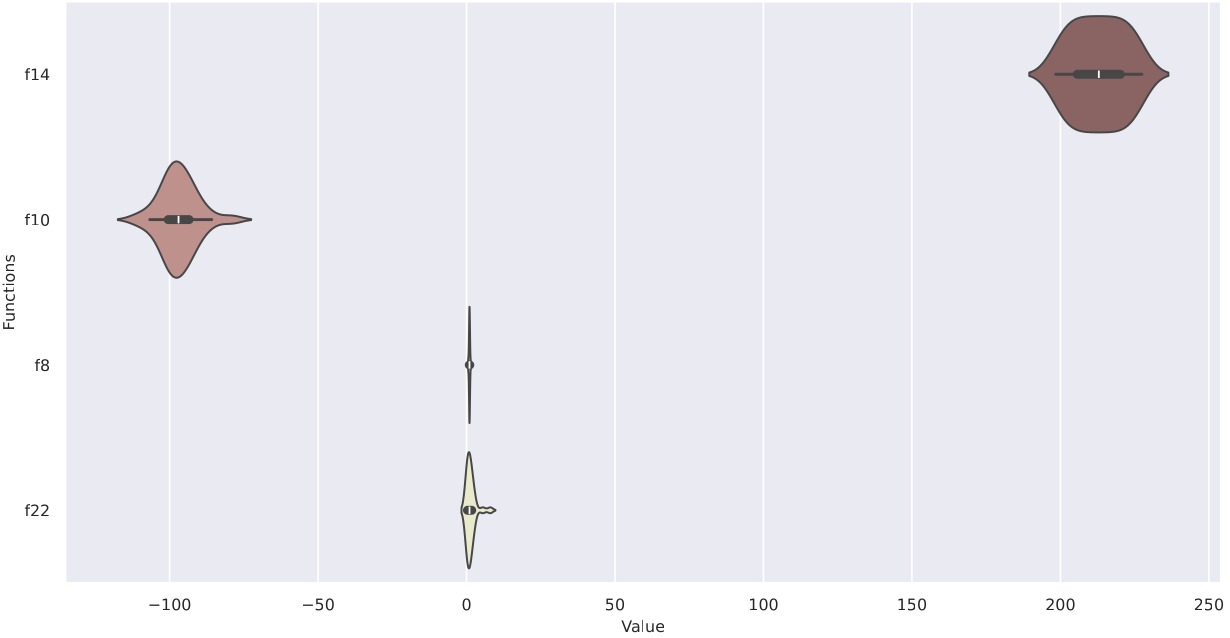}
		\caption{Multi-modal functions}
		\label{fig:violinb}
	\end{subfigure}
	\caption{Distribution of best fitness values achieved by IWSO across uni-modal and multi-modal separable and non-separable benchmark functions highlighting its optimization consistency and robustness} \label{fig:violin}
\end{figure*}

\subsubsection{Sensitivity Analysis}\label{sen}
Sensitivity analysis is conducted to investigate the impact of critical control parameters on the performance of metaheuristic optimization algorithms. In this process, one parameter is varied within a predefined range while maintaining the remaining parameters at their baseline values, thereby, isolating its effect on the algorithm’s behavior to test the robustness and adaptability of the algorithm to different parameter settings. The key parameters are analyzed for population size ($n$) with specific control parameters of the IWSO algorithm. These parameters include the maximum number of iterations ($T_{Max}$) reflecting the problem complexity and the  Matchmaker ($M$) factor to analyze the balancing of exploration-exploitation.  

Table \ref{tab:sena1} analyzes the sensitivity of the proposed IWSO algorithm concerning the iteration parameter ($T_{Max}$) for different values: 125, 250, 375, and 500 to showcase the performance of proposed algorithm in increasing order of problem complexity. These values are used to optimize benchmark test functions $f_1$-$f_{23}$. It is evident that increasing the number of iterations enhances the IWSO's ability to explore the search space with increasing problem complexity, leading to improved algorithm performance and reduced objective function values. 
\begin{table}[!htbp]
	\centering
	\caption{IWSO sensitivity analysis with respect to the control parameter $T_{Max}$ highlights its critical impact on convergence stability and overall optimization performance.}
	\label{tab:sena1}
	\resizebox{0.99\linewidth}{!}{
		\tiny
		\begin{tabular}{l|l|l|l|l}
			\hline
			\multirow{2}{*}{\textbf{Function}} & \multicolumn{4}{c}{\textbf{\# Problem Complexity ($T_{Max}$)}} \\
			\cline{2-5} &\textbf{125} &\textbf{250} &\textbf{375} &\textbf{500} \\ \hline 
			$f_1$ 	&	1.6792	&	1.0772	&1.65685 &	0.9349	\\ \hline
			$f_2$  	&	0.0014	&	0.0095	&0.05793 &	0.0084	\\ \hline
			$f_3$  	&	0.0000	&	0.0001	&0.00074 &	0.0000	\\ \hline
			$f_4$  	&	-853.4504	&	-876.9334	&-861.48668 &	-879.1763	\\ \hline
			$f_5$ 	&	9.3279	&	6.9405	&7.74318 &	4.2323	\\ \hline
			$f_6$  	&	0.998	&	0.998	&1.03113 &	0.998	\\ \hline
			$f_7$ 	&	0.1423	&	0.05	&0.17468 &	0.0593	\\ \hline
			$f_8$  	&	0.2207	&	0.1377	&0.24133 &	0.1233	\\ \hline
			$f_9$  	&	-5.9346	&	-6.2669	&-5.407202 &	-6.3729	\\ \hline
			$f_{10}$  	&	-115.1041	&	-117.0804	&-111.91533 &	-119.4707	\\ \hline
			$f_{11}$  	&	0.2827	&	0.2802	&0.60681 &	0.2754	\\ \hline
			$f_{12}$  	&	0.4569	&	0.2785	&0.52602 &	0.1928	\\ \hline
			$f_{13}$ 	&	0.0013	&	0.0022	&0.00458 &	0.0006	\\ \hline
			$f_{14}$  	&	95.0978	&	65.5245	&120.68067 &	47.0571	\\ \hline
			$f_{15}$  	&	98.1542	&	89.0611	&116.81714 &	68.5534	\\ \hline
			$f_{16}$ 	&	67.677	&	55.408	&66.14068 &	44.8396	\\ \hline
			$f_{17}$ 	&	0.2004	&	0.1999	&0.24075 &	0.1148	\\ \hline
			$f_{18}$ 	&	0.3618	&	0.2227	&0.36662 &	0.0941	\\ \hline
			$f_{19}$  	&	-29.8156	&	-29.8872	&-29.84187 &	-29.917	\\ \hline
			$f_{20}$  	&	0.4068	&	0.1773	&0.35729 &	0.1979	\\ \hline
			$f_{21}$ 	&	0.0000	&	0.0000	&9.82059 &	0.0000	\\ \hline
			$f_{22}$ 	&	0.0116	&	0.0632	&1.86202 &	0.1717	\\ \hline
			$f_{23}$ 	&	1.7906	&	1.9139	&3.15767 &	1.2499	\\ \hline
		\end{tabular}
	}
\end{table}


To assess the sensitivity of the proposed IWSO algorithm, four experimental cases as summarized in Table~\ref{tab:sena-cs}, are designed. These case studies systematically evaluate the algorithm’s performance by varying the degree of divergence (exploration) and convergence (exploitation) across different problem complexities, ranging from low-dimensional to high-dimensional search spaces. 
\begin{table}[!htbp]
	\centering
	\caption{Case studies for sensitivity analysis}
	\label{tab:sena-cs}
	\resizebox{0.99\linewidth}{!}{
		\tiny
		\begin{tabular}{l|l|l|l|l|l}
			\hline
			Case & $M_{max}$ &$M_{min}$ &Divergence  &Convergence &Search Space \\ \hline 
			C1 	&	1.2	& 0.3		& Low & Low	& \multirow{2}{*}{Low} 	\\ 
			\cline{1-5}
			C2 	&	1.4	& 0.2		& Moderate	 & Moderate	&   	\\ \hline
			C3 	&	1.8	& 0.05		& Moderate	 & Moderate	& \multirow{2}{*}{Complex} 	\\ 
			\cline{1-5} C4 	&	2.0	& 1.0		& High	 & High	&  	\\ \hline
	\end{tabular}}
\end{table}

Table \ref{tab:sena2} explains the sensitivity of the IWSO algorithm concerning the impact of the Matchmaker Factor to optimize benchmark test functions $f_1$-$f_{23}$. For simpler problems, the maximum Matchmaker Factor ($M_{\max}$) ranges from 1.2 to 2.0 to promote a broad search. Conversely, the minimum Matchmaker Factor ($M_{\min}$) ranges from 0.05 to 0.3 to ensure convergence. In the case of complex and high-dimensional searches, $M_{\max}$ is set between 1.8 and 2.0, while $M_{\min}$ is restricted to the range of 0.05 to 0.1 to facilitate extensive exploration. The experimental findings demonstrate that the IWSO consistently exhibits a strong capacity to converge towards high-quality solutions under diverse conditions considering different scenarios. This behavior highlights the robustness of its exploitation mechanism and validates its efficiency and adaptability across varying optimization scenarios. 
\begin{table}[!htbp]
	\centering
	\caption{IWSO sensitivity analysis with respect to parameter $M$ shows its key influence on convergence and solution quality.}
	\label{tab:sena2}
	\resizebox{0.99\linewidth}{!}{
		\tiny
		\begin{tabular}{l|l|l|l|l}
			\hline
			\multirow{2}{*}{\textbf{Function}} & \multicolumn{4}{c}{\textbf{Matchmaker Influence}}\\ \cline{2-5} 
			& \textbf{C1} &\textbf{C2} &\textbf{C3} &\textbf{C4} \\ \hline 
			$f_1$ 	&	2.6651	&3.04206 &	1.9823	&	3.8327 	\\ \hline
			$f_2$ 	&	0.0168	&0.25497 &	0.0104	&	0.0879	\\ \hline
			$f_3$ 	&	0.0002	&0.00634 &	0.0004	&	0.0003	\\ \hline
			$f_4$ 	&	-696.067	&-545.51653 &	-798.1557	&	-296.6815	\\ \hline
			$f_5$ 	&	30.6302	&53.87077 &	17.0752	&	271.3871	\\ \hline
			$f_6$ 	&	0.998	&3.54859 &	0.998	&	0.998	\\ \hline
			$f_7$ 	&	0.3171	&0.78622 &	0.2904	&	0.6696	\\ \hline
			$f_8$ 	&	0.8757	&0.99982 &	0.4603	&	2.1532	\\ \hline
			$f_9$ 	&	-5.8532	&-4.73421 &	-5.7026	&	-5.6194	\\ \hline
			$f_{10}$ 	&	-115.6974	&-96.76242 &	-110.6257	&	-97.866	\\ \hline
			$f_{11}$ 	&	1.779	&2.26606 &	1.1848	&	8.0909	\\ \hline
			$f_{12}$ 	&	1.8122	&2.63687 &	1.0041	&	10.1168	\\ \hline
			$f_{13}$ 	&	0.0236	&0.057228 &	0.0102	&	0.4507	\\ \hline
			$f_{14}$ 	&	156.1949	&201.14975 &	120.1052	&	173.8159	\\ \hline
			$f_{15}$ 	&	152.2805	&192.8358 &	134.5871	&	179.0121	\\ \hline
			$f_{16}$ 	&	220.5026	&304.70993 &	149.8461	&	1291.9722	\\ \hline
			$f_{17}$ 	&	0.2015	&0.38902 &	0.2816	&	0.4277	\\ \hline
			$f_{18}$ 	&	0.635	&2.99179 &	0.2598	&	2.2016	\\ \hline
			$f_{19}$ 	&	-29.3761	&-29.08378 &	-29.5122	&	-26.3885	\\ \hline
			$f_{20}$ 	&	1.8283	&2.02968 &	0.8344	&	7.8241	\\ \hline
			$f_{21}$ 	&	0	&0.00155 &	0	&	0.0002	\\ \hline
			$f_{22}$ 	&	0.1294	&1.64003 &	0.102	&	0.475	\\ \hline
			$f_{23}$ 	&	6.4863	&18.04481&	3.1425	&	12.6595	\\ \hline
		\end{tabular}
	}
\end{table}
\subsection{Comparative Analysis} \label{pfe}
This section presents a comprehensive comparative evaluation of the proposed IWSO algorithm against several state-of-the-art metaheuristic optimization techniques.

The lower standard deviations obtained by the IWSO algorithm, as shown in Fig. \ref{fig:functions_1_23}, demonstrate superior stability compared to widely used benchmark peer metaheuristic algorithms, including the GA~\cite{back1997handbook},  ABC~\cite{karaboga2007powerful},  PSO~\cite{kennedy1995particle}, ACO~\cite{dorigo1996ant}, DE~\cite{storn1997differential}, BHO~\cite{hatamlou2013black}, WOA~\cite{mirjalili2016whale}, HS~\cite{geem2001new}, SA~\cite{kirkpatrick1983optimization}, and CS~\cite{yang2009cuckoo}. IWSO achieves a significant standard deviation on 18 out of 23 functions, highlighting its capability to reliably identify optimal values. This performance results from its socio-inspired mechanisms, such as the matchmaker influence and elimination factor, which balance exploration and exploitation more effectively than other methods. Notably, IWSO’s near-zero deviations on functions $f_2$, $f_3$, $f_4$, $f_5$, $f_7$, $f_8$, $f_{11}$, $f_{12}$, $f_{13}$, $f_{16}$, $f_{17}$, $f_{18}$, $f_{19}$, $f_{20}$, $f_{21}$, $f_{22}$, and $f_{23}$ suggest a high level of precision in avoiding local optima, a common challenge for optimization algorithms. 

Fig. \ref{fig:com} depict a comparative performance analysis of prominent metaheuristic algorithms GA~\cite{back1997handbook}, DE~\cite{storn1997differential}, and the proposed IWSO on functions including $f_1$, $f_3$, $f_4$, $f_7$, $f_8$, $f_{17}$, $f_{21}$, and $f_{23}$, respectively. The experiments are conducted by varying the population size across the set \{10, 20, 30, 40, 50\} and evaluating the algorithms under four different iteration counts: 25, 50, 75, and 100. These three-dimensional surface plots provide a dynamic visualization of how each algorithm's fitness value evolves with respect to increasing population size and iteration count. The fitness values represent the best solutions obtained in each configuration, effectively reflecting the convergence characteristics and optimization capabilities of the algorithms. Notably, the plots demonstrate that the proposed IWSO consistently guides the search process towards lower prediction errors by efficiently navigating the search space and enhancing both exploration and exploitation. The observed performance confirms IWSO’s superior adaptability and precision in minimizing the objective function values under varying parameter settings, thereby establishing its effectiveness in building a highly accurate model.

\input{SD}

\input{res2}

\subsection{Computational Complexity Analysis}\label{com}
The \textit{computational complexity} of the proposed IWSO algorithm and State-of-the-Art metaheuristics is evaluated under worst-case, average-case, and best-case scenarios for single-objective optimization problems. Let $n$ be the population size, $D$ the problem dimensionality, and $T_{\max}$ the maximum number of iterations. Table~\ref{summary_prediction} presents a comparative analysis of asymptotic time and space complexities, offering a comprehensive assessment of algorithmic efficiency and scalability.
\begin{table}[!htbp]
	\centering
	\caption{Comparison of time and space complexity between IWSO and state-of-the-art metaheuristic algorithms.}
	\label{summary_prediction}
	\resizebox{0.48\textwidth}{!}{
		
		\begin{tabular}{l|c|c|c|c}
			\hline
			\textbf{Algorithm} & $\mathcal{O}$(Worst) & $\Theta$(Average) & $\Omega$(Best) & \textbf{Space Complexity} \\
			\hline
			GA  & $\mathcal{O}(T_{\max}  nD)$ & $\Theta(T_{\max}  nD)$ & $\Omega(nD)$ & $\mathcal{O}(nD)$ \\
			ABC  & $\mathcal{O}(T_{\max}n  D  )$ & $\Theta(n  D  T_{\max})$ & $\Omega(nD)$ & $\mathcal{O}(n  D)$ \\
			PSO  & $\mathcal{O}(T_{\max}  nD)$ & $\Theta(T_{\max}  nD)$ & $\Omega(nD)$ & $\mathcal{O}(nD)$ \\
			ACO  & $\mathcal{O}(T_{\max}  n^2D)$ & $\Theta(T_{\max}  n\log n  D)$ & $\Omega(nD)$ & $\mathcal{O}(n^2 + nD)$ \\
			DE  & $\mathcal{O}(T_{\max}  nD)$ & $\Theta(T_{\max}  nD)$ & $\Omega(nD)$ & $\mathcal{O}(nD)$ \\
			CS  & $\mathcal{O}(T_{\max}  nD)$ & $\Theta(T_{\max}  nD)$ & $\Omega(nD)$ & $\mathcal{O}(nD)$ \\
			BHO  & $\mathcal{O}(T_{\max}  nD)$ & $\Theta(T_{\max}  nD)$ & $\Omega(nD)$ & $\mathcal{O}(nD)$ \\
			SA  & $\mathcal{O}(T_{\max}  D)$ & $\Theta(T_{\max}  D)$ & $\Omega(D)$ & $\mathcal{O}(D)$ \\
			HS  & $\mathcal{O}(T_{\max}  nD)$ & $\Theta(T_{\max}  nD)$ & $\Omega(nD)$ & $\mathcal{O}(nD)$ \\
			WOA  & $\mathcal{O}(T_{\max}  nD)$ & $\Theta(T_{\max}  nD)$ & $\Omega(nD)$ & $\mathcal{O}(nD)$ \\
			\textbf{IWSO} & $\mathcal{O}(T_{\max}nD)$ & $\Theta(T_{\max}  nD)$ & $\Omega(nD)$ & $\mathcal{O}(nD)$ \\
			\hline
	\end{tabular}}
	
	\footnotesize{$n$: Size of population,  D: Dimension, $T_{max}$: Maximum iterations}
\end{table}
The proposed IWSO algorithm exhibits a worst-case time complexity of $\mathcal{O}(T_{\max}nD)$, with dynamic elimination and reinitialization contributing to adaptive population control. Its average-case complexity remains $\Theta(T_{\max}nD)$, comparable to standard evolutionary approaches, while the best-case scenario reduces to $\Omega(nD)$ when rapid convergence occurs. The space complexity is linear, i.e., $\mathcal{O}(nD)$, owing to its population-based structure and memory-efficient design.

Compared to existing metaheuristics, GA, PSO, DE, CS, HS, BHO, and WOA share similar linear time and space scaling, but often suffer from premature convergence when diversity degrades. SA offers reduced space complexity at $\mathcal{O}(D)$, but requires more iterations due to its single-solution search. In contrast, ACO incurs significantly higher computational overhead, reaching $\mathcal{O}(T_{\max}n^2D)$ with $\mathcal{O}(n^2+nD)$ memory requirements. The balanced exploration–exploitation mechanism of IWSO, supported by structured hierarchy and adaptive reinitialization, enables robust convergence, scalability, and improved resilience in dynamic and multi-objective optimization environments.

\subsection{Statistical Analysis}\label{sta}
The optimization performance is evaluated using four statistical measures: mean fitness, standard deviation, best fitness, and runtime. To account for stochastic variation inherent in metaheuristics, each benchmark function ($f_1$–$f_{23}$) is executed 30 independent times. The proposed IWSO algorithm is compared against ten established optimization techniques: GA~\cite{back1997handbook}, ABC~\cite{karaboga2007powerful}, PSO~\cite{kennedy1995particle}, ACO~\cite{dorigo1996ant}, DE~\cite{storn1997differential}, BHO~\cite{hatamlou2013black}, WOA~\cite{mirjalili2016whale}, HS~\cite{geem2001new}, SA~\cite{kirkpatrick1983optimization}, and CS~\cite{yang2009cuckoo}. The results in Table~\ref{tab:3} demonstrate that IWSO consistently achieves superior mean accuracy, competitive best values, and reduced computational cost across a diverse set of unimodal, multimodal, separable, and non-separable functions. 
\begin{table*}[!htbp]
	\centering
	\caption{Statistical comparison of IWSO and state-of-the-art algorithms across 23 benchmark functions using fitness and runtime metrics.}\label{tab:3}
	\resizebox{0.995\textwidth}{!}{
		\begin{tabular}{ccrrrrrrrrrrr}
		 \hline
			\multirow{2}{*}{\textbf{Function}} & \multirow{2}{*}{\textbf{Metrics}} & \multicolumn{11}{c}{\textbf{Algorithms}} 
			\\ \cline{3-13}
			&  & \textbf{GA} \cite{back1997handbook}& \textbf{ABC} \cite{karaboga2007powerful} & \textbf{PSO} \cite{kennedy1995particle} & \textbf{ACO} \cite{dorigo1996ant}& \textbf{DE} \cite{storn1997differential}& \textbf{BHO} \cite{hatamlou2013black}& \textbf{WOA} \cite{mirjalili2016whale}& \textbf{HS} \cite{geem2001new} & \textbf{SA} \cite{kirkpatrick1983optimization}& \textbf{CS} \cite{yang2009cuckoo} & \textbf{IWSO (Proposed)} \\ \hline
			\multirow{4}{*}{$f_1$} 	&	 mean  	&	20.7083	&	20.2263	&	7.7383	&	20.2263	&	20.0537	&	10.9485	&	0	&	20.1639	&	20.7083	&	20.3587	&	2.557	\\
			&	 std 	&	0.1821	&	0.17	&	2.4884	&	0.17	&	0.4196	&	1.4029	&	0.0001	&	0.1647	&	0.1821	&	0.1509	&	0.3586	\\
			&	 best  	&	20.1796	&	19.6979	&	3.9523	&	19.6979	&	18.6632	&	7.7527	&	0	&	19.6122	&	20.1796	&	19.8712	&	1.4883	\\
			&	 RT 	&	1.1312	&	1.3249	&	4.1707	&	1.3249	&	4.8826	&	2.7267	&	2.9264	&	2.0482	&	1.1312	&	3.4873	&	0.0697	\\ \hline
			\multirow{4}{*}{$f_2$} 	&	 mean  	&	13.7039	&	3.747	&	0	&	3.747	&	0.0027	&	0	&	0	&	7.4874	&	13.7039	&	5.5617	&	0.2054	\\
			&	 std 	&	8.3664	&	1.1543	&	0	&	1.1543	&	0.0022	&	0	&	0	&	5.3594	&	8.3664	&	1.8237	&	0.1548	\\
			&	 best  	&	0.0306	&	0.9476	&	0	&	0.9476	&	0.0002	&	0	&	0	&	0.0884	&	0.0306	&	1.8713	&	0.0236	\\
			&	 RT 	&	0.6394	&	0.8039	&	2.3966	&	0.8039	&	4.0501	&	1.6613	&	2.07	&	1.3546	&	0.6394	&	1.5328	&	0.0593	\\ \hline
			\multirow{4}{*}{$f_3$} 	&	 mean  	&	0.0045	&	0.2489	&	0	&	0.2489	&	0	&	0.0005	&	0.5097	&	0.0079	&	0.0045	&	0.9257	&	0.0058	\\
			&	 std 	&	0.0044	&	0.246	&	0	&	0.246	&	0	&	0.002	&	0.752	&	0.0106	&	0.0044	&	0.9207	&	0.0053	\\
			&	 best  	&	0.0001	&	0.0046	&	0	&	0.0046	&	0	&	0	&	0.0006	&	0.0001	&	0.0001	&	0.0161	&	0.0001	\\
			&	 RT 	&	0.4883	&	0.6408	&	1.9041	&	0.6408	&	3.6916	&	1.3342	&	1.7513	&	1.1695	&	0.4883	&	1.1302	&	0.0523	\\ \hline
			\multirow{4}{*}{$f_4$} 	&	 mean  	&	-123.4026	&	-124.2178	&	-290.7573	&	-124.2178	&	-146.5788	&	-281.35	&	-567.7844	&	-179.3566	&	-123.4026	&	-123.3274	&	-624.2057	\\
			&	 std 	&	17.4571	&	14.3598	&	80.1674	&	14.3598	&	20.6056	&	62.181	&	226.6232	&	19.6636	&	17.4571	&	14.0162	&	209.7289	\\
			&	 best  	&	-174.3431	&	-168.3274	&	-496.6506	&	-168.3274	&	-205.815	&	-457.558	&	-898.9782	&	-235.6576	&	-174.3431	&	-164.6177	&	-823.2724	\\
			&	 RT 	&	2.156	&	2.7107	&	6.1976	&	2.7107	&	6.1596	&	4.4637	&	4.3734	&	3.7121	&	2.156	&	5.9606	&	0.0861	\\ \hline
			\multirow{4}{*}{$f_5$} 	&	 mean  	&	1030.8418	&	977907.1418	&	1213.2005	&	977907.1418	&	891553.1489	&	2403.196	&	0.9515	&	1317758.5418	&	1030.8418	&	1182050.2541	&	33.5096	\\
			&	 std 	&	393.0941	&	164908.9197	&	6046.2603	&	164908.9197	&	317933.8707	&	1530.2204	&	0.0921	&	267623.2469	&	393.0941	&	216275.5152	&	16.3439	\\
			&	 best  	&	391.5788	&	564259.7862	&	30.559	&	564259.7862	&	294522.7219	&	414.3911	&	0.5063	&	669494.4332	&	391.5788	&	640378.6687	&	14.8179	\\
			&	 RT 	&	1.1927	&	1.4174	&	4.3032	&	1.4174	&	4.7113	&	2.5101	&	2.9486	&	2.2676	&	1.1927	&	2.6029	&	0.07	\\ \hline
			\multirow{4}{*}{$f_6$} 	&	 mean  	&	12.2574	&	3.1216	&	3.0369	&	3.1216	&	2.0132	&	4.4421	&	4.6147	&	3.1861	&	12.2574	&	6.455	&	2.0854	\\
			&	 std 	&	8.0781	&	1.9018	&	3.0539	&	1.9018	&	1.2574	&	3.2985	&	4.1357	&	2.4048	&	8.0781	&	4.1927	&	1.9913	\\
			&	 best  	&	0.998	&	0.9987	&	0.998	&	0.9987	&	0.998	&	0.998	&	0.998	&	0.998	&	0.998	&	1.0071	&	0.998	\\
			&	 RT 	&	1.9773	&	2.407	&	6.1687	&	2.407	&	5.8802	&	4.1793	&	4.4497	&	3.9168	&	1.9773	&	5.0682	&	0.0908	\\ \hline
			\multirow{4}{*}{$f_7$} 	&	 mean  	&	721.7954	&	446.3362	&	4.2381	&	446.3362	&	376.1086	&	15.8674	&	0.0284	&	612.3413	&	721.7954	&	492.9028	&	0.6477	\\
			&	 std 	&	129.4955	&	39.7943	&	6.1136	&	39.7943	&	68.9089	&	5.3126	&	0.1444	&	60.7358	&	129.4955	&	43.9694	&	0.2689	\\
			&	 best  	&	439.1735	&	338.8543	&	1.3705	&	338.8543	&	216.881	&	6.7424	&	0	&	456.3183	&	439.1735	&	378.4734	&	0.1522	\\
			&	 RT 	&	1.147	&	1.4865	&	4.835	&	1.4865	&	4.7282	&	2.5135	&	2.9132	&	2.8622	&	1.147	&	3.1344	&	0.0701	\\ \hline
			\multirow{4}{*}{$f_8$} 	&	 mean  	&	110.507	&	143.4114	&	14.1252	&	143.4114	&	181.8608	&	14.2427	&	0.631	&	129.3638	&	110.507	&	162.3421	&	0.6885	\\
			&	 std 	&	25.6171	&	15.8367	&	6.5718	&	15.8367	&	37.2538	&	6.4585	&	0.4206	&	19.696	&	25.6171	&	19.1431	&	0.1553	\\
			&	 best  	&	54.9753	&	102.4569	&	3.3887	&	102.4569	&	96.6541	&	4.3321	&	0.056	&	82.3956	&	54.9753	&	109.2777	&	0.3673	\\
			&	 RT 	&	1.7697	&	2.1923	&	5.9664	&	2.1923	&	5.6738	&	3.6769	&	4.0984	&	3.4745	&	1.7697	&	5.2506	&	0.0848	\\ \hline
			\multirow{4}{*}{$f_9$} 	&	 mean  	&	-4.0172	&	-4.1838	&	-6.7477	&	-4.1838	&	-4.2286	&	-7.0509	&	-4.4631	&	-4.4528	&	-4.0172	&	-4.1041	&	-4.9512	\\
			&	 std 	&	0.4279	&	0.3769	&	0.8646	&	0.3769	&	0.3913	&	0.7001	&	0.5957	&	0.4066	&	0.4279	&	0.3679	&	0.4776	\\
			&	 best  	&	-5.2191	&	-5.2336	&	-8.5855	&	-5.2336	&	-5.3375	&	-8.5138	&	-5.9559	&	-5.6858	&	-5.2191	&	-5.1893	&	-6.0178	\\
			&	 RT 	&	1.2456	&	1.5233	&	3.7144	&	1.5233	&	4.9346	&	2.8437	&	3.101	&	2.445	&	1.2456	&	3.4596	&	0.0861	\\ \hline
			\multirow{4}{*}{$f_{10}$} 	&	 mean  	&	-82.683	&	-57.3646	&	-118.0626	&	-57.3646	&	-110.002	&	-47.3856	&	-110.9544	&	-56.6839	&	-82.683	&	-55.1089	&	-99.7537	\\
			&	 std 	&	4.808	&	2.7577	&	1.8353	&	2.7577	&	3.6606	&	4.2672	&	8.5696	&	3.4426	&	4.808	&	2.8816	&	8.2421	\\
			&	 best  	&	-93.7481	&	-65.1618	&	-119.6102	&	-65.1618	&	-117.6036	&	-58.1113	&	-119.5724	&	-66.3016	&	-93.7481	&	-63.2315	&	-112.6916	\\
			&	 RT 	&	0.8422	&	1.0976	&	2.8618	&	1.0976	&	4.1679	&	1.8411	&	2.2226	&	1.9295	&	0.8422	&	1.9795	&	0.0621	\\ \hline
			\multirow{4}{*}{$f_{11}$} 	&	 mean  	&	10.7866	&	123.2576	&	1.42	&	123.2576	&	104.0645	&	4.1265	&	0.0056	&	128.4301	&	10.7866	&	136.8053	&	1.6559	\\
			&	 std 	&	1.761	&	11.0898	&	2.0468	&	11.0898	&	18.7823	&	1.4189	&	0.1008	&	15.6352	&	1.761	&	12.5737	&	0.3538	\\
			&	 best  	&	6.9111	&	93.4049	&	0.3472	&	93.4049	&	60.3688	&	1.6236	&	-0.2215	&	88.2859	&	6.9111	&	100.3732	&	1.0804	\\
			&	 RT 	&	0.8523	&	1.0279	&	3.0648	&	1.0279	&	4.2197	&	1.8428	&	2.2306	&	2.1818	&	0.8523	&	1.8832	&	0.0607	\\ \hline
			\multirow{4}{*}{$f_{12}$} 	&	 mean  	&	10.7085	&	123.2012	&	1.3772	&	123.2012	&	104.3035	&	4.0892	&	0.0041	&	129.1945	&	10.7085	&	137.1271	&	1.4863	\\
			&	 std 	&	1.8228	&	11.0708	&	1.7321	&	11.0708	&	18.557	&	1.397	&	0.0983	&	15.6883	&	1.8228	&	12.6742	&	0.4934	\\
			&	 best  	&	6.8615	&	92.7539	&	0.3427	&	92.7539	&	59.506	&	1.7012	&	-0.2065	&	88.4148	&	6.8615	&	102.6607	&	0.5922	\\
			&	 RT 	&	0.8333	&	1.0686	&	3.0684	&	1.0686	&	4.2701	&	1.8342	&	2.229	&	2.1427	&	0.8333	&	1.8814	&	0.0604	\\ \hline
			\multirow{4}{*}{$f_{13}$} 	&	 mean  	&	2.2369	&	0.4436	&	0	&	0.4436	&	1.0456	&	0.0005	&	0	&	0.1764	&	2.2369	&	0.6112	&	0.0207	\\
			&	 std 	&	0.6265	&	0.1273	&	0	&	0.1273	&	0.3537	&	0.0005	&	0	&	0.0842	&	0.6265	&	0.1714	&	0.0119	\\
			&	 best  	&	0.7344	&	0.157	&	0	&	0.157	&	0.2859	&	0	&	0	&	0.0393	&	0.7344	&	0.2263	&	0.0025	\\
			&	 RT 	&	0.9242	&	1.2975	&	3.2921	&	1.2975	&	4.1567	&	2.0167	&	2.4299	&	2.3353	&	0.9242	&	2.1414	&	0.063	\\ \hline
			\multirow{4}{*}{$f_{14}$} 	&	 mean  	&	371.7657	&	373.8469	&	124.9974	&	373.8469	&	372.6694	&	119.1805	&	0.0016	&	323.4922	&	371.7657	&	390.8479	&	181.0416	\\
			&	 std 	&	36.1681	&	17.2552	&	27.5995	&	17.2552	&	23.8369	&	22.337	&	0.0109	&	21.197	&	36.1681	&	18.8892	&	18.6921	\\
			&	 best  	&	290.8499	&	323.6791	&	71.7578	&	323.6791	&	312.241	&	73.8718	&	0	&	265.7455	&	290.8499	&	339.5	&	147.6367	\\
			&	 RT 	&	0.9956	&	1.2763	&	3.3543	&	1.2763	&	4.5514	&	2.1786	&	2.5059	&	2.0911	&	0.9956	&	2.1848	&	0.0665	\\ \hline
			\multirow{4}{*}{$f_{15}$} 	&	 mean  	&	371.7113	&	372.5696	&	126.1918	&	372.5696	&	372.9791	&	120.091	&	0.0315	&	322.5109	&	371.7113	&	390.8588	&	177.4875	\\
			&	 std 	&	39.8777	&	17.5954	&	27.8992	&	17.5954	&	24.2697	&	22.9113	&	0.2207	&	21.7293	&	39.8777	&	19.0521	&	20.3809	\\
			&	 best  	&	283.2971	&	322.7659	&	69.3772	&	322.7659	&	304.8076	&	73.7321	&	0	&	265.969	&	283.2971	&	337.7359	&	124.1226	\\
			&	 RT 	&	0.9755	&	1.2247	&	3.3053	&	1.2247	&	4.4314	&	2.1551	&	2.5046	&	2.0784	&	0.9755	&	2.2465	&	0.0661	\\ \hline
			\multirow{4}{*}{$f_{16}$} 	&	 mean  	&	2995.5657	&	723019.5528	&	11265.1478	&	723019.5528	&	392654.4416	&	15900.0433	&	27.1143	&	1100753.3991	&	2995.5657	&	941511.4551	&	163.1935	\\
			&	 std 	&	1041.6904	&	158534.6916	&	22942.0856	&	158534.6916	&	121223.3502	&	8369.3118	&	5.6278	&	291815.1596	&	1041.6904	&	204872.0837	&	42.7259	\\
			&	 best  	&	1407.6476	&	367462.8944	&	195.6143	&	367462.8944	&	158694.2723	&	3493.2787	&	2.8573	&	470482.0796	&	1407.6476	&	477450.6054	&	87.0321	\\
			&	 RT 	&	1.0994	&	1.6986	&	4.006	&	1.6986	&	4.6629	&	2.3693	&	2.8031	&	2.334	&	1.0994	&	2.6478	&	0.0674	\\ \hline
			\multirow{4}{*}{$f_{17}$} 	&	 mean  	&	29.5951	&	22.9271	&	5.3647	&	22.9271	&	21.3432	&	6.5068	&	0.1096	&	26.0982	&	29.5951	&	24.0047	&	0.3917	\\
			&	 std 	&	2.5449	&	1.046	&	1.4956	&	1.046	&	1.9478	&	1.0063	&	0.0745	&	1.3606	&	2.5449	&	1.0367	&	0.0895	\\
			&	 best  	&	23.1355	&	19.8502	&	2.9144	&	19.8502	&	16.5326	&	4.4482	&	0	&	22.6499	&	23.1355	&	20.7593	&	0.2133	\\
			&	 RT 	&	0.7016	&	1.3427	&	12.563	&	1.3427	&	4.2829	&	1.6509	&	2.0497	&	1.6297	&	0.7016	&	1.5608	&	0.0587	\\ \hline
			\multirow{4}{*}{$f_{18}$} 	&	 mean  	&	7472.5441	&	9025.825	&	5128.4046	&	9025.825	&	8469.7005	&	8247.5784	&	525.9109	&	10317.8363	&	7472.5441	&	9392.9262	&	1.9797	\\
			&	 std 	&	861.4514	&	347.3494	&	705.9686	&	347.3494	&	411.0457	&	592.8784	&	936.6911	&	507.4466	&	861.4514	&	372.5666	&	1.4654	\\
			&	 best  	&	5635.2214	&	8005.0836	&	3502.9394	&	8005.0836	&	7412.9351	&	6834.8532	&	0.0008	&	8913.437	&	5635.2214	&	8310.4392	&	0.5515	\\
			&	 RT 	&	0.825	&	1.6873	&	3.1534	&	1.6873	&	4.2161	&	1.9178	&	2.3316	&	2.0637	&	0.825	&	2.1386	&	0.0619	\\ \hline
			\multirow{4}{*}{$f_{19}$} 	&	 mean  	&	-18.9561	&	-12.7934	&	-25.1704	&	-12.7934	&	-15.8938	&	-21.4609	&	-29.9184	&	-13.1462	&	-18.9561	&	-11.4042	&	-29.4721	\\
			&	 std 	&	2.4725	&	1.2007	&	2.04	&	1.2007	&	1.5855	&	1.7725	&	0.1176	&	1.6908	&	2.4725	&	1.3024	&	0.1365	\\
			&	 best  	&	-24.1167	&	-16.0625	&	-29.14	&	-16.0625	&	-19.7347	&	-25.0593	&	-29.9987	&	-17.611	&	-24.1167	&	-14.8399	&	-29.7224	\\
			&	 RT 	&	0.6882	&	1.2724	&	2.7437	&	1.2724	&	4.0293	&	1.6078	&	1.9743	&	1.6669	&	0.6882	&	1.594	&	0.0611	\\ \hline
			\multirow{4}{*}{$f_{20}$} 	&	 mean  	&	10.9816	&	129.7315	&	0.9232	&	129.7315	&	109.3304	&	4.1138	&	0	&	137.0986	&	10.9816	&	143.3375	&	1.2672	\\
			&	 std 	&	1.7671	&	11.404	&	1.7931	&	11.404	&	19.7661	&	1.4213	&	0	&	16.0695	&	1.7671	&	12.4903	&	0.3739	\\
			&	 best  	&	7.2049	&	98.1594	&	0.1213	&	98.1594	&	66.2345	&	1.6066	&	0	&	97.1343	&	7.2049	&	109.0634	&	0.6065	\\
			&	 RT 	&	0.6812	&	1.2289	&	2.692	&	1.2289	&	4.1733	&	1.5693	&	1.9421	&	1.6177	&	0.6812	&	1.5437	&	0.0583	\\ \hline
			\multirow{4}{*}{$f_{21}$} 	&	 mean  	&	0.03	&	0.027	&	0.0008	&	0.027	&	0	&	0.0001	&	0.0601	&	0.0011	&	0.03	&	0.102	&	0.0008	\\
			&	 std 	&	0.0681	&	0.026	&	0.0056	&	0.026	&	0	&	0.0009	&	0.119	&	0.0015	&	0.0681	&	0.0904	&	0.0007	\\
			&	 best  	&	0	&	0.0007	&	0	&	0.0007	&	0	&	0	&	0	&	0	&	0	&	0.0024	&	0.0001	\\
			&	 RT 	&	0.5263	&	1.0257	&	1.9721	&	1.0257	&	3.8308	&	1.3808	&	1.7889	&	1.3056	&	0.5263	&	1.0278	&	0.0535	\\ \hline
			\multirow{4}{*}{$f_{22}$} 	&	 mean  	&	0.0602	&	0.0393	&	37.5178	&	0.0393	&	43.3762	&	0.0338	&	6.3562	&	0.0399	&	0.0602	&	0.0388	&	1.3246	\\
			&	 std 	&	0.0585	&	0.0383	&	29.8536	&	0.0383	&	37.0923	&	0.0321	&	11.8082	&	0.0381	&	0.0585	&	0.0394	&	1.0455	\\
			&	 best  	&	0.0008	&	0.0009	&	1.1519	&	0.0009	&	1.2858	&	0.0007	&	0.0002	&	0.0008	&	0.0008	&	0.0008	&	0.0255	\\
			&	 RT 	&	1.0176	&	1.8978	&	3.1271	&	1.8978	&	4.6127	&	2.2161	&	2.6571	&	2.1629	&	1.0176	&	2.1947	&	0.0665	\\ \hline
			\multirow{4}{*}{$f_{23}$} 	&	 mean  	&	236.8709	&	95490.488	&	407.442	&	95490.488	&	572.6001	&	165472.6289	&	487.6583	&	441106651.4365	&	236.8709	&	8253718.3912	&	22.9432	\\
			&	 std 	&	64.6311	&	408664.3156	&	135.1243	&	408664.3156	&	88.369	&	965990.1256	&	140.4331	&	915307155.3103	&	64.6311	&	24844467.4202	&	22.337	\\
			&	 best  	&	117.4776	&	367.611	&	164.3393	&	367.611	&	368.3973	&	92.7639	&	94.7097	&	362.9299	&	117.4776	&	421.8473	&	5.0756	\\
			&	 RT 	&	1.4231	&	2.1523	&	3.6997	&	2.1523	&	4.6493	&	2.4269	&	2.8384	&	2.4691	&	1.4231	&	2.392	&	0.0672	\\
			\hline
	\end{tabular}}
\end{table*}
The comparative analysis shows that IWSO offers strong robustness, as reflected by low standard deviation across multiple runs and rapid convergence without premature stagnation. Its hybrid elimination and matchmaker-driven adaptation facilitates an effective balance between exploration and exploitation, enabling reliable convergence even on complex, noisy, or high-dimensional landscapes. Overall, the results confirm that IWSO provides an efficient, scalable, and highly competitive optimization strategy, outperforming classical and modern metaheuristics in both stability and solution quality.

\section{Conclusion} \label{conclusion}

This work introduced the IWSO algorithm, a socially inspired metaheuristic grounded in the collaborative and hierarchical decision logic of traditional Indian wedding matchmaking. Unlike nature- or physics-based optimizers, IWSO leverages mechanisms such as adaptive elimination, guided reinitialization, and matchmaker-driven influence to sustain population diversity, regulate exploration–exploitation balance, and avoid premature convergence with minimal parameter dependency. Experimental results across a broad set of benchmark functions show that IWSO consistently surpasses classical optimizers including GA, PSO, DE, and ACO in convergence rate, solution quality, and stability, especially in high-dimensional and multimodal optimization scenarios. Analysis confirms high computational efficiency and strong scalability, supporting its use in large-scale and dynamic optimization problems.

Future research will focus on developing hybrid extensions by incorporating reinforcement learning and adaptive local refinement to further strengthen global convergence characteristics. Additionally, extending IWSO to multi-objective, distributed, and time-varying optimization settings will broaden its applicability to practical domains such as cloud resource orchestration, industrial scheduling, and intelligent cyber–physical infrastructures.

\bibliographystyle{IEEEtran}
\bibliography{bibfile1}

\end{document}

%% file: Ematch.tex
\begin{figure*}[!htbp]
		\begin{tikzpicture}
			\begin{axis}[
				width=0.99\textwidth,
				height=5cm,
				xlabel=\scriptsize Iterations,
				ylabel=\scriptsize $E_{match}$,
				xmin=-2, xmax=102,
				ymin=-8, ymax=27,
				yticklabel style = {font=\scriptsize},
				xticklabel style = {font=\scriptsize},
				xtick={1,10,20,30,40,50,60,70,80,90,100},   
				ytick={0,5,10,15,20,25},grid=both,
				legend style={at={(0.00,0.95)},anchor=north west,legend cell align=left,legend columns=8},
				]
				\addplot[Apricot] table[x=Epochs,y=f1_E] {ematch.dat}; \addlegendentry{\scriptsize $f_1$} 
                \addplot[Aquamarine] table[x=Epochs,y=f2_E] {ematch.dat}; \addlegendentry{\scriptsize $f_2$} 
				\addplot[Bittersweet] table[x=Epochs,y=f3_E] {ematch.dat}; \addlegendentry{\scriptsize $f_3$} 
                \addplot[Black] table[x=Epochs,y=f4_E] {ematch.dat}; \addlegendentry{\scriptsize $f_4$} 
				\addplot[Blue] table[x=Epochs,y=f5_E] {ematch.dat}; \addlegendentry{\scriptsize $f_5$} 
                \addplot[BlueGreen] table[x=Epochs,y=f6_E] {ematch.dat}; \addlegendentry{\scriptsize $f_6$} 
                
				\addplot[BlueViolet] table[x=Epochs,y=f7_E] {ematch.dat}; \addlegendentry{\scriptsize $f_7$} 
                \addplot[BrickRed] table[x=Epochs,y=f8_E] {ematch.dat}; \addlegendentry{\scriptsize $f_8$} 
				\addplot[Brown] table[x=Epochs,y=f9_E] {ematch.dat}; \addlegendentry{\scriptsize $f_9$} 
                \addplot[BurntOrange] table[x=Epochs,y=f10_E] {ematch.dat}; \addlegendentry{\scriptsize $f_{10}$} 
				\addplot[CadetBlue] table[x=Epochs,y=f11_E] {ematch.dat}; \addlegendentry{\scriptsize $f_{11}$} 
                \addplot[CarnationPink] table[x=Epochs,y=f12_E] {ematch.dat}; \addlegendentry{\scriptsize $f_{12}$} 
				\addplot[Cerulean] table[x=Epochs,y=f13_E] {ematch.dat}; \addlegendentry{\scriptsize $f_{13}$} 
                \addplot[CornflowerBlue] table[x=Epochs,y=f14_E] {ematch.dat}; \addlegendentry{\scriptsize $f_{14}$} 
				\addplot[Cyan] table[x=Epochs,y=f15_E] {ematch.dat}; \addlegendentry{\scriptsize $f_{15}$} 
                \addplot[Dandelion] table[x=Epochs,y=f16_E] {ematch.dat}; \addlegendentry{\scriptsize $f_{16}$} 
				\addplot[DarkOrchid] table[x=Epochs,y=f17_E] {ematch.dat}; \addlegendentry{\scriptsize $f_{17}$} 
                \addplot[Emerald] table[x=Epochs,y=f18_E] {ematch.dat}; \addlegendentry{\scriptsize $f_{18}$} 
                
				\addplot[ForestGreen] table[x=Epochs,y=f19_E] {ematch.dat}; \addlegendentry{\scriptsize $f_{19}$} 
                \addplot[Fuchsia] table[x=Epochs,y=f20_E] {ematch.dat}; \addlegendentry{\scriptsize $f_{20}$} 
				\addplot[Goldenrod] table[x=Epochs,y=f21_E] {ematch.dat}; \addlegendentry{\scriptsize $f_{21}$} 
                \addplot[Gray] table[x=Epochs,y=f22_E] {ematch.dat}; \addlegendentry{\scriptsize $f_{22}$} 
				\addplot[Green] table[x=Epochs,y=f23_E] {ematch.dat}; \addlegendentry{\scriptsize $f_{23}$} 
			\end{axis}
		\end{tikzpicture}
        \caption{Evolution of the Expected Match ($E_{Match}$) Threshold in IWSO: Parameter dynamics across 100 iterations for 23 benchmark functions highlighting exploration behavior and adaptation trends}
\label{fig:cp}
    \end{figure*}

%% file: con-div11.tex
\begin{figure*}[!htbp]
	\centering
	\begin{subfigure}{0.32\textwidth}
		\centering
		\begin{tikzpicture}
			\begin{axis}[
				width=1.1\textwidth,
				height=4.75cm,
				xlabel=\scriptsize Iterations,
				ylabel=\scriptsize Fitness value,
				xmin=-2, xmax=52,
				ymin=20, ymax=120,
				yticklabel style = {font=\scriptsize},
				xticklabel style = {font=\scriptsize},
				xtick={0,10,20,30,40,50},   
				ytick={0,20,40,60,80,100},grid=both,
				legend style={at={(1.32,1.2)},anchor=north west,cells={anchor=west},legend cell align=left,legend columns=2},
				]
				\addplot[mark=halfcircle*,blue,
				mark options={scale=1.5,fill=blue}] table[x=label,y=Diver-f1] {foo1.dat}; \addlegendentry{\scriptsize Divergence} 
				\addplot[mark=halfcircle*,magenta,
				mark options={scale=1.5,fill=red}] table[x=label,y=Conver-f1] {foo1.dat}; \addlegendentry{\scriptsize Convergence} 
			\end{axis}
		\end{tikzpicture}
		\caption{with $f_1$}\label{fig:mp1}
	\end{subfigure} \hfill
    \begin{subfigure}{0.32\textwidth}
		\centering
		\begin{tikzpicture}
			\begin{axis}[
				width=1.1\textwidth,
				height=4.75cm,
				xlabel=\scriptsize Iterations,
				ylabel=\scriptsize Fitness value,
				xmin=-2, xmax=52,
				ymin=6, ymax=25,
				yticklabel style = {font=\scriptsize},
				xticklabel style = {font=\scriptsize},
				xtick={0,10,20,30,40,50},   
				ytick={6,12,18,24},grid=both,
				]
				\addplot[mark=halfcircle*,blue,
				mark options={scale=1.5,fill=blue}] table[x=label,y=Diver-f2] {foo1.dat}; 
				\addplot[mark=halfcircle*,magenta,
				mark options={scale=1.5,fill=red}] table[x=label,y=Conver-f2] {foo1.dat}; 
			\end{axis}
		\end{tikzpicture}
		\caption{with $f_2$}\label{fig:mp2}
	\end{subfigure} \hfill
	\begin{subfigure}{0.32\textwidth}
		\centering
		\begin{tikzpicture}
			\begin{axis}[
				width=1.1\textwidth,
				height=4.75cm,
				xlabel=\scriptsize Iterations,
				ylabel=\scriptsize Fitness value,
				xmin=-2, xmax=52,
				ymin=1.5, ymax=9,
				yticklabel style = {font=\scriptsize},
				xticklabel style = {font=\scriptsize},
				xtick={0,10,20,30,40,50},   
				ytick={2,4,6,8},grid=both,
				]
				\addplot[mark=halfcircle*,blue,
				mark options={scale=1.5,fill=blue}] table[x=label,y=Diver-f3] {foo1.dat}; 
				\addplot[mark=halfcircle*,magenta,
				mark options={scale=1.5,fill=red}] table[x=label,y=Conver-f3] {foo1.dat}; 
			\end{axis}
		\end{tikzpicture}
		\caption{with $f_3$}\label{fig:mp3}
	\end{subfigure} \\
		\begin{subfigure}{0.32\textwidth}
		\centering
		\begin{tikzpicture}
			\begin{axis}[
				width=1.1\textwidth,
				height=4.75cm,
				xlabel=\scriptsize Iterations,
				ylabel=\scriptsize Fitness value,
				xmin=-2, xmax=52,
				ymin=0, ymax=35,
				yticklabel style = {font=\scriptsize},
				xticklabel style = {font=\scriptsize},
				xtick={0,10,20,30,40,50},   
				ytick={8,16,24,32},grid=both,
				legend style={at={(1.0,1.2)},anchor=north west,cells={anchor=west},legend cell align=left,legend columns=2},
				]
				\addplot[mark=halfcircle*,blue,
				mark options={scale=1.5,fill=blue}] table[x=label,y=Diver-f4] {foo1.dat};
				\addplot[mark=halfcircle*,magenta,
				mark options={scale=1.5,fill=red}] table[x=label,y=Conver-f4] {foo1.dat};
			\end{axis}
		\end{tikzpicture}
		\caption{with $f_4$}\label{fig:mp4}
	\end{subfigure} \hfill\begin{subfigure}{0.32\textwidth}
		\centering
		\begin{tikzpicture}
			\begin{axis}[
				width=1.1\textwidth,
				height=4.75cm,
				xlabel=\scriptsize Iterations,
				ylabel=\scriptsize Fitness value,
				xmin=-2, xmax=52,
				ymin=6, ymax=33,
				yticklabel style = {font=\scriptsize},
				xticklabel style = {font=\scriptsize},
				xtick={0,10,20,30,40,50},   
				ytick={8,16,24,32},grid=both,
				]
				\addplot[mark=halfcircle*,blue,
				mark options={scale=1.5,fill=blue}] table[x=label,y=Diver-f5] {foo1.dat}; 
				\addplot[mark=halfcircle*,magenta,
				mark options={scale=1.5,fill=red}] table[x=label,y=Conver-f5] {foo1.dat}; 
			\end{axis}
		\end{tikzpicture}
		\caption{with $f_5$}\label{fig:mp5}
	\end{subfigure} \hfill
	\begin{subfigure}{0.32\textwidth}
		\centering
		\begin{tikzpicture}
			\begin{axis}[
				width=1.1\textwidth,
				height=4.75cm,
				xlabel=\scriptsize Iterations,
				ylabel=\scriptsize Fitness value,
				xmin=-2, xmax=52,
				ymin=17, ymax=52,
				yticklabel style = {font=\scriptsize},
				xticklabel style = {font=\scriptsize},
				xtick={0,10,20,30,40,50},   
				ytick={20,30,40,50},grid=both,
				]
				\addplot[mark=halfcircle*,blue,
				mark options={scale=1.5,fill=blue}] table[x=label,y=Diver-f6] {foo1.dat}; 
				\addplot[mark=halfcircle*,magenta,
				mark options={scale=1.5,fill=red}] table[x=label,y=Conver-f6] {foo1.dat}; 
			\end{axis}
		\end{tikzpicture}
		\caption{with $f_6$}\label{fig:mp6}
	\end{subfigure} \\
		\begin{subfigure}{0.32\textwidth}
		\centering
		\begin{tikzpicture}
			\begin{axis}[
				width=1.1\textwidth,
				height=4.75cm,
				xlabel=\scriptsize Iterations,
				ylabel=\scriptsize Fitness value,
				xmin=-2, xmax=52,
				ymin=500, ymax=1700,
				yticklabel style = {font=\scriptsize},
				xticklabel style = {font=\scriptsize},
				xtick={0,10,20,30,40,50},   
				ytick={0,600,900,1200,1500},grid=both,
				legend style={at={(1.0,1.2)},anchor=north west,cells={anchor=west},legend cell align=left,legend columns=2},
				]
				\addplot[mark=halfcircle*,blue,
				mark options={scale=1.5,fill=blue}] table[x=label,y=Diver-f7] {foo1.dat};
				\addplot[mark=halfcircle*,magenta,
				mark options={scale=1.5,fill=red}] table[x=label,y=Conver-f7] {foo1.dat};
			\end{axis}
		\end{tikzpicture}
		\caption{with $f_{7}$}\label{fig:mp7}
	\end{subfigure} \hfill\begin{subfigure}{0.32\textwidth}
		\centering
		\begin{tikzpicture}
			\begin{axis}[
				width=1.1\textwidth,
				height=4.75cm,
				xlabel=\scriptsize Iterations,
				ylabel=\scriptsize Fitness value,
				xmin=-2, xmax=52,
				ymin=5, ymax=33,
				yticklabel style = {font=\scriptsize},
				xticklabel style = {font=\scriptsize},
				xtick={0,10,20,30,40,50},   
				ytick={8,16,24,32},grid=both,
				]
				\addplot[mark=halfcircle*,blue,
				mark options={scale=1.5,fill=blue}] table[x=label,y=Diver-f8] {foo1.dat}; 
				\addplot[mark=halfcircle*,magenta,
				mark options={scale=1.5,fill=red}] table[x=label,y=Conver-f8] {foo1.dat}; 
			\end{axis}
		\end{tikzpicture}
		\caption{with $f_{8}$}\label{fig:mp8}
	\end{subfigure} \hfill
	\begin{subfigure}{0.32\textwidth}
		\centering
		\begin{tikzpicture}
			\begin{axis}[
				width=1.1\textwidth,
				height=4.75cm,
				xlabel=\scriptsize Iterations,
				ylabel=\scriptsize Fitness value,
				xmin=-2, xmax=52,
				ymin=2, ymax=5.5,
				yticklabel style = {font=\scriptsize},
				xticklabel style = {font=\scriptsize},
				xtick={0,10,20,30,40,50},   
				ytick={2,3,4,5},grid=both,
				]
				\addplot[mark=halfcircle*,blue,
				mark options={scale=1.5,fill=blue}] table[x=label,y=Diver-f9] {foo1.dat}; 
				\addplot[mark=halfcircle*,magenta,
				mark options={scale=1.5,fill=red}] table[x=label,y=Conver-f9] {foo1.dat}; 
			\end{axis}
		\end{tikzpicture}
		\caption{with $f_{9}$}\label{fig:mp9}
	\end{subfigure} \\
		\begin{subfigure}{0.32\textwidth}
		\centering
		\begin{tikzpicture}
			\begin{axis}[
				width=1.1\textwidth,
				height=4.75cm,
				xlabel=\scriptsize Iterations,
				ylabel=\scriptsize Fitness value,
				xmin=-2, xmax=52,
				ymin=3, ymax=10.5,
				yticklabel style = {font=\scriptsize},
				xticklabel style = {font=\scriptsize},
				xtick={0,10,20,30,40,50},   
				ytick={4,6,8,10},grid=both,
				legend style={at={(1.0,1.2)},anchor=north west,cells={anchor=west},legend cell align=left,legend columns=2},
				]
				\addplot[mark=halfcircle*,blue,
				mark options={scale=1.5,fill=blue}] table[x=label,y=Diver-f10] {foo1.dat};
				\addplot[mark=halfcircle*,magenta,
				mark options={scale=1.5,fill=red}] table[x=label,y=Conver-f10] {foo1.dat};
			\end{axis}
		\end{tikzpicture}
		\caption{with $f_{10}$}\label{fig:mp10}
	\end{subfigure} \hfill\begin{subfigure}{0.32\textwidth}
		\centering
		\begin{tikzpicture}
			\begin{axis}[
				width=1.1\textwidth,
				height=4.75cm,
				xlabel=\scriptsize Iterations,
				ylabel=\scriptsize Fitness value,
				xmin=-2, xmax=52,
				ymin=5, ymax=18,
				yticklabel style = {font=\scriptsize},
				xticklabel style = {font=\scriptsize},
				xtick={0,10,20,30,40,50},   
				ytick={6,9,12,15},grid=both,
				]
				\addplot[mark=halfcircle*,blue,
				mark options={scale=1.5,fill=blue}] table[x=label,y=Diver-f11] {foo1.dat}; 
				\addplot[mark=halfcircle*,magenta,
				mark options={scale=1.5,fill=red}] table[x=label,y=Conver-f11] {foo1.dat}; 
			\end{axis}
		\end{tikzpicture}
		\caption{with $f_{11}$}\label{fig:mp11}
	\end{subfigure} \hfill
	\begin{subfigure}{0.32\textwidth}
		\centering
		\begin{tikzpicture}
			\begin{axis}[
				width=1.1\textwidth,
				height=4.75cm,
				xlabel=\scriptsize Iterations,
				ylabel=\scriptsize Fitness value,
				xmin=-2, xmax=52,
				ymin=6.5, ymax=20,
				yticklabel style = {font=\scriptsize},
				xticklabel style = {font=\scriptsize},
				xtick={0,10,20,30,40,50},   
				ytick={9,12,15,18},grid=both,
				]
				\addplot[mark=halfcircle*,blue,
				mark options={scale=1.5,fill=blue}] table[x=label,y=Diver-f12] {foo1.dat}; 
				\addplot[mark=halfcircle*,magenta,
				mark options={scale=1.5,fill=red}] table[x=label,y=Conver-f12] {foo1.dat}; 
			\end{axis}
		\end{tikzpicture}
		\caption{with $f_{12}$}\label{fig:mp12}
	\end{subfigure} \\
    \begin{subfigure}{0.32\textwidth}
		\centering
		\begin{tikzpicture}
			\begin{axis}[
				width=1.1\textwidth,
				height=4.75cm,
				xlabel=\scriptsize Iterations,
				ylabel=\scriptsize Fitness value,
				xmin=-2, xmax=52,
				ymin=3, ymax=7.5,
				yticklabel style = {font=\scriptsize},
				xticklabel style = {font=\scriptsize},
				xtick={0,10,20,30,40,50},   
				ytick={4,5,6,7},grid=both,
				legend style={at={(1.15,1.22)},anchor=north west,cells={anchor=west},legend cell align=left,legend columns=2},
				]
				\addplot[mark=halfcircle*,blue,
				mark options={scale=1.5,fill=blue}] table[x=label,y=Diver-f13] {foo1.dat}; 
				\addplot[mark=halfcircle*,magenta,
				mark options={scale=1.5,fill=red}] table[x=label,y=Conver-f13] {foo1.dat}; 
			\end{axis}
		\end{tikzpicture}
		\caption{with $f_{13}$}\label{fig:mp13}
	\end{subfigure} \hfill\begin{subfigure}{0.32\textwidth}
		\centering
		\begin{tikzpicture}
			\begin{axis}[
				width=1.1\textwidth,
				height=4.75cm,
				xlabel=\scriptsize Iterations,
				ylabel=\scriptsize Fitness value,
				xmin=-2, xmax=52,
				ymin=6, ymax=19,
				yticklabel style = {font=\scriptsize},
				xticklabel style = {font=\scriptsize},
				xtick={0,10,20,30,40,50},   
				ytick={9,12,15,18},grid=both,
				]
				\addplot[mark=halfcircle*,blue,
				mark options={scale=1.5,fill=blue}] table[x=label,y=Diver-f14] {foo1.dat}; 
				\addplot[mark=halfcircle*,magenta,
				mark options={scale=1.5,fill=red}] table[x=label,y=Conver-f14] {foo1.dat}; 
			\end{axis}
		\end{tikzpicture}
		\caption{with $f_{14}$}\label{fig:mp14}
	\end{subfigure} \hfill
	\begin{subfigure}{0.32\textwidth}
		\centering
		\begin{tikzpicture}
			\begin{axis}[
				width=1.1\textwidth,
				height=4.75cm,
				xlabel=\scriptsize Iterations,
				ylabel=\scriptsize Fitness value,
				xmin=-2, xmax=52,
				ymin=6, ymax=19,
				yticklabel style = {font=\scriptsize},
				xticklabel style = {font=\scriptsize},
				xtick={0,10,20,30,40,50},   
				ytick={9,1,2,15,18},grid=both,
				]
				\addplot[mark=halfcircle*,blue,
				mark options={scale=1.5,fill=blue}] table[x=label,y=Diver-f15] {foo1.dat}; 
				\addplot[mark=halfcircle*,magenta,
				mark options={scale=1.5,fill=red}] table[x=label,y=Conver-f15] {foo1.dat}; 
			\end{axis}
		\end{tikzpicture}
		\caption{with $f_{15}$}\label{fig:mp15}
	\end{subfigure} 
\end{figure*}


%% file: con-div22.tex
\begin{figure*}[!htbp]
	\centering
    \ContinuedFloat
		\begin{subfigure}{0.32\textwidth}
		\centering
		\begin{tikzpicture}
			\begin{axis}[
				width=1.1\textwidth,
				height=4.75cm,
				xlabel=\scriptsize Iterations,
				ylabel=\scriptsize Fitness value,
				xmin=-2, xmax=52,
				ymin=7, ymax=26,
				yticklabel style = {font=\scriptsize},
				xticklabel style = {font=\scriptsize},
				xtick={0,10,20,30,40,50},   
				ytick={10,15,20,25},grid=both,
				legend style={at={(1.37,1.2)},anchor=north west,cells={anchor=west},legend cell align=left,legend columns=2},
				]
				\addplot[mark=halfcircle*,blue,
				mark options={scale=1.5,fill=blue}] table[x=label,y=Diver-f16] {foo1.dat};
			 \addlegendentry{\scriptsize Divergence} 
				\addplot[mark=halfcircle*,magenta,
				mark options={scale=1.5,fill=red}] table[x=label,y=Conver-f16] {foo1.dat};
			 \addlegendentry{\scriptsize Convergence} 
			\end{axis}
		\end{tikzpicture}
		\caption{with $f_{16}$}\label{fig:mp16}
	\end{subfigure} \hfill
    \begin{subfigure}{0.32\textwidth}
		\centering
		\begin{tikzpicture}
			\begin{axis}[
				width=1.1\textwidth,
				height=4.75cm,
				xlabel=\scriptsize Iterations,
				ylabel=\scriptsize Fitness value,
				xmin=-2, xmax=52,
				ymin=80, ymax=300,
				yticklabel style = {font=\scriptsize},
				xticklabel style = {font=\scriptsize},
				xtick={0,10,20,30,40,50},   
				ytick={100,150,200,250},grid=both,
				]
				\addplot[mark=halfcircle*,blue,
				mark options={scale=1.5,fill=blue}] table[x=label,y=Diver-f17] {foo1.dat}; 
				\addplot[mark=halfcircle*,magenta,
				mark options={scale=1.5,fill=red}] table[x=label,y=Conver-f17] {foo1.dat}; 
			\end{axis}
		\end{tikzpicture}
		\caption{with $f_{17}$}\label{fig:mp17}
	\end{subfigure} \hfill
	\begin{subfigure}{0.32\textwidth}
		\centering
		\begin{tikzpicture}
			\begin{axis}[
				width=1.1\textwidth,
				height=4.75cm,
				xlabel=\scriptsize Iterations,
				ylabel=\scriptsize Fitness value,
				xmin=-2, xmax=52,
				ymin=450, ymax=1700,
				yticklabel style = {font=\scriptsize},
				xticklabel style = {font=\scriptsize},
				xtick={0,10,20,30,40,50},   
				ytick={600,900,1200,1500},grid=both,
				]
				\addplot[mark=halfcircle*,blue,
				mark options={scale=1.5,fill=blue}] table[x=label,y=Diver-f18] {foo1.dat}; 
				\addplot[mark=halfcircle*,magenta,
				mark options={scale=1.5,fill=red}] table[x=label,y=Conver-f18] {foo1.dat}; 
			\end{axis}
		\end{tikzpicture}
		\caption{with $f_{18}$}\label{fig:mp18}
	\end{subfigure} \\
		\begin{subfigure}{0.32\textwidth}
		\centering
		\begin{tikzpicture}
			\begin{axis}[
				width=1.1\textwidth,
				height=4.75cm,
				xlabel=\scriptsize Iterations,
				ylabel=\scriptsize Fitness value,
				xmin=-2, xmax=52,
				ymin=4.5, ymax=13,
				yticklabel style = {font=\scriptsize},
				xticklabel style = {font=\scriptsize},
				xtick={0,10,20,30,40,50},   
				ytick={6,8,10,12},grid=both,
				legend style={at={(1.0,1.2)},anchor=north west,cells={anchor=west},legend cell align=left,legend columns=2},
				]
				\addplot[mark=halfcircle*,blue,
				mark options={scale=1.5,fill=blue}] table[x=label,y=Diver-f19] {foo1.dat};
				\addplot[mark=halfcircle*,magenta,
				mark options={scale=1.5,fill=red}] table[x=label,y=Conver-f19] {foo1.dat};
			\end{axis}
		\end{tikzpicture}
		\caption{with $f_{19}$}\label{fig:mp19}
	\end{subfigure} \hfill\begin{subfigure}{0.32\textwidth}
		\centering
		\begin{tikzpicture}
			\begin{axis}[
				width=1.1\textwidth,
				height=4.75cm,
				xlabel=\scriptsize Iterations,
				ylabel=\scriptsize Fitness value,
				xmin=-2, xmax=52,
				ymin=6, ymax=18,
				yticklabel style = {font=\scriptsize},
				xticklabel style = {font=\scriptsize},
				xtick={0,10,20,30,40,50},   
				ytick={7,10,13,16},grid=both,
				]
				\addplot[mark=halfcircle*,blue,
				mark options={scale=1.5,fill=blue}] table[x=label,y=Diver-f20] {foo1.dat}; 
				\addplot[mark=halfcircle*,magenta,
				mark options={scale=1.5,fill=red}] table[x=label,y=Conver-f20] {foo1.dat}; 
			\end{axis}
		\end{tikzpicture}
		\caption{with $f_{20}$}\label{fig:mp20}
	\end{subfigure} \hfill
	\begin{subfigure}{0.32\textwidth}
		\centering
		\begin{tikzpicture}
			\begin{axis}[
				width=1.1\textwidth,
				height=4.75cm,
				xlabel=\scriptsize Iterations,
				ylabel=\scriptsize Fitness value,
				xmin=-2, xmax=52,
				ymin=1.5, ymax=4.25,
				yticklabel style = {font=\scriptsize},
				xticklabel style = {font=\scriptsize},
				xtick={0,10,20,30,40,50},   
				ytick={2,3,4},grid=both,
				]
				\addplot[mark=halfcircle*,blue,
				mark options={scale=1.5,fill=blue}] table[x=label,y=Diver-f21] {foo1.dat}; 
				\addplot[mark=halfcircle*,magenta,
				mark options={scale=1.5,fill=red}] table[x=label,y=Conver-f21] {foo1.dat}; 
			\end{axis}
		\end{tikzpicture}
		\caption{with $f_{21}$}\label{fig:mp21}
	\end{subfigure} \\
		\begin{subfigure}{0.32\textwidth}
		\centering
		\begin{tikzpicture}
			\begin{axis}[
				width=1.1\textwidth,
				height=4.75cm,
				xlabel=\scriptsize Iterations,
				ylabel=\scriptsize Fitness value,
				xmin=-2, xmax=52,
				ymin=8, ymax=36,
				yticklabel style = {font=\scriptsize},
				xticklabel style = {font=\scriptsize},
				xtick={0,10,20,30,40,50},   
				ytick={0,8,16,24,32},grid=both,
				legend style={at={(1.0,1.2)},anchor=north west,cells={anchor=west},legend cell align=left,legend columns=2},
				]
				\addplot[mark=halfcircle*,blue,
				mark options={scale=1.5,fill=blue}] table[x=label,y=Diver-f22] {foo1.dat};
				\addplot[mark=halfcircle*,magenta,
				mark options={scale=1.5,fill=red}] table[x=label,y=Conver-f22] {foo1.dat};
			\end{axis}
		\end{tikzpicture}
		\caption{with $f_{22}$}\label{fig:mp22}
	\end{subfigure} \vspace{2cm}
	\begin{subfigure}{0.32\textwidth}
		\centering
		\begin{tikzpicture}
			\begin{axis}[
				width=1.1\textwidth,
				height=4.75cm,
				xlabel=\scriptsize Iterations,
				ylabel=\scriptsize Fitness value,
				xmin=-2, xmax=52,
				ymin=6, ymax=20,
				yticklabel style = {font=\scriptsize},
				xticklabel style = {font=\scriptsize},
				xtick={0,10,20,30,40,50},   
				ytick={7,11,15,19},grid=both,
				]
				\addplot[mark=halfcircle*,blue,
				mark options={scale=1.5,fill=blue}] table[x=label,y=Diver-f12] {foo1.dat}; 
				\addplot[mark=halfcircle*,magenta,
				mark options={scale=1.5,fill=red}] table[x=label,y=Conver-f12] {foo1.dat}; 
			\end{axis}
		\end{tikzpicture}
		\caption{with $f_{23}$}\label{fig:mp23}
	\end{subfigure} 
    \vspace{-2cm}
	\caption{Convergence and Divergence dynamics of IWSO across 23 benchmark functions: Illustrating its adaptive balance between exploration and exploitation during optimization}
	\label{fig:mp}
\end{figure*}

%% file: SD.tex
\begin{figure*}[!htbp]
    \centering
    \begin{subfigure}{0.16\textwidth}
        \centering
        \begin{tikzpicture}
\begin{axis} [
    width=0.99\textwidth,
    height = 11cm, 
    xbar = .1cm,
    bar width = 1.0pt,
    xmin = 0,
    xmax = 3,
    ytick = {1,2,3,4,5},
    yticklabels = {20,40,60,80,100},
    ymax=5.25, ymin=0.75,
    label style={font=\scriptsize},
    ticklabel style = {font=\scriptsize},
    ylabel={\#Epochs},
    xlabel={SD},
    enlarge y limits  = 0.05,
    enlarge x limits  = 0.05,
    legend columns=11,
    legend style={at={(6.50,1.01)},anchor=south, legend cell align=left, font=\scriptsize},
]

\addplot [draw = red, fill = red] coordinates {
(0.8617,1) (0.9084,2) (0.9351,3) (0.9154,4) (0.9246,5)};
\addlegendentry{GA}; 

\addplot [draw = green, fill = green] coordinates {
(0.4833,1) (0.5628,2) (0.5273,3) (0.5334,4) (0.5473,5)};
\addlegendentry{ABC}; 

\addplot [draw = pink, fill = pink] coordinates {
(2.4558,1) (2.3201,2) (2.4384,3) (2.5038,4) (2.4384,5)};
\addlegendentry{PSO}; 

\addplot [draw = blue, fill = blue] coordinates {
(0.1756,1) (0.1867,2) (0.1786,3) (0.1688,4) (0.1889,5)};
\addlegendentry{ACO}; 

\addplot [draw = yellow, fill = yellow] coordinates {
(0.0001,1) (0.0001,2) (0.0001,3) (0.0001,4) (0.0001,5)};
\addlegendentry{WOA}; 

\addplot [draw = brown, fill = brown] coordinates {
(0.1258,1) (0.1532,2) (0.1324,3) (0.1446,4) (0.1441,5)};
\addlegendentry{CS}; 

\addplot [draw = cyan, fill = cyan] coordinates {
(0.1689,1) (0.1677,2) (0.1706,3) (0.1869,4) (0.1868,5)};
\addlegendentry{SA}; 

\addplot [draw = lime, fill = lime] coordinates {
(1.4112,1) (1.3881,2) (1.3371,3) (1.4028,4) (1.3871,5)};
\addlegendentry{BHO}; 

\addplot [draw = purple, fill = purple] coordinates {
(0.1534,1) (0.1710,2) (0.1579,3) (0.1575,4) (0.1511,5)};
\addlegendentry{HS}; 

\addplot [draw = violet, fill = violet] coordinates {
(0.3971,1) (0.3893,2) (0.3826,3) (0.4010,4) (0.4304,5)};
\addlegendentry{DE}; 

\addplot [draw = orange, fill = orange] coordinates {
(0.4295,1) (0.2515,2) (0.2875,3) (0.1215,4) (0.3377,5)};
\addlegendentry{\textbf{IWSO}}; 

\end{axis}
\end{tikzpicture}
        \caption{$f_1$}
    \end{subfigure} 
    \hfill
    \begin{subfigure}{0.16\textwidth}
        \centering
        \begin{tikzpicture}
\begin{axis} [
    width=0.99\textwidth,
    height = 11cm, 
    xbar = .1cm,
    bar width = 1pt,
    xmin = 0,
    xmax = 9,
    ytick = {1,2,3,4,5},
    yticklabels = {20,40,60,80,100},
         ymax=5.25, ymin=0.75,
    label style={font=\scriptsize},
    ticklabel style = {font=\scriptsize},
    ylabel={\#Epochs},
    xlabel={SD},
    enlarge y limits  = 0.05,
    enlarge x limits  = 0.05
]

\addplot [draw = red, fill = red] coordinates {
(0.5005,1) (0.5781,2) (0.5648,3) (0.5547,4) (0.5664,5)};
\addplot [draw = green, fill = green] coordinates {
(0.0048,1) (0.0052,2) (0.0062,3) (0.0066,4) (0.0057,5)};
\addplot [draw = pink, fill = pink] coordinates {
(0,1) (0,2) (0,3) (0,4) (0,5)};
\addplot [draw = blue, fill = blue] coordinates {
(1.1906,1) (1.1975,2) (1.1346,3) (1.1473,4) (1.2233,5)};
\addplot [draw = yellow, fill = yellow] coordinates {
(0,1) (0,2) (0,3) (0,4) (0,5)};
\addplot [draw = brown, fill = brown] coordinates {
(1.8674,1) (1.9051,2) (1.8871,3) (1.8900,4) (1.8542,5)};
\addplot [draw = cyan, fill = cyan] coordinates {
(8.0831,1) (7.8395,2) (8.3254,3) (8.1585,4) (8.4792,5)};
\addplot [draw = lime, fill = lime] coordinates {
(0,1) (0.0001,2) (0,3) (0,4) (0.0001,5)};
\addplot [draw = purple, fill = purple] coordinates {
(5.2718,1) (5.3494,2) (5.4961,3) (5.3502,4) (5.2128,5)};
\addplot [draw = violet, fill = violet] coordinates {
(0.0019,1) (0.0024,2) (0.0022,3) (0.0023,4) (0.0022,5)};
\addplot [draw = orange, fill = orange] coordinates {
(1.0102,1) (0.1106,2) (0.2170,3) (0.0831,4) (0.0792,5)};
\end{axis}
\end{tikzpicture}
        \caption{$f_2$}
    \end{subfigure}
    \hfill
    \begin{subfigure}{0.16\textwidth}
        \centering
        \begin{tikzpicture}
\begin{axis} [
    width=0.99\textwidth,
    height = 11cm, 
    xbar = .1cm,
    bar width = 1pt,
    xmin = 0,
    xmax = 1,
    ytick = {1,2,3,4,5},
    yticklabels = {20,40,60,80,100},
         ymax=5.25, ymin=0.75,
    label style={font=\scriptsize},
    ticklabel style = {font=\scriptsize},
    ylabel={\#Epochs},
    xlabel={SD},
    enlarge y limits  = 0.05,
    enlarge x limits  = 0.05
]

\addplot [draw = red, fill = red] coordinates {
(0.1223,1) (0.1296,2) (0.1471,3) (0.1402,4) (0.1307,5)};
\addplot [draw = green, fill = green] coordinates {
(0.0005,1) (0.0006,2) (0.0005,3) (0.0008,4) (0.0011,5)};
\addplot [draw = pink, fill = pink] coordinates {
(0,1) (0,2) (0,3) (0,4) (0,5)};
\addplot [draw = blue, fill = blue] coordinates {
(0.2464,1) (0.2468,2) (0.2894,3) (0.2641,4) (0.2373,5)};
\addplot [draw = yellow, fill = yellow] coordinates {
(0.8024,1) (0.7851,2) (0.6809,3) (0.7844,4) (0.8123,5)};
\addplot [draw = brown, fill = brown] coordinates {
(0.7338,1) (0.9344,2) (0.8959,3) (0.8472,4) (0.8960,5)};
\addplot [draw = cyan, fill = cyan] coordinates {
(0.0041,1) (0.0040,2) (0.0042,3) (0.0045,4) (0.0043,5)};
\addplot [draw = lime, fill = lime] coordinates {
(0.0041,1) (0.0018,2) (0.0014,3) (0.0021,4) (0.0019,5)};
\addplot [draw = purple, fill = purple] coordinates {
(0.0117,1) (0.0129,2) (0.0127,3) (0.0098,4) (0.0119,5)};
\addplot [draw = violet, fill = violet] coordinates {
(0,1) (0,2) (0,3) (0,4) (0,5)};
\addplot [draw = orange, fill = orange] coordinates {
(0.0867,1) (0.0091,2) (0.0048,3) (0.0026,4) (0.0019,5)};
\end{axis}
\end{tikzpicture}
        \caption{$f_3$}
    \end{subfigure} 
    \hfill
    \begin{subfigure}{0.16\textwidth}
        \centering
        \begin{tikzpicture}
\begin{axis} [
    width=0.99\textwidth,
    height = 11cm, 
    xbar = .1cm,
    bar width = 1pt,
    xmin = 0,
    xmax = 300,
    xtick = {0,130,260},
    ytick = {1,2,3,4,5},
    yticklabels = {20,40,60,80,100},
         ymax=5.25, ymin=0.75,
    label style={font=\scriptsize},
    ticklabel style = {font=\scriptsize},
    ylabel={\#Epochs},
    xlabel={SD},
    enlarge y limits  = 0.05,
    enlarge x limits  = 0.05
]

\addplot [draw = red, fill = red] coordinates {
(40.8281,1) (39.6382,2) (38.8865,3) (39.6891,4) (38.6631,5)};
\addplot [draw = green, fill = green] coordinates {
(18.9914,1) (19.3429,2) (19.3304,3) (19.1417,4) (18.5244,5)};
\addplot [draw = pink, fill = pink] coordinates {
(85.1871,1) (81.9263,2) (80.9232,3) (80.9578,4) (80.9232,5)};
\addplot [draw = blue, fill = blue] coordinates {
(13.9255,1) (14.9833,2) (15.2932,3) (14.5014,4) (15.0626,5)};
\addplot [draw = yellow, fill = yellow] coordinates {
(213.4524,1) (231.7765,2) (225.3172,3) (224.2487,4) (226.7169,5)};
\addplot [draw = brown, fill = brown] coordinates {
(12.5461,1) (13.6966,2) (13.0725,3) (12.8127,4) (13.8288,5)};
\addplot [draw = cyan, fill = cyan] coordinates {
(20.2528,1) (20.7516,2) (18.0643,3) (18.8968,4) (17.9240,5)};
\addplot [draw = lime, fill = lime] coordinates {
(60.5248,1) (60.8415,2) (62.4478,3) (66.0211,4) (65.8817,5)};
\addplot [draw = purple, fill = purple] coordinates {
(18.5826,1) (18.6333,2) (19.0626,3) (20.2957,4) (19.2608,5)};
\addplot [draw = violet, fill = violet] coordinates {
(18.7272,1) (20.6141,2) (20.5461,3) (20.6437,4) (21.0199,5)};
\addplot [draw = orange, fill = orange] coordinates {
(125.4941,1) (230.1400,2) (264.7624,3) (216.4998,4) (163.5968,5)};
\end{axis}
\end{tikzpicture}
        \caption{$f_4$}
    \end{subfigure}   \hfill
      \begin{subfigure}{0.16\textwidth}
        \centering
        \begin{tikzpicture}
\begin{axis} [
    width=0.99\textwidth,
    height = 11cm, 
    xbar = .1cm,
    bar width = 1pt,
    xmin = 0,
    xmax = 335000,
    ytick = {1,2,3,4,5},
    yticklabels = {20,40,60,80,100},
         ymax=5.25, ymin=0.75,
    label style={font=\scriptsize},
    ticklabel style = {font=\scriptsize},
    ylabel={\#Epochs},
    xlabel={SD},
    enlarge y limits  = 0.05,
    enlarge x limits  = 0.05,
    legend columns=11,
    legend style={at={(3.625,1.01)},anchor=south, legend cell align=left, font=\scriptsize},
]

\addplot [draw = red, fill = red] coordinates {
(5228.1963,1) (5431.6578,2) (5391.5211,3) (5438.2527,4) (5104.4134,5)};
\addplot [draw = green, fill = green] coordinates {
(149199.6059,1) (152443.1913,2) (157774.8573,3) (149209.4558,4) (151504.3643,5)};
\addplot [draw = pink, fill = pink] coordinates {
(3599.8288,1) (9983.9485,2) (11636.9484,3) (9923.4200,4) (11636.9484,5)};
\addplot [draw = blue, fill = blue] coordinates {
(153565.1373,1) (167222.0790,2) (166096.4214,3) (166380.1508,4) (173322.5773,5)};
\addplot [draw = yellow, fill = yellow] coordinates {
(0.0420,1) (0.0957,2) (0.0810,3) (0.1020,4) (0.0813,5)};
\addplot [draw = brown, fill = brown] coordinates {
(224503.6961,1) (196639.8362,2) (214452.7476,3) (205419.4166,4) (210488.0318,5)};
\addplot [draw = cyan, fill = cyan] coordinates {
(406.9016,1) (420.6461,2) (388.4646,3) (398.2857,4) (397.1096,5)};
\addplot [draw = lime, fill = lime] coordinates {
(1424.8536,1) (1470.1065,2) (1485.5171,3) (1514.0463,4) (1479.8152,5)};
\addplot [draw = purple, fill = purple] coordinates {
(288384.7884,1) (256855.9132,2) (283375.1233,3) (272438.1047,4) (270641.0690,5)};
\addplot [draw = violet, fill = violet] coordinates {
(292732.5317,1) (312578.6020,2) (322790.6455,3) (317071.0431,4) (314916.0420,5)};  
\addplot [draw = orange, fill = orange] coordinates {
(55.6528,1) (19.1988,2) (9.8714,3) (5.7059,4) (4.8065,5)}; 
\end{axis}
\end{tikzpicture}
        \caption{$f_5$}
    \end{subfigure} 
    \hfill
    \begin{subfigure}{0.16\textwidth}
        \centering
        \begin{tikzpicture}
\begin{axis} [
    width=0.99\textwidth,
    height = 11cm, 
    xbar = .1cm,
    bar width = 1pt,
    xmin = 0,
    xmax = 9,
    ytick = {1,2,3,4,5},
    yticklabels = {20,40,60,80,100},
         ymax=5.25, ymin=0.75,
    label style={font=\scriptsize},
    ticklabel style = {font=\scriptsize},
    ylabel={\#Epochs},
    xlabel={SD},
    enlarge y limits  = 0.05,
    enlarge x limits  = 0.05
]

\addplot [draw = red, fill = red] coordinates {
(1.5069,1) (1.7250,2) (1.6007,3) (1.6528,4) (1.9445,5)};
\addplot [draw = green, fill = green] coordinates {
(2.0449,1) (2.0649,2) (2.1786,3) (2.2707,4) (2.0913,5)};
\addplot [draw = pink, fill = pink] coordinates {
(3.2145,1) (3.3274,2) (3.3187,3) (3.1805,4) (3.3187,5)};
\addplot [draw = blue, fill = blue] coordinates {
(1.8669,1) (1.7527,2) (1.7751,3) (1.7862,4) (1.7847,5)};
\addplot [draw = yellow, fill = yellow] coordinates {
(3.7675,1) (4.1796,2) (4.1386,3) (4.3008,4) (3.9684,5)};
\addplot [draw = brown, fill = brown] coordinates {
(4.1989,1) (4.4091,2) (3.9150,3) (4.1271,4) (4.2179,5)};
\addplot [draw = cyan, fill = cyan] coordinates {
(8.1648,1) (8.1004,2) (8.1310,3) (8.0445,4) (8.0860,5)};
\addplot [draw = lime, fill = lime] coordinates {
(2.7902,1) (3.3766,2) (3.5481,3) (3.2904,4) (3.3053,5)};
\addplot [draw = purple, fill = purple] coordinates {
(2.4083,1) (2.2781,2) (2.3292,3) (2.4312,4) (2.4275,5)};
\addplot [draw = violet, fill = violet] coordinates {
(1.1608,1) (1.3651,2) (1.4054,3) (1.3035,4) (1.3478,5)};
\addplot [draw = orange, fill = orange] coordinates {
(1.8801,1) (4.4455,2) (2.3329,3) (3.5683,4) (0.3143,5)};
\end{axis}
\end{tikzpicture}
        \caption{$f_6$}
    \end{subfigure}
\\
\vspace{0.25cm}
    \begin{subfigure}{0.16\textwidth}
        \centering
        \begin{tikzpicture}
\begin{axis} [
    width=0.99\textwidth,
    height = 11cm, 
    xbar = .1cm,
    bar width = 1pt,
    xmin = 0,
    xmax = 140,
    ytick = {1,2,3,4,5},
    yticklabels = {20,40,60,80,100},
         ymax=5.25, ymin=0.75,
    label style={font=\scriptsize},
    ticklabel style = {font=\scriptsize},
    ylabel={\#Epochs},
    xlabel={SD},
    enlarge y limits  = 0.05,
    enlarge x limits  = 0.05
]

\addplot [draw = red, fill = red] coordinates {
(5.8317,1) (6.6042,2) (6.6543,3) (6.4888,4) (6.6228,5)};
\addplot [draw = green, fill = green] coordinates {
(39.4688,1) (40.1820,2) (39.2871,3) (39.5084,4) (39.4463,5)};
\addplot [draw = pink, fill = pink] coordinates {
(6.7693,1) (5.1600,2) (6.2558,3) (6.3814,4) (6.2558,5)};
\addplot [draw = blue, fill = blue] coordinates {
(35.9020,1) (38.6601,2) (37.3035,3) (37.6933,4) (39.2029,5)};
\addplot [draw = yellow, fill = yellow] coordinates {
(0.1267,1) (0.1328,2) (0.1843,3) (0.1960,4) (0.1682,5)};
\addplot [draw = brown, fill = brown] coordinates {
(42.6818,1) (44.2678,2) (43.7290,3) (42.8069,4) (44.8478,5)};
\addplot [draw = cyan, fill = cyan] coordinates {
(126.4001,1) (129.8258,2) (127.0622,3) (128.6777,4) (134.8695,5)};
\addplot [draw = lime, fill = lime] coordinates {
(5.3485,1) (5.4819,2) (5.3532,3) (5.1698,4) (5.2981,5)};
\addplot [draw = purple, fill = purple] coordinates {
(62.7683,1) (63.4419,2) (61.1595,3) (65.5141,4) (66.1017,5)};
\addplot [draw = violet, fill = violet] coordinates {
(64.5785,1) (70.2522,2) (68.3983,3) (68.9129,4) (70.4618,5)};
\addplot [draw = orange, fill = orange] coordinates {
(0.7949,1) (0.1976,2) (0.1861,3) (0.2935,4) (0.2082,5)};
\end{axis}
\end{tikzpicture}
        \caption{$f_7$}
    \end{subfigure} 
    \hfill
    \begin{subfigure}{0.16\textwidth}
        \centering
        \begin{tikzpicture}
\begin{axis} [
    width=0.99\textwidth,
    height = 11cm, 
    xbar = .1cm,
    bar width = 1pt,
    xmin = 0,
    xmax = 40,
    xtick = {0,15,30,45},
    ytick = {1,2,3,4,5},
    yticklabels = {20,40,60,80,100},
         ymax=5.25, ymin=0.75,
    label style={font=\scriptsize},
    ticklabel style = {font=\scriptsize},
    ylabel={\#Epochs},
    xlabel={SD},
    enlarge y limits  = 0.05,
    enlarge x limits  = 0.05
]

\addplot [draw = red, fill = red] coordinates {
(7.9527,1) (8.2003,2) (8.1856,3) (8.5303,4) (8.5804,5)};
\addplot [draw = green, fill = green] coordinates {
(25.9760,1) (27.5140,2) (28.4262,3) (27.2094,4) (28.4780,5)};
\addplot [draw = pink, fill = pink] coordinates {
(6.2339,1) (6.7718,2) (6.8107,3) (6.5611,4) (6.8107,5)};
\addplot [draw = blue, fill = blue] coordinates {
(14.2385,1) (16.4811,2) (16.3321,3) (16.6355,4) (17.1866,5)};
\addplot [draw = yellow, fill = yellow] coordinates {
(0.3915,1) (0.3741,2) (0.3930,3) (0.4289,4) (0.4028,5)};
\addplot [draw = brown, fill = brown] coordinates {
(18.3403,1) (17.9266,2) (19.8998,3) (19.6533,4) (19.7731,5)};
\addplot [draw = cyan, fill = cyan] coordinates {
(24.8282,1) (24.7728,2) (25.8966,3) (26.5138,4) (25.6135,5)};
\addplot [draw = lime, fill = lime] coordinates {
(6.1001,1) (6.5424,2) (6.0324,3) (6.6965,4) (6.5681,5)};
\addplot [draw = purple, fill = purple] coordinates {
(20.1252,1) (21.0146,2) (20.0869,3) (19.6342,4) (20.1240,5)};
\addplot [draw = violet, fill = violet] coordinates {
(36.5804,1) (38.5722,2) (37.7062,3) (37.7686,4) (37.6358,5)};
\addplot [draw = orange, fill = orange] coordinates {
(0.3238,1) (0.2911,2) (0.1687,3) (0.1020,4) (0.1490,5)};
\end{axis}
\end{tikzpicture}
        \caption{$f_8$}
    \end{subfigure}   \hfill
    \begin{subfigure}{0.16\textwidth}
        \centering
        \begin{tikzpicture}
\begin{axis} [
    width=0.99\textwidth,
    height = 11cm, 
    xbar = .1cm,
    bar width = 1pt,
    xmin = 0,
    xmax = 1,
    ytick = {1,2,3,4,5},
    yticklabels = {20,40,60,80,100},
         ymax=5.25, ymin=0.75,
    label style={font=\scriptsize},
    ticklabel style = {font=\scriptsize},
    ylabel={\#Epochs},
    xlabel={SD},
    enlarge y limits  = 0.05,
    enlarge x limits  = 0.05,
    legend columns=6,
    legend style={at={(3.625,1.01)},anchor=south, legend cell align=left, font=\scriptsize},
]

\addplot [draw = red, fill = red] coordinates {
(0.4071,1) (0.4186,2) (0.4341,3) (0.4278,4) (0.4324,5)};
\addplot [draw = green, fill = green] coordinates {
(0.4030,1) (0.4534,2) (0.4292,3) (0.4361,4) (0.4441,5)};
\addplot [draw = pink, fill = pink] coordinates {
(0.8899,1) (0.8622,2) (0.9044,3) (0.8879,4) (0.9044,5)};
\addplot [draw = blue, fill = blue] coordinates {
(0.3719,1) (0.3683,2) (0.3749,3) (0.3701,4) (0.3745,5)};
\addplot [draw = yellow, fill = yellow] coordinates {
(0.5770,1) (0.5840,2) (0.5431,3) (0.5872,4) (0.5690,5)};
\addplot [draw = brown, fill = brown] coordinates {
(0.3285,1) (0.3426,2) (0.3408,3) (0.3384,4) (0.3453,5)};
\addplot [draw = cyan, fill = cyan] coordinates {
(0.3772,1) (0.4063,2) (0.4155,3) (0.4417,4) (0.4096,5)};
\addplot [draw = lime, fill = lime] coordinates {
(0.6861,1) (0.7015,2) (0.7546,3) (0.7234,4) (0.7200,5)};
\addplot [draw = purple, fill = purple] coordinates {
(0.4097,1) (0.3926,2) (0.3923,3) (0.3892,4) (0.3985,5)};
\addplot [draw = violet, fill = violet] coordinates {
(0.4031,1) (0.4110,2) (0.4100,3) (0.3950,4) (0.4038,5)};
\addplot [draw = orange, fill = orange] coordinates {
(0.4167,1) (0.6538,2) (0.4944,3) (0.5219,4) (0.3367,5)};
\end{axis}
\end{tikzpicture}
        \caption{$f_9$}
    \end{subfigure} 
    \hfill
    \begin{subfigure}{0.16\textwidth}
        \centering
        \begin{tikzpicture}
\begin{axis} [
    width=0.99\textwidth,
    height = 11cm, 
    xbar = .1cm,
    bar width = 1pt,
    xmin = 0,
    xmax = 12,
    ytick = {1,2,3,4,5},
    yticklabels = {20,40,60,80,100},
         ymax=5.25, ymin=0.75,
    label style={font=\scriptsize},
    ticklabel style = {font=\scriptsize},
    ylabel={\#Epochs},
    xlabel={SD},
    enlarge y limits  = 0.05,
    enlarge x limits  = 0.05
]

\addplot [draw = red, fill = red] coordinates {
(3.0214,1) (3.0446,2) (3.0911,3) (3.1044,4) (3.1107,5)};
\addplot [draw = green, fill = green] coordinates {
(4.6551,1) (4.5384,2) (4.4042,3) (4.3895,4) (4.3717,5)};
\addplot [draw = pink, fill = pink] coordinates {
(1.4507,1) (1.9027,2) (1.6348,3) (1.9080,4) (1.6348,5)};
\addplot [draw = blue, fill = blue] coordinates {
(2.5174,1) (2.8708,2) (2.7475,3) (2.7038,4) (2.7757,5)};
\addplot [draw = yellow, fill = yellow] coordinates {
(7.9282,1) (8.2609,2) (8.4209,3) (8.0481,4) (8.1575,5)};
\addplot [draw = brown, fill = brown] coordinates {
(2.6575,1) (2.7560,2) (2.8940,3) (2.8982,4) (2.8357,5)};
\addplot [draw = cyan, fill = cyan] coordinates {
(4.9649,1) (4.7245,2) (4.5118,3) (4.7959,4) (4.8567,5)};
\addplot [draw = lime, fill = lime] coordinates {
(4.1239,1) (4.5077,2) (4.5610,3) (4.3614,4) (4.3176,5)};
\addplot [draw = purple, fill = purple] coordinates {
(3.4665,1) (3.7399,2) (3.3312,3) (3.5362,4) (3.6230,5)};
\addplot [draw = violet, fill = violet] coordinates {
(3.6692,1) (3.6216,2) (3.6559,3) (3.8211,4) (3.7599,5)};
\addplot [draw = orange, fill = orange] coordinates {
(11.4434,1) (10.4410,2) (10.1929,3) (5.6760,4) (10.6141,5)};
\end{axis}
\end{tikzpicture}
        \caption{$f_{10}$}
    \end{subfigure}
    \hfill
    \begin{subfigure}{0.16\textwidth}
        \centering
        \begin{tikzpicture}
\begin{axis} [
    width=0.99\textwidth,
    height = 11cm, 
    xbar = .1cm,
    bar width = 1pt,
    xmin = 0,
    xmax = 20,
    xtick = {0,6,12,18},
    ytick = {1,2,3,4,5},
    yticklabels = {20,40,60,80,100},
         ymax=5.25, ymin=0.75,
    label style={font=\scriptsize},
    ticklabel style = {font=\scriptsize},
    ylabel={\#Epochs},
    xlabel={SD},
    enlarge y limits  = 0.05,
    enlarge x limits  = 0.05
]

\addplot [draw = red, fill = red] coordinates {
(1.7132,1) (1.7853,2) (1.8017,3) (1.8292,4) (1.7961,5)};
\addplot [draw = green, fill = green] coordinates {
(10.5488,1) (10.9319,2) (10.8470,3) (11.1098,4) (11.3335,5)};
\addplot [draw = pink, fill = pink] coordinates {
(1.2101,1) (1.7466,2) (2.0175,3) (1.6694,4) (2.0175,5)};
\addplot [draw = blue, fill = blue] coordinates {
(10.5067,1) (10.5771,2) (10.3210,3) (10.7032,4) (10.6124,5)};
\addplot [draw = yellow, fill = yellow] coordinates {
(0.0989,1) (0.1021,2) (0.0993,3) (0.0984,4) (0.0992,5)};
\addplot [draw = brown, fill = brown] coordinates {
(12.5829,1) (12.1018,2) (13.0031,3) (12.3243,4) (12.6845,5)};
\addplot [draw = cyan, fill = cyan] coordinates {
(1.7327,1) (1.7279,2) (1.7984,3) (1.7755,4) (1.7670,5)};
\addplot [draw = lime, fill = lime] coordinates {
(1.2876,1) (1.4783,2) (1.4142,3) (1.3605,4) (1.4087,5)};
\addplot [draw = purple, fill = purple] coordinates {
(15.6552,1) (16.1446,2) (15.3509,3) (15.1983,4) (15.9877,5)};
\addplot [draw = violet, fill = violet] coordinates {
(17.3292,1) (18.4640,2) (18.7118,3) (18.4990,4) (18.8642,5)};
\addplot [draw = orange, fill = orange] coordinates {
(1.0600,1) (0.5017,2) (0.2237,3) (0.2309,4) (0.1210,5)};
\end{axis}
\end{tikzpicture}
        \caption{$f_{11}$}
    \end{subfigure} 
    \hfill
    \begin{subfigure}{0.16\textwidth}
        \centering
        \begin{tikzpicture}
\begin{axis} [
    width=0.99\textwidth,
    height = 11cm, 
    xbar = .1cm,
    bar width = 1pt,
    xmin = 0,
    xmax = 20,
    xtick = {0,6,12,18},
    ytick = {1,2,3,4,5},
    yticklabels = {20,40,60,80,100},
         ymax=5.25, ymin=0.75,
    label style={font=\scriptsize},
    ticklabel style = {font=\scriptsize},
    ylabel={\#Epochs},
    xlabel={SD},
    enlarge y limits  = 0.05,
    enlarge x limits  = 0.05
]

\addplot [draw = red, fill = red] coordinates {
(1.8196,1) (1.9199,2) (1.8767,3) (1.8012,4) (1.9109,5)};
\addplot [draw = green, fill = green] coordinates {
(11.2236,1) (11.1748,2) (10.5802,3) (10.8655,4) (11.0226,5)};
\addplot [draw = pink, fill = pink] coordinates {
(1.6961,1) (1.7982,2) (1.9096,3) (1.4859,4) (1.9096,5)};
\addplot [draw = blue, fill = blue] coordinates {
(10.3818,1) (10.4495,2) (11.4663,3) (10.7920,4) (10.5402,5)};
\addplot [draw = yellow, fill = yellow] coordinates {
(0.0885,1) (0.1028,2) (0.0995,3) (0.1067,4) (0.0976,5)};
\addplot [draw = brown, fill = brown] coordinates {
(11.6885,1) (12.2139,2) (12.0687,3) (12.8782,4) (12.5704,5)};
\addplot [draw = cyan, fill = cyan] coordinates {
(1.6959,1) (1.7489,2) (1.7609,3) (1.7441,4) (1.7623,5)};
\addplot [draw = lime, fill = lime] coordinates {
(1.3529,1) (1.4265,2) (1.4129,3) (1.4254,4) (1.3791,5)};
\addplot [draw = purple, fill = purple] coordinates {
(14.5630,1) (14.9534,2) (15.9135,3) (15.3960,4) (15.1911,5)};
\addplot [draw = violet, fill = violet] coordinates {
(18.5139,1) (18.4567,2) (19.5009,3) (19.4138,4) (19.1673,5)};
\addplot [draw = orange, fill = orange] coordinates {
(1.2314,1) (0.3657,2) (0.3512,3) (0.3198,4) (0.1593,5)};
\end{axis}
\end{tikzpicture}
        \caption{$f_{12}$}
    \end{subfigure}   
\end{figure*}
\begin{figure*}[!htbp]
    \centering
    \ContinuedFloat
\vspace{0.25cm}
    \begin{subfigure}{0.16\textwidth}
        \centering
        \begin{tikzpicture}
\begin{axis} [
    width=0.99\textwidth,
    height = 11cm, 
    xbar = .1cm,
    bar width = 1pt,
    xmin = 0,
    xmax = 0.7,
    ytick = {1,2,3,4,5},
    yticklabels = {20,40,60,80,100},
         ymax=5.25, ymin=0.75,
    label style={font=\scriptsize},
    ticklabel style = {font=\scriptsize},
    ylabel={\#Epochs},
    xlabel={SD},
    enlarge y limits  = 0.05,
    enlarge x limits  = 0.05,
    legend columns=11,
    legend style={at={(6.50,1.01)},anchor=south, legend cell align=left, font=\scriptsize},
]

\addplot [draw = red, fill = red] coordinates {
(0.0040,1) (0.0038,2) (0.0037,3) (0.0039,4) (0.0039,5)}; \addlegendentry{GA};
\addplot [draw = green, fill = green] coordinates {
(0.2897,1) (0.3058,2) (0.3177,3) (0.3096,4) (0.3079,5)};\addlegendentry{ABC}; 
\addplot [draw = pink, fill = pink] coordinates {
(0,1) (0,2) (0,3) (0,4) (0,5)}; \addlegendentry{PSO};
\addplot [draw = blue, fill = blue] coordinates {
(0.1183,1) (0.1257,2) (0.1265,3) (0.1284,4) (0.1259,5)}; \addlegendentry{ACO};
\addplot [draw = yellow, fill = yellow] coordinates {
(0,1) (0,2) (0,3) (0,4) (0,5)}; \addlegendentry{WOA};
\addplot [draw = brown, fill = brown] coordinates {
(0.1561,1) (0.1648,2) (0.1707,3) (0.1644,4) (0.1760,5)};\addlegendentry{CS};
\addplot [draw = cyan, fill = cyan] coordinates {
(0.6283,1) (0.6148,2) (0.6322,3) (0.6082,4) (0.6071,5)}; \addlegendentry{SA};
\addplot [draw = lime, fill = lime] coordinates {
(0.0004,1) (0.0006,2) (0.0005,3) (0.0005,4) (0.0005,5)};\addlegendentry{BHO};
\addplot [draw = purple, fill = purple] coordinates {
(0.0746,1) (0.0886,2) (0.0806,3) (0.0861,4) (0.0853,5)};\addlegendentry{HS};
\addplot [draw = violet, fill = violet] coordinates {
(0.3436,1) (0.3639,2) (0.3588,3) (0.3780,4) (0.3574,5)}; \addlegendentry{DE}; 
\addplot [draw = orange, fill = orange] coordinates {
(0.0786,1) (0.0114,2) (0.0125,3) (0.0096,4) (0.0092,5)};\addlegendentry{IWSO};
\end{axis}
\end{tikzpicture}
        \caption{$f_{13}$}
    \end{subfigure} 
    \hfill
    \begin{subfigure}{0.16\textwidth}
        \centering
        \begin{tikzpicture}
\begin{axis} [
    width=0.99\textwidth,
    height = 11cm, 
    xbar = .1cm,
    bar width = 1pt,
    xmin = 0,
    xmax = 50,
    xtick = {0,15,30,45},
    ytick = {1,2,3,4,5},
    yticklabels = {20,40,60,80,100},
         ymax=5.25, ymin=0.75,
    label style={font=\scriptsize},
    ticklabel style = {font=\scriptsize},
    ylabel={\#Epochs},
    xlabel={SD},
    enlarge y limits  = 0.05,
    enlarge x limits  = 0.05
]

\addplot [draw = red, fill = red] coordinates {
(18.2543,1) (19.0669,2) (19.1506,3) (18.9957,4) (19.4699,5)};
\addplot [draw = green, fill = green] coordinates {
(19.4008,1) (21.2460,2) (21.7754,3) (21.5718,4) (20.5906,5)};
\addplot [draw = pink, fill = pink] coordinates {
(29.1310,1) (28.1963,2) (28.8826,3) (28.9160,4) (28.8826,5)};
\addplot [draw = blue, fill = blue] coordinates {
(17.2710,1) (17.5478,2) (18.3534,3) (17.6834,4) (17.0810,5)};
\addplot [draw = yellow, fill = yellow] coordinates {
(0,1) (0.6018,2) (0.0004,3) (0.1472,4) (0.0009,5)};
\addplot [draw = brown, fill = brown] coordinates {
(18.2454,1) (18.4074,2) (17.5761,3) (18.2006,4) (18.4242,5)};
\addplot [draw = cyan, fill = cyan] coordinates {
(39.1156,1) (38.7731,2) (38.2777,3) (37.1370,4) (38.5708,5)};
\addplot [draw = lime, fill = lime] coordinates {
(23.2919,1) (22.0580,2) (21.5317,3) (23.9675,4) (22.5324,5)};
\addplot [draw = purple, fill = purple] coordinates {
(22.3050,1) (21.4150,2) (19.9748,3) (20.5770,4) (22.0368,5)};
\addplot [draw = violet, fill = violet] coordinates {
(21.9901,1) (22.4437,2) (24.9033,3) (24.3524,4) (23.9827,5)};
\addplot [draw = orange, fill = orange] coordinates {
(11.9584,1) (21.6835,2) (23.9970,3) (19.6376,4) (20.8828,5)};
\end{axis}
\end{tikzpicture}
        \caption{$f_{14}$}
    \end{subfigure}
    \hfill
    \begin{subfigure}{0.16\textwidth}
        \centering
        \begin{tikzpicture}
\begin{axis} [
    width=0.99\textwidth,
    height = 11cm, 
    xbar = .1cm,
    bar width = 1pt,
    xmin = 0,
    xmax = 50,
    xtick = {0,15,30,45},
    ytick = {1,2,3,4,5},
    yticklabels = {20,40,60,80,100},
         ymax=5.25, ymin=0.75,
    label style={font=\scriptsize},
    ticklabel style = {font=\scriptsize},
    ylabel={\#Epochs},
    xlabel={SD},
    enlarge y limits  = 0.05,
    enlarge x limits  = 0.05
]

\addplot [draw = red, fill = red] coordinates {
(17.9430,1) (18.9223,2) (19.2927,3) (19.5474,4) (19.9620,5)};
\addplot [draw = green, fill = green] coordinates {
(21.6497,1) (21.9606,2) (21.3305,3) (21.7584,4) (21.3559,5)};
\addplot [draw = pink, fill = pink] coordinates {
(28.4644,1) (28.1743,2) (27.8672,3) (28.6231,4) (27.8672,5)};
\addplot [draw = blue, fill = blue] coordinates {
(16.6208,1) (18.3544,2) (16.9906,3) (17.6833,4) (17.3407,5)};
\addplot [draw = yellow, fill = yellow] coordinates {
(0,1) (0,2) (0.0002,3) (0,4) (0.0001,5)};
\addplot [draw = brown, fill = brown] coordinates {
(18.6694,1) (18.3204,2) (17.8257,3) (17.6711,4) (18.6513,5)};
\addplot [draw = cyan, fill = cyan] coordinates {
(35.1207,1) (39.9203,2) (37.4481,3) (36.4586,4) (37.7524,5)};
\addplot [draw = lime, fill = lime] coordinates {
(22.7582,1) (22.7040,2) (22.4777,3) (22.8081,4) (22.7209,5)};
\addplot [draw = purple, fill = purple] coordinates {
(21.0224,1) (21.0124,2) (22.0027,3) (22.3678,4) (21.1419,5)};
\addplot [draw = violet, fill = violet] coordinates {
(22.3455,1) (22.5735,2) (22.7032,3) (23.7073,4) (23.5389,5)};
\addplot [draw = orange, fill = orange] coordinates {
(21.6549,1) (14.7260,2) (21.2648,3) (24.0151,4) (21.6206,5)};
\end{axis}
\end{tikzpicture}
        \caption{$f_{15}$}
    \end{subfigure} 
    \hfill
    \begin{subfigure}{0.16\textwidth}
        \centering
        \begin{tikzpicture}
\begin{axis} [
    width=0.99\textwidth,
    height = 11cm, 
    xbar = .1cm,
    bar width = 1pt,
    xmin = 0,
    xmax = 300000,
    ytick = {1,2,3,4,5},
    yticklabels = {20,40,60,80,100},
         ymax=5.25, ymin=0.75,
    label style={font=\scriptsize},
    ticklabel style = {font=\scriptsize},
    ylabel={\#Epochs},
    xlabel={SD},
    enlarge y limits  = 0.05,
    enlarge x limits  = 0.05
]

\addplot [draw = red, fill = red] coordinates {
(4205.9318,1) (4051.4611,2) (3995.7601,3) (3979.5974,4) (3990.5439,5)};
\addplot [draw = green, fill = green] coordinates {
(81284.7501,1) (81652.5410,2) (84996.8485,3) (82282.9500,4) (83138.2728,5)};
\addplot [draw = pink, fill = pink] coordinates {
(19904.7239,1) (21032.0742,2) (21437.4309,3) (21573.1202,4) (21437.4309,5)};
\addplot [draw = blue, fill = blue] coordinates {
(150335.0352,1) (161648.7093,2) (160069.4677,3) (159580.2300,4) (167059.2939,5)};
\addplot [draw = yellow, fill = yellow] coordinates {
(3.9519,1) (5.3458,2) (5.9482,3) (4.7415,4) (5.4503,5)};
\addplot [draw = brown, fill = brown] coordinates {
(214434.2327,1) (209156.3750,2) (208664.3677,3) (206763.5596,4) (214392.0295,5)};
\addplot [draw = cyan, fill = cyan] coordinates {
(1094.9545,1) (1044.3244,2) (1081.2624,3) (1161.7546,4) (1068.5729,5)};
\addplot [draw = lime, fill = lime] coordinates {
(8264.7733,1) (8540.5840,2) (8355.7526,3) (8432.5711,4) (8551.9467,5)};
\addplot [draw = purple, fill = purple] coordinates {
(290347.7734,1) (271047.2576,2) (280576.0374,3) (278247.9284,4) (302528.5830,5)};
\addplot [draw = violet, fill = violet] coordinates {
(129323.5925,1) (117130.3345,2) (126876.5519,3) (123188.9065,4) (127056.7394,5)};
\addplot [draw = orange, fill = orange] coordinates {
(141.8589,1) (49.8146,2) (31.2461,3) (22.9586,4) (15.6390,5)};
\end{axis}
\end{tikzpicture}
        \caption{$f_{16}$}
    \end{subfigure}   \hfill
       \begin{subfigure}{0.16\textwidth}
        \centering
        \begin{tikzpicture}
\begin{axis} [
    width=0.99\textwidth,
    height = 11cm, 
    xbar = .1cm,
    bar width = 1pt,
    xmin = 0,
    xmax = 3,
    ytick = {1,2,3,4,5},
    yticklabels = {20,40,60,80,100},
         ymax=5.25, ymin=0.75,
    label style={font=\scriptsize},
    ticklabel style = {font=\scriptsize},
    ylabel={\#Epochs},
    xlabel={SD},
    enlarge y limits  = 0.05,
    enlarge x limits  = 0.05,
    legend columns=6,
    legend style={at={(3.625,1.01)},anchor=south, legend cell align=left, font=\scriptsize},
]

\addplot [draw = red, fill = red] coordinates {
(0.8113,1) (0.8423,2) (0.8394,3) (0.8159,4) (0.8371,5)};
\addplot [draw = green, fill = green] coordinates {
(1.5025,1) (1.5099,2) (1.5573,3) (1.5211,4) (1.5533,5)};
\addplot [draw = pink, fill = pink] coordinates {
(1.4968,1) (1.4213,2) (1.4378,3) (1.4450,4) (1.4378,5)};
\addplot [draw = blue, fill = blue] coordinates {
(0.9388,1) (0.9933,2) (1.0140,3) (0.9614,4) (1.0211,5)};
\addplot [draw = yellow, fill = yellow] coordinates {
(0.0663,1) (0.0757,2) (0.0732,3) (0.0781,4) (0.0793,5)};
\addplot [draw = brown, fill = brown] coordinates {
(1.0062,1) (0.9858,2) (1.0479,3) (1.0301,4) (1.0799,5)};
\addplot [draw = cyan, fill = cyan] coordinates {
(2.3553,1) (2.6279,2) (2.6222,3) (2.6124,4) (2.7076,5)};
\addplot [draw = lime, fill = lime] coordinates {
(0.9303,1) (0.9697,2) (0.9883,3) (0.9966,4) (0.9903,5)};
\addplot [draw = purple, fill = purple] coordinates {
(1.4578,1) (1.3605,2) (1.4553,3) (1.3869,4) (1.4296,5)};
\addplot [draw = violet, fill = violet] coordinates {
(2.1694,1) (1.8333,2) (1.8334,3) (1.8811,4) (1.9364,5)};
\addplot [draw = orange, fill = orange] coordinates {
(0.4739,1) (0.1212,2) (0.0967,3) (0.0687,4) (0.0671,5)};
\end{axis}
\end{tikzpicture}
        \caption{$f_{17}$}
    \end{subfigure} 
    \hfill
    \begin{subfigure}{0.16\textwidth}
        \centering
        \begin{tikzpicture}
\begin{axis} [
    width=0.99\textwidth,
    height = 11cm, 
    xbar = .1cm,
    bar width = 1pt,
    xmin = 0,
    xmax = 2000,
    ytick = {1,2,3,4,5},
    xtick = {0, 1000, 2000},
    yticklabels = {20,40,60,80,100},
    xticklabels = {0, 1K, 2K},
         ymax=5.25, ymin=0.75,
    label style={font=\scriptsize},
    ticklabel style = {font=\scriptsize},
    ylabel={\#Epochs},
    xlabel={SD},
    enlarge y limits  = 0.05,
    enlarge x limits  = 0.05
]

\addplot [draw = red, fill = red] coordinates {
(434.6017,1) (470.3059,2) (465.2301,3) (475.0237,4) (485.5712,5)};
\addplot [draw = green, fill = green] coordinates {
(354.6610,1) (369.7874,2) (368.4049,3) (365.3680,4) (371.7097,5)};
\addplot [draw = pink, fill = pink] coordinates {
(688.8150,1) (697.9624,2) (670.7961,3) (723.8830,4) (670.7961,5)};
\addplot [draw = blue, fill = blue] coordinates {
(345.8950,1) (341.6932,2) (380.0301,3) (355.9291,4) (358.5627,5)};
\addplot [draw = yellow, fill = yellow] coordinates {
(1049.1994,1) (1056.3128,2) (1052.1810,3) (1065.5998,4) (1000.4259,5)};
\addplot [draw = brown, fill = brown] coordinates {
(341.5618,1) (360.2761,2) (380.5239,3) (375.0409,4) (357.8241,5)};
\addplot [draw = cyan, fill = cyan] coordinates {
(757.6395,1) (860.0120,2) (828.2056,3) (852.5283,4) (847.4366,5)};
\addplot [draw = lime, fill = lime] coordinates {
(545.9197,1) (619.3931,2) (620.0781,3) (578.7868,4) (587.2578,5)};
\addplot [draw = purple, fill = purple] coordinates {
(471.3936,1) (509.1951,2) (561.4950,3) (511.3212,4) (490.6023,5)};
\addplot [draw = violet, fill = violet] coordinates {
(388.7106,1) (404.8463,2) (399.8829,3) (430.4209,4) (422.8490,5)};
\addplot [draw = orange, fill = orange] coordinates {
(1608.3839,1) (78.8823,2) (0.6404,3) (0.3925,4) (0.5243,5)};
\end{axis}
\end{tikzpicture}
        \caption{$f_{18}$}
    \end{subfigure}
\\
\centering
    \begin{subfigure}{0.16\textwidth}
        \centering
        \begin{tikzpicture}
\begin{axis} [
    width=0.99\textwidth,
    height = 11cm, 
    xbar = .1cm,
    bar width = 1pt,
    xmin = 0,
    xmax = 2.5,
    ytick = {1,2,3,4,5},
    yticklabels = {20,40,60,80,100},
         ymax=5.25, ymin=0.75,
    label style={font=\scriptsize},
    ticklabel style = {font=\scriptsize},
    ylabel={\#Epochs},
    xlabel={SD},
    enlarge y limits  = 0.05,
    enlarge x limits  = 0.05
]

\addplot [draw = red, fill = red] coordinates {
(1.2952,1) (1.3309,2) (1.3208,3) (1.2943,4) (1.3255,5)};
\addplot [draw = green, fill = green] coordinates {
(1.3276,1) (1.3553,2) (1.4019,3) (1.3861,4) (1.3878,5)};
\addplot [draw = pink, fill = pink] coordinates {
(1.9531,1) (1.9663,2) (2.0071,3) (1.9833,4) (2.0071,5)};
\addplot [draw = blue, fill = blue] coordinates {
(1.1590,1) (1.2114,2) (1.2204,3) (1.2843,4) (1.2071,5)};
\addplot [draw = yellow, fill = yellow] coordinates {
(0.1105,1) (0.0909,2) (0.0968,3) (0.1176,4) (0.1092,5)};
\addplot [draw = brown, fill = brown] coordinates {
(1.2604,1) (1.3326,2) (1.3265,3) (1.3760,4) (1.3373,5)};
\addplot [draw = cyan, fill = cyan] coordinates {
(2.4238,1) (2.3198,2) (2.4583,3) (2.3338,4) (2.4324,5)};
\addplot [draw = lime, fill = lime] coordinates {
(1.8597,1) (1.8248,2) (1.7173,3) (1.7714,4) (1.8382,5)};
\addplot [draw = purple, fill = purple] coordinates {
(1.6747,1) (1.7405,2) (1.7303,3) (1.6039,4) (1.6948,5)};
\addplot [draw = violet, fill = violet] coordinates {
(1.5847,1) (1.6270,2) (1.6285,3) (1.5738,4) (1.6531,5)};
\addplot [draw = orange, fill = orange] coordinates {
(0.4819,1) (0.1509,2) (0.1096,3) (0.0925,4) (0.0617,5)};
\end{axis}
\end{tikzpicture}
        \caption{$f_{19}$}
    \end{subfigure} 
    \hspace{0.005cm}
    \begin{subfigure}{0.16\textwidth}
        \centering
        \begin{tikzpicture}
\begin{axis} [
    width=0.99\textwidth,
    height = 11cm, 
    xbar = .1cm,
    bar width = 1pt,
    xmin = 0,
    xmax = 22,
    ytick = {1,2,3,4,5},
    yticklabels = {20,40,60,80,100},
         ymax=5.25, ymin=0.75,
    label style={font=\scriptsize},
    ticklabel style = {font=\scriptsize},
    ylabel={\#Epochs},
    xlabel={SD},
    enlarge y limits  = 0.05,
    enlarge x limits  = 0.05
]

\addplot [draw = red, fill = red] coordinates {
(1.9480,1) (1.9260,2) (1.9267,3) (1.9159,4) (1.9722,5)};
\addplot [draw = green, fill = green] coordinates {
(11.3530,1) (11.2260,2) (11.8253,3) (11.8063,4) (11.7965,5)};
\addplot [draw = pink, fill = pink] coordinates {
(1.2942,1) (1.1297,2) (1.9288,3) (1.4233,4) (1.9288,5)};
\addplot [draw = blue, fill = blue] coordinates {
(10.0608,1) (10.9626,2) (11.4137,3) (11.1678,4) (11.4679,5)};
\addplot [draw = yellow, fill = yellow] coordinates {
(0,1) (0,2) (0,3) (0,4) (0,5)};
\addplot [draw = brown, fill = brown] coordinates {
(13.7117,1) (13.1818,2) (13.2864,3) (13.6438,4) (13.5856,5)};
\addplot [draw = cyan, fill = cyan] coordinates {
(1.7366,1) (1.7881,2) (1.7337,3) (1.7796,4) (1.8100,5)};
\addplot [draw = lime, fill = lime] coordinates {
(1.5551,1) (1.5567,2) (1.5076,3) (1.5163,4) (1.5147,5)};
\addplot [draw = purple, fill = purple] coordinates {
(14.7775,1) (16.2841,2) (16.6279,3) (16.7032,4) (17.0013,5)};
\addplot [draw = violet, fill = violet] coordinates {
(20.5834,1) (19.8947,2) (21.2215,3) (19.7229,4) (20.3580,5)};
\addplot [draw = orange, fill = orange] coordinates {
(1.5759,1) (0.3670,2) (0.3690,3) (0.2478,4) (0.1429,5)};
\end{axis}
\end{tikzpicture}
        \caption{$f_{20}$}
    \end{subfigure}   
\hspace{0.005cm}
    \begin{subfigure}{0.16\textwidth}
        \centering
        \begin{tikzpicture}
\begin{axis} [
    width=0.99\textwidth,
    height = 11cm, 
    xbar = .1cm,
    bar width = 1pt,
    xmin = 0,
    xmax = 0.14,
    ytick = {1,2,3,4,5},
    yticklabels = {20,40,60,80,100},
    xticklabels = {0, 0.1, 0.2, 0.3},
         ymax=5.25, ymin=0.75,
    label style={font=\scriptsize},
    ticklabel style = {font=\scriptsize},
    ylabel={\#Epochs},
    xlabel={SD},
    enlarge y limits  = 0.05,
    enlarge x limits  = 0.05
]

\addplot [draw = red, fill = red] coordinates {
(0.0860,1) (0.0896,2) (0.0928,3) (0.0907,4) (0.0877,5)};
\addplot [draw = green, fill = green] coordinates {
(0.0001,1) (0.0005,2) (0.0020,3) (0.0004,4) (0.0012,5)};
\addplot [draw = pink, fill = pink] coordinates {
(0.0043,1) (0.0068,2) (0.0082,3) (0.0135,4) (0.0082,5)};
\addplot [draw = blue, fill = blue] coordinates {
(0.0272,1) (0.0247,2) (0.0251,3) (0.0257,4) (0.0278,5)};
\addplot [draw = yellow, fill = yellow] coordinates {
(0.1314,1) (0.1188,2) (0.1244,3) (0.1243,4) (0.1203,5)};
\addplot [draw = brown, fill = brown] coordinates {
(0.0872,1) (0.0904,2) (0.0919,3) (0.0914,4) (0.0889,5)};
\addplot [draw = cyan, fill = cyan] coordinates {
(0.0606,1) (0.0658,2) (0.0702,3) (0.0673,4) (0.0681,5)};
\addplot [draw = lime, fill = lime] coordinates {
(0,1) (0,2) (0,3) (0,4) (0.0001,5)};
\addplot [draw = purple, fill = purple] coordinates {
(0.0012,1) (0.0011,2) (0.0013,3) (0.0016,4) (0.0014,5)};
\addplot [draw = violet, fill = violet] coordinates {
(0,1) (0,2) (0,3) (0,4) (0,5)};
\addplot [draw = orange, fill = orange] coordinates {
(0.0037,1) (0.0005,2) (0.0007,3) (0.0004,4) (0.0003,5)};
\end{axis}
\end{tikzpicture}
        \caption{$f_{21}$}
    \end{subfigure} \hspace{0.005cm}
    \begin{subfigure}{0.16\textwidth}
        \centering
        \begin{tikzpicture}
\begin{axis} [
    width=0.99\textwidth,
    height = 11cm, 
    xbar = .1cm,
    bar width = 1pt,
    xmin = 0,
    xmax = 50,
    xtick = {0,15,30,45},
    ytick = {1,2,3,4,5},
    yticklabels = {20,40,60,80,100},
         ymax=5.25, ymin=0.75,
    label style={font=\scriptsize},
    ticklabel style = {font=\scriptsize},
    ylabel={\#Epochs},
    xlabel={SD},
    enlarge y limits  = 0.05,
    enlarge x limits  = 0.05
]

\addplot [draw = red, fill = red] coordinates {
(27.8368,1) (26.5909,2) (27.6362,3) (28.5756,4) (28.6930,5)};
\addplot [draw = green, fill = green] coordinates {
(0.0722,1) (0.0716,2) (0.0711,3) (0.0729,4) (0.0700,5)};
\addplot [draw = pink, fill = pink] coordinates {
(25.9323,1) (28.3536,2) (29.6242,3) (29.2831,4) (29.6242,5)};
\addplot [draw = blue, fill = blue] coordinates {
(0.0350,1) (0.0365,2) (0.0380,3) (0.0380,4) (0.0392,5)};
\addplot [draw = yellow, fill = yellow] coordinates {
(9.7795,1) (11.4212,2) (12.9108,3) (18.3885,4) (11.2817,5)};
\addplot [draw = brown, fill = brown] coordinates {
(0.0331,1) (0.0380,2) (0.0388,3) (0.0317,4) (0.0396,5)};
\addplot [draw = cyan, fill = cyan] coordinates {
(0.0543,1) (0.0586,2) (0.0532,3) (0.0538,4) (0.0553,5)};
\addplot [draw = lime, fill = lime] coordinates {
(0.0412,1) (0.0366,2) (0.0303,3) (0.0357,4) (0.0350,5)};
\addplot [draw = purple, fill = purple] coordinates {
(0.0317,1) (0.0400,2) (0.0386,3) (0.0344,4) (0.0382,5)};
\addplot [draw = violet, fill = violet] coordinates {
(34.1434,1) (39.7354,2) (36.9538,3) (37.7540,4) (34.6266,5)};
\addplot [draw = orange, fill = orange] coordinates {
(2.1063,1) (1.4518,2) (1.3184,3) (0.5520,4) (1.3522,5)};
\end{axis}
\end{tikzpicture}
        \caption{$f_{22}$}
    \end{subfigure} 
    \hspace{0.005cm}
    \begin{subfigure}{0.16\textwidth}
        \centering
        \begin{tikzpicture}
\begin{axis} [
    width=0.99\textwidth,
    height = 11cm, 
    xbar = .1cm,
    bar width = 1pt,
    xmin = 0,
    xmax = 1200000000,
    ytick = {1,2,3,4,5},
    yticklabels = {20,40,60,80,100},
         ymax=5.25, ymin=0.75,
    label style={font=\scriptsize},
    ticklabel style = {font=\scriptsize},
    ylabel={\#Epochs},
    xlabel={SD},
    enlarge y limits  = 0.05,
    enlarge x limits  = 0.05,
    legend style={at={(1.05,1)},anchor=north west, legend cell align=left,font=\scriptsize},
]

\addplot [draw = red, fill = red] coordinates {
(74.5830,1) (75.1199,2) (76.0310,3) (76.1230,4) (75.5424,5)};
\addplot [draw = green, fill = green] coordinates {
 (71.9747,1) (75.7740,2) (76.6180,3) (75.7690,4) (77.2051,5)};
\addplot [draw = pink, fill = pink] coordinates {
 (133.7068,1) (134.5897,2) (135.1893,3) (136.6079,4) (135.1893,5)};
\addplot [draw = blue, fill = blue] coordinates {
 (336754.3313,1) (356850.0807,2) (274552.9271,3) (539259.2591,4) (411751.5527,5)};
\addplot [draw = yellow, fill = yellow] coordinates {
(132.0801,1) (136.6226,2) (141.5513,3) (138.6817,4) (141.7158,5)};
\addplot [draw = brown, fill = brown] coordinates {
(11926583.4939,1) (30414511.5957,2) (26843676.2475,3) (20984621.4115,4) (20057007.1203,5)};
\addplot [draw = cyan, fill = cyan] coordinates {
 (63.1271,1) (68.6875,2) (68.2987,3) (65.0139,4) (66.3528,5)};
\addplot [draw = lime, fill = lime] coordinates {
 (1127956.3659,1) (674733.9953,2) (897220.0731,3) (1955556.1119,4) (1338347.2802,5)};
\addplot [draw = purple, fill = purple] coordinates {
 (1217000103.0404,1) (923636362.2260,2) (954384184.3777,3) (994277798.4440,4) (996972572.3398,5)};
\addplot [draw = violet, fill = violet] coordinates {
 (95.6448,1) (90.4096,2) (85.2201,3) (90.3013,4) (90.6321,5)};
\addplot [draw = orange, fill = orange] coordinates {
(80.9868,1) (13.9484,2) (42.8730,3) (7.7512,4) (7.4657,5)};
\end{axis}
\end{tikzpicture}
        \caption{$f_{23}$}
    \end{subfigure} 
    \caption{Comparative analysis of average standard deviation across 23 benchmark functions over different epochs against state-of-the-art methods}
    \label{fig:functions_1_23}
\end{figure*}
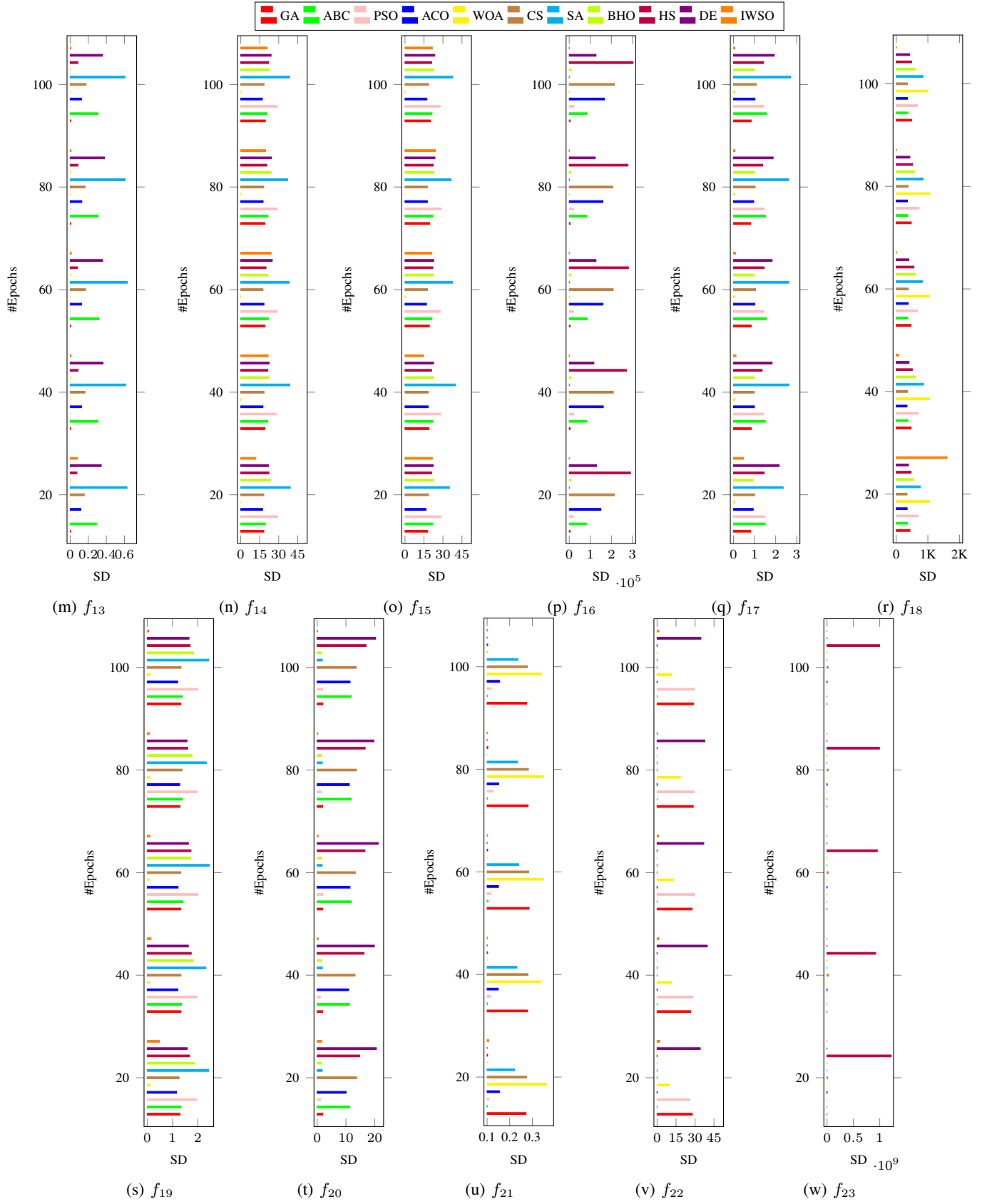

%% file: res2.tex
\begin{figure*}[!htbp]
	\centering
	\begin{subfigure}{0.48\textwidth}
		\centering
	\begin{tikzpicture}
		\begin{axis}[
			ymajorgrids,
			xmajorgrids,
			zmajorgrids,
			enlargelimits,
			z buffer=sort,
			width=0.99\textwidth,
            height=5.80cm,
			ylabel={\scriptsize Iterations},xmin=8,
			ticklabel style = {font=\scriptsize},
			xlabel={\scriptsize Population size}, ymin=0,zmax=22, 
            ytick={0,25,50,75,100},
			zlabel={\scriptsize Fitness value},zmin=-0.09,
			ylabel style = {rotate=40},
			xlabel style = {rotate=-10},
			enlargelimits = false,
			scatter/use mapped color={
				draw=mapped color,
				fill=mapped color!70,
			},legend style={at={(1.13,1.2)},  
				anchor=north,legend columns=5},
			ycomb,
			]
			
			\addplot3 [
			mark=cube*, mark size=5, fill=blue,
			] table [x=X, y=Y, z=Z, meta=Meta] {
				X       Y       Z       Meta
				10	25	19.2767	GA
10	50	19.017	GA
10	75	18.9177	GA
10	100	18.9087	GA
20	25	18.6325	GA
20	50	18.2618	GA
20	75	18.1253	GA
20	100	18.1918	GA
30	25	18.2263	GA
30	50	18.0084	GA
30	75	17.85	GA
30	100	17.7598	GA
40	25	17.8977	GA
40	50	17.906	GA
40	75	17.7314	GA
40	100	17.4709	GA
50	25	17.7227	GA
50	50	17.3998	GA
50	75	17.5039	GA
50	100	17.3288	GA
				
			};
			\addlegendentry{\scriptsize GA}
			
			
			\addplot3 [
			mark=triangle*, mark size=5, fill =red,
			] table [x=X, y=Y, z=Z, meta=Meta] {
				X       Y       Z       Meta
				10	25	16.5301	DE
10	50	16.1981	DE
10	75	15.6812	DE
10	100	15.6582	DE
20	25	18.3386	DE
20	50	17.7509	DE
20	75	18.0598	DE
20	100	17.5305	DE
30	25	18.954	DE
30	50	18.8377	DE
30	75	18.5594	DE
30	100	18.5319	DE
40	25	19.5327	DE
40	50	19.3649	DE
40	75	18.9083	DE
40	100	18.6126	DE
50	25	19.4221	DE
50	50	19.4125	DE
50	75	19.1207	DE
50	100	19.0343	DE
			};
			\addlegendentry{\scriptsize DE}
			
			\addplot3 [
			mark=otimes*, mark size=4, fill =brown,
			] table [x=X, y=Y, z=Z, meta=Meta] {
				X       Y       Z       Meta
				10	25	3.157013862	IWSO
10	50	2.982709838	IWSO
10	75	3.119322469	IWSO
10	100	2.913474181	IWSO
20	25	2.859083923	IWSO
20	50	2.936836433	IWSO
20	75	2.839158078	IWSO
20	100	2.847859755	IWSO
30	25	2.700003459	IWSO
30	50	2.751601733	IWSO
30	75	2.783660658	IWSO
30	100	2.757635259	IWSO
40	25	2.699336796	IWSO
40	50	2.687763055	IWSO
40	75	2.678757785	IWSO
40	100	2.567977452	IWSO
50	25	2.408546614	IWSO
50	50	2.516631345	IWSO
50	75	2.498064074	IWSO
50	100	2.570822697	IWSO
			};
			\addlegendentry{\scriptsize \textbf{IWSO}}
			
		\end{axis}
	\end{tikzpicture}
	\caption{$f_1$}
	\label{fig:com1}
\end{subfigure}
	\begin{subfigure}{0.48\textwidth}
	\centering
		\begin{tikzpicture}
		\begin{axis}[
			ymajorgrids,
			xmajorgrids,
			zmajorgrids,
			enlargelimits,
			z buffer=sort,
			width=0.99\textwidth,   height=5.80cm,
			ylabel={\scriptsize Iterations},xmin=8,
			xlabel={\scriptsize Population size}, ymin=20,zmax=0.05, 
            ytick={0,25,50,75,100},
			zlabel={\scriptsize Fitness value},zmin=-0.0,
			ylabel style = {rotate=40},
			xlabel style = {rotate=-10},
			ticklabel style = {font=\scriptsize},
			enlargelimits = false,
			scatter/use mapped color={
				draw=mapped color,
				fill=mapped color!70,
			},legend style={at={(0.45,-0.25)}, 
				anchor=north,legend columns=5},
			ycomb,
			]
			
			\addplot3 [
			mark=cube*, mark size=5, fill=blue,
			] table [x=X, y=Y, z=Z, meta=Meta] {
				X       Y       Z       Meta
				10	25	0.0771	GA
10	50	0.0452	GA
10	75	0.0213	GA
10	100	0.0229	GA
20	25	0.0207	GA
20	50	0.0073	GA
20	75	0.0073	GA
20	100	0.0043	GA
30	25	0.0102	GA
30	50	0.0042	GA
30	75	0.0023	GA
30	100	0.0024	GA
40	25	0.004	GA
40	50	0.0027	GA
40	75	0.0023	GA
40	100	0.0015	GA
50	25	0.0042	GA
50	50	0.0022	GA
50	75	0.0008	GA
50	100	0.0009	GA
			};
			
			\addplot3 [
			mark=triangle*, mark size=5, fill =red,
			] table [x=X, y=Y, z=Z, meta=Meta]  {
				X       Y       Z       Meta
				10	25	0	DE
10	50	0	DE
10	75	0	DE
10	100	0	DE
20	25	0	DE
20	50	0	DE
20	75	0	DE
20	100	0	DE
30	25	0	DE
30	50	0	DE
30	75	0	DE
30	100	0	DE
40	25	0	DE
40	50	0	DE
40	75	0	DE
40	100	0	DE
50	25	0	DE
50	50	0	DE
50	75	0	DE
50	100	0	DE
			};

			\addplot3 [
			mark=otimes*, mark size=4, fill =brown,
			] table [x=X, y=Y, z=Z, meta=Meta] {
				X       Y       Z       Meta
			10	25	0.016465528	IWSO
10	50	0.011294028	IWSO
10	75	0.009071462	IWSO
10	100	0.012092163	IWSO
20	25	0.004678459	IWSO
20	50	0.003495784	IWSO
20	75	0.004562084	IWSO
20	100	0.005936569	IWSO
30	25	0.003526931	IWSO
30	50	0.00370814	IWSO
30	75	0.00429274	IWSO
30	100	0.003269128	IWSO
40	25	0.002102635	IWSO
40	50	0.002843612	IWSO
40	75	0.002147883	IWSO
40	100	0.002316639	IWSO
50	25	0.002278373	IWSO
50	50	0.0024336	IWSO
50	75	0.001903027	IWSO
50	100	0.001397137	IWSO
				
			};
			
		\end{axis}
	\end{tikzpicture}
	\caption{$f_3$}
	\label{fig:com2}
\end{subfigure}
\\
	\begin{subfigure}{0.48\textwidth}
	\centering
		\begin{tikzpicture}
		\begin{axis}[
			ymajorgrids,
			xmajorgrids,
			zmajorgrids,
			enlargelimits,
			z buffer=sort,
			width=0.99\textwidth,   height=5.80cm,
			ylabel={\scriptsize Iterations},xmin=8,
			xlabel={\scriptsize Population size}, ymin=20,zmax=-150, 
            ytick={0,25,50,75,100},
			zlabel={\scriptsize Fitness value},zmin=-720,
			ylabel style = {rotate=40},
			xlabel style = {rotate=-10},
			ticklabel style = {font=\scriptsize},
			enlargelimits = false,
			scatter/use mapped color={
				draw=mapped color,
				fill=mapped color!70,
			},legend style={at={(0.45,-0.25)}, 
				anchor=north,legend columns=5},
			ycomb,
			]
			
			\addplot3 [
			mark=cube*, mark size=5, fill=blue,
			] table [x=X, y=Y, z=Z, meta=Meta] {
				X       Y       Z       Meta
				10	25	-194.985	GA
10	50	-208.1046	GA
10	75	-219.1785	GA
10	100	-218.8381	GA
20	25	-220.9085	GA
20	50	-226.3327	GA
20	75	-238.3772	GA
20	100	-238.2481	GA
30	25	-218.3166	GA
30	50	-231.081	GA
30	75	-242.8515	GA
30	100	-246.0245	GA
40	25	-232.956	GA
40	50	-239.5993	GA
40	75	-251.8622	GA
40	100	-254.9204	GA
50	25	-236.022	GA
50	50	-241.3332	GA
50	75	-250.0813	GA
50	100	-253.1728	GA
			};
			
			\addplot3 [
			mark=triangle*, mark size=5, fill =red,
			] table [x=X, y=Y, z=Z, meta=Meta]  {
				X       Y       Z       Meta
				10	25	-203.0301	DE
10	50	-213.4775	DE
10	75	-210.6464	DE
10	100	-210.4467	DE
20	25	-196.3378	DE
20	50	-191.6831	DE
20	75	-198.855	DE
20	100	-205.0312	DE
30	25	-200.04	DE
30	50	-205.1606	DE
30	75	-221.4265	DE
30	100	-216.8854	DE
40	25	-196.8079	DE
40	50	-211.8337	DE
40	75	-206.6055	DE
40	100	-234.1054	DE
50	25	-196.6367	DE
50	50	-211.9165	DE
50	75	-213.5696	DE
50	100	-226.4196	DE
			};

			\addplot3 [
			mark=otimes*, mark size=4, fill =brown,
			] table [x=X, y=Y, z=Z, meta=Meta] {
				X       Y       Z       Meta
				10	25	-563.9517716	IWSO
10	50	-593.6294	IWSO
10	75	-494.9265295	IWSO
10	100	-625.1263254	IWSO
20	25	-641.219199	IWSO
20	50	-637.1487084	IWSO
20	75	-606.882036	IWSO
20	100	-592.4829455	IWSO
30	25	-625.6392903	IWSO
30	50	-697.9253274	IWSO
30	75	-697.4317263	IWSO
30	100	-713.5063328	IWSO
40	25	-703.334374	IWSO
40	50	-607.1289551	IWSO
40	75	-656.9691792	IWSO
40	100	-674.4365961	IWSO
50	25	-685.4612	IWSO
50	50	-635.5409149	IWSO
50	75	-716.2777159	IWSO
50	100	-717.3745822	IWSO

			};
			
		\end{axis}
	\end{tikzpicture}
	\caption{$f_4$}
	\label{fig:com3}
\end{subfigure}
	\begin{subfigure}{0.48\textwidth}
	\centering
		\begin{tikzpicture}
		\begin{axis}[
			ymajorgrids,
			xmajorgrids,
			zmajorgrids,
			enlargelimits,
			z buffer=sort,
			width=0.99\textwidth,   height=5.80cm,
			ylabel={\scriptsize Iterations},xmin=8,
			xlabel={\scriptsize Population size}, ymin=20,zmax=290, 
            ytick={0,25,50,75,100},
			zlabel={\scriptsize Fitness value},zmin=0.01,
			ylabel style = {rotate=40},
			xlabel style = {rotate=-10},
			ticklabel style = {font=\scriptsize},
			enlargelimits = false,
			scatter/use mapped color={
				draw=mapped color,
				fill=mapped color!70,
			},legend style={at={(0.45,-0.25)}, 
				anchor=north,legend columns=5},
			ycomb,
			]
			
			\addplot3 [
			mark=cube*, mark size=5, fill=blue,
			] table [x=X, y=Y, z=Z, meta=Meta] {
				X       Y       Z       Meta
				10	25	282.7939	GA
10	50	263.8458	GA
10	75	237.8971	GA
10	100	246.1506	GA
20	25	215.3472	GA
20	50	210.5649	GA
20	75	194.7993	GA
20	100	195.4328	GA
30	25	188.327	GA
30	50	180.349	GA
30	75	170.7243	GA
30	100	167.0954	GA
40	25	170.3804	GA
40	50	162.7693	GA
40	75	154.1228	GA
40	100	149.3655	GA
50	25	154.5293	GA
50	50	151.0582	GA
50	75	149.3381	GA
50	100	143.3126	GA
			};
			
			\addplot3 [
			mark=triangle*, mark size=5, fill =red,
			] table [x=X, y=Y, z=Z, meta=Meta]  {
				X       Y       Z       Meta
				10	25	128.8215	DE
10	50	109.5522	DE
10	75	100.6463	DE
10	100	93.5461	DE
20	25	189.2068	DE
20	50	179.2102	DE
20	75	153.9279	DE
20	100	163.9465	DE
30	25	217.0743	DE
30	50	238.0764	DE
30	75	207.39	DE
30	100	202.717	DE
40	25	258.0257	DE
40	50	246.9948	DE
40	75	233.2449	DE
40	100	221.622	DE
50	25	322.6214	DE
50	50	286.9472	DE
50	75	263.7205	DE
50	100	268.596	DE
			};

			\addplot3 [
			mark=otimes*, mark size=4, fill =brown,
			] table [x=X, y=Y, z=Z, meta=Meta] {
				X       Y       Z       Meta
				10	25	0.738120903	IWSO
10	50	0.794296992	IWSO
10	75	0.774486851	IWSO
10	100	0.712595105	IWSO
20	25	0.52135972	IWSO
20	50	0.563625357	IWSO
20	75	0.545780602	IWSO
20	100	0.565566816	IWSO
30	25	0.442191439	IWSO
30	50	0.491460907	IWSO
30	75	0.478679404	IWSO
30	100	0.447608147	IWSO
40	25	0.430802321	IWSO
40	50	0.377387246	IWSO
40	75	0.396941467	IWSO
40	100	0.408224714	IWSO
50	25	0.391591376	IWSO
50	50	0.359023131	IWSO
50	75	0.369207682	IWSO
50	100	0.331154956	IWSO
				
			};
			
		\end{axis}
	\end{tikzpicture}
	\caption{$f_7$}
	\label{fig:com4}
\end{subfigure}
\\
	\begin{subfigure}{0.48\textwidth}
	\centering
		\begin{tikzpicture}
		\begin{axis}[
			ymajorgrids,
			xmajorgrids,
			zmajorgrids,
			enlargelimits,
			z buffer=sort,
			width=0.99\textwidth,   height=5.80cm,
			ylabel={\scriptsize Iterations},xmin=8,
			xlabel={\scriptsize Population size}, ymin=20,zmax=125, 
            ytick={0,25,50,75,100},
			zlabel={\scriptsize Fitness value},zmin=0.01,
			ylabel style = {rotate=40},
			xlabel style = {rotate=-10},
			ticklabel style = {font=\scriptsize},
			enlargelimits = false,
			scatter/use mapped color={
				draw=mapped color,
				fill=mapped color!70,
			},legend style={at={(0.45,-0.25)}, 
				anchor=north,legend columns=5},
			ycomb,
			]
			
			\addplot3 [
			mark=cube*, mark size=5, fill=blue,
			] table [x=X, y=Y, z=Z, meta=Meta] {
				X       Y       Z       Meta
			10	25	79.4288	GA
10	50	75.468	GA
10	75	73.5563	GA
10	100	67.4238	GA
20	25	64.5801	GA
20	50	58.2181	GA
20	75	60.646	GA
20	100	55.1736	GA
30	25	57.1171	GA
30	50	53.3635	GA
30	75	50.6936	GA
30	100	51.0576	GA
40	25	50.4979	GA
40	50	47.9693	GA
40	75	49.4402	GA
40	100	45.8872	GA
50	25	48.2944	GA
50	50	45.9515	GA
50	75	42.6397	GA
50	100	43.5089	GA
			};
			
			\addplot3 [
			mark=triangle*, mark size=5, fill =red,
			] table [x=X, y=Y, z=Z, meta=Meta]  {
				X       Y       Z       Meta
				10	25	49.1931	DE
10	50	49.4557	DE
10	75	44.7706	DE
10	100	42.1686	DE
20	25	89.9332	DE
20	50	74.1287	DE
20	75	70.0476	DE
20	100	66.0406	DE
30	25	110.4467	DE
30	50	106.6577	DE
30	75	88.442	DE
30	100	79.2251	DE
40	25	122.3465	DE
40	50	97.3933	DE
40	75	93.6542	DE
40	100	85.5869	DE
50	25	110.1627	DE
50	50	110.5231	DE
50	75	110.6242	DE
50	100	103.3288	DE
			};

			\addplot3 [
			mark=otimes*, mark size=4, fill =brown,
			] table [x=X, y=Y, z=Z, meta=Meta] {
				X       Y       Z       Meta
				10	25	1.070966459	IWSO
10	50	1.012514215	IWSO
10	75	1.024909433	IWSO
10	100	1.044465963	IWSO
20	25	0.869255592	IWSO
20	50	0.848335318	IWSO
20	75	0.859236597	IWSO
20	100	0.804160012	IWSO
30	25	0.834022294	IWSO
30	50	0.776527846	IWSO
30	75	0.743149815	IWSO
30	100	0.789580519	IWSO
40	25	0.774116699	IWSO
40	50	0.727097526	IWSO
40	75	0.691737789	IWSO
40	100	0.73655779	IWSO
50	25	0.715694897	IWSO
50	50	0.667186753	IWSO
50	75	0.704266564	IWSO
50	100	0.710325639	IWSO
				
			};
			
		\end{axis}
	\end{tikzpicture}
	\caption{$f_8$}
	\label{fig:com5}
\end{subfigure}
	\begin{subfigure}{0.48\textwidth}
	\centering
		\begin{tikzpicture}
		\begin{axis}[
			ymajorgrids,
			xmajorgrids,
			zmajorgrids,
			enlargelimits,
			z buffer=sort,
			width=0.99\textwidth,   height=5.80cm,
			ylabel={\scriptsize Iterations},xmin=8,
			xlabel={\scriptsize Population size}, ymin=20,zmax=20, 
            ytick={0,25,50,75,100},
			zlabel={\scriptsize Fitness value},zmin=0.01,
			ylabel style = {rotate=40},
			xlabel style = {rotate=-10},
			ticklabel style = {font=\scriptsize},
			enlargelimits = false,
			scatter/use mapped color={
				draw=mapped color,
				fill=mapped color!70,
			},legend style={at={(0.45,-0.25)}, 
				anchor=north,legend columns=5},
			ycomb,
			]
			
			\addplot3 [
			mark=cube*, mark size=5, fill=blue,
			] table [x=X, y=Y, z=Z, meta=Meta] {
				X       Y       Z       Meta
				10	25	18.7903	GA
10	50	17.9958	GA
10	75	17.9134	GA
10	100	17.5254	GA
20	25	16.3511	GA
20	50	15.9065	GA
20	75	15.4159	GA
20	100	15.3215	GA
30	25	15.226	GA
30	50	14.7231	GA
30	75	14.4396	GA
30	100	14.4321	GA
40	25	14.649	GA
40	50	14.0984	GA
40	75	13.7254	GA
40	100	13.5798	GA
50	25	14.1201	GA
50	50	13.6666	GA
50	75	13.4657	GA
50	100	13.1282	GA
			};
			
			\addplot3 [
			mark=triangle*, mark size=5, fill =red,
			] table [x=X, y=Y, z=Z, meta=Meta]  {
				X       Y       Z       Meta
				10	25	11.8963	DE
10	50	11.1836	DE
10	75	11.1363	DE
10	100	11.1826	DE
20	25	14.6034	DE
20	50	15.1168	DE
20	75	14.8331	DE
20	100	66.0406	DE
30	25	15.9213	DE
30	50	16.6882	DE
30	75	16.254	DE
30	100	16.7071	DE
40	25	18.4794	DE
40	50	17.7594	DE
40	75	17.7121	DE
40	100	16.3142	DE
50	25	18.4819	DE
50	50	17.8973	DE
50	75	18.1029	DE
50	100	17.2993	DE
			};

			\addplot3 [
			mark=otimes*, mark size=4, fill =brown,
			] table [x=X, y=Y, z=Z, meta=Meta] {
				X       Y       Z       Meta
				10	25	0.513039901	IWSO
10	50	0.485213265	IWSO
10	75	0.444227012	IWSO
10	100	0.489799543	IWSO
20	25	0.412021874	IWSO
20	50	0.388195291	IWSO
20	75	0.390391066	IWSO
20	100	0.411301592	IWSO
30	25	0.371873456	IWSO
30	50	0.345637682	IWSO
30	75	0.320293956	IWSO
30	100	0.372817668	IWSO
40	25	0.331141033	IWSO
40	50	0.347121992	IWSO
40	75	0.321808701	IWSO
40	100	0.329689283	IWSO
50	25	0.307874679	IWSO
50	50	0.30289558	IWSO
50	75	0.307030405	IWSO
50	100	0.32508726	IWSO
				
			};
			
		\end{axis}
	\end{tikzpicture}
	\caption{$f_{17}$}
	\label{fig:com6}
    \end{subfigure}
    \\
    \begin{subfigure}{0.48\textwidth}
	\centering
		\begin{tikzpicture}
		\begin{axis}[
			ymajorgrids,
			xmajorgrids,
			zmajorgrids,
			enlargelimits,
			z buffer=sort,
			width=0.99\textwidth,   height=5.80cm,
			ylabel={\scriptsize Iterations},xmin=8,
			xlabel={\scriptsize Population size}, ymin=20,zmax=0.001, 
            ytick={0,25,50,75,100},
			zlabel={\scriptsize Fitness value},zmin=0.0,
			ylabel style = {rotate=40},
			xlabel style = {rotate=-10},
			ticklabel style = {font=\scriptsize},
			enlargelimits = false,
			scatter/use mapped color={
				draw=mapped color,
				fill=mapped color!70,
			},legend style={at={(0.45,-0.25)}, 
				anchor=north,legend columns=5},
			ycomb,
			]
			
			\addplot3 [
			mark=cube*, mark size=5, fill=blue,
			] table [x=X, y=Y, z=Z, meta=Meta] {
				X       Y       Z       Meta
				10	25	0.0041	GA
10	50	0.0027	GA
10	75	0.0019	GA
10	100	0.0013	GA
20	25	0.0012	GA
20	50	0.0008	GA
20	75	0.0003	GA
20	100	0.0004	GA
30	25	0.0007	GA
30	50	0.0004	GA
30	75	0.0002	GA
30	100	0.0001	GA
40	25	0.0004	GA
40	50	0.0002	GA
40	75	0.0001	GA
40	100	0.0001	GA
50	25	0.0003	GA
50	50	0.0001	GA
50	75	0.0001	GA
50	100	0	GA
			};
			
			\addplot3 [
			mark=triangle*, mark size=5, fill =red,
			] table [x=X, y=Y, z=Z, meta=Meta]  {
				X       Y       Z       Meta
				10	25	0	DE
10	50	0	DE
10	75	0	DE
10	100	0	DE
20	25	0	DE
20	50	0	DE
20	75	0	DE
20	100	0	DE
30	25	0	DE
30	50	0	DE
30	75	0	DE
30	100	0	DE
40	25	0	DE
40	50	0	DE
40	75	0	DE
40	100	0	DE
50	25	0	DE
50	50	0	DE
50	75	0	DE
50	100	0	DE
			};

			\addplot3 [
			mark=otimes*, mark size=4, fill =brown,
			] table [x=X, y=Y, z=Z, meta=Meta] {
				X       Y       Z       Meta
				10	25	0.002209834	IWSO
10	50	0.002443576	IWSO
10	75	0.002550318	IWSO
10	100	0.002794823	IWSO
20	25	0.000910849	IWSO
20	50	0.001517574	IWSO
20	75	0.000986119	IWSO
20	100	0.00089212	IWSO
30	25	0.000636935	IWSO
30	50	0.000397095	IWSO
30	75	0.000843384	IWSO
30	100	0.000652255	IWSO
40	25	0.000522155	IWSO
40	50	0.000446151	IWSO
40	75	0.000595806	IWSO
40	100	0.000471162	IWSO
50	25	0.000438884	IWSO
50	50	0.000474623	IWSO
50	75	0.00039135	IWSO
50	100	0.000394391	IWSO
				
			};
			
		\end{axis}
	\end{tikzpicture}
	\caption{$f_{21}$}
	\label{fig:com7}
\end{subfigure}
	\begin{subfigure}{0.48\textwidth}
	\centering
		\begin{tikzpicture}
		\begin{axis}[
			ymajorgrids,
			xmajorgrids,
			zmajorgrids,
			enlargelimits,
			z buffer=sort,
			width=0.99\textwidth,   height=5.80cm,
			ylabel={\scriptsize Iterations},xmin=8,
			xlabel={\scriptsize Population size}, ymin=20,zmax=420, 
            ytick={0,25,50,75,100},
			zlabel={\scriptsize Fitness value},zmin=5,
			ylabel style = {rotate=40},
			xlabel style = {rotate=-10},
			ticklabel style = {font=\scriptsize},
			enlargelimits = false,
			scatter/use mapped color={
				draw=mapped color,
				fill=mapped color!70,
			},legend style={at={(0.45,-0.25)}, 
				anchor=north,legend columns=5},
			ycomb,
			]
			
			\addplot3 [
			mark=cube*, mark size=5, fill=blue,
			] table [x=X, y=Y, z=Z, meta=Meta] {
				X       Y       Z       Meta
				10	25	328.3972	GA
10	50	312.0366	GA
10	75	293.1254	GA
10	100	301.6437	GA
20	25	305.2999	GA
20	50	277.8456	GA
20	75	270.8984	GA
20	100	264.6326	GA
30	25	290.0233	GA
30	50	264.0263	GA
30	75	253.595	GA
30	100	252.9422	GA
40	25	268.1885	GA
40	50	249.1855	GA
40	75	243.154	GA
40	100	238.7996	GA
50	25	264.2104	GA
50	50	258.3144	GA
50	75	227.6527	GA
50	100	239.0404	GA
			};
			
			\addplot3 [
			mark=triangle*, mark size=5, fill =red,
			] table [x=X, y=Y, z=Z, meta=Meta]  {
				X       Y       Z       Meta
				10	25	312.5142	DE
10	50	354.2445	DE
10	75	267.5453	DE
10	100	274.5277	DE
20	25	390.9273	DE
20	50	316.5515	DE
20	75	318.6889	DE
20	100	312.4787	DE
30	25	385.778	DE
30	50	336.8286	DE
30	75	370.9337	DE
30	100	335.2135	DE
40	25	405.5684	DE
40	50	359.3426	DE
40	75	360.6773	DE
40	100	321.3344	DE
50	25	413.263	DE
50	50	398.8999	DE
50	75	373.5772	DE
50	100	373.6174	DE
			};

			\addplot3 [
			mark=otimes*, mark size=4, fill =brown,
			] table [x=X, y=Y, z=Z, meta=Meta] {
				X       Y       Z       Meta
				10	25	44.20529575	IWSO
10	50	45.95314672	IWSO
10	75	33.40328902	IWSO
10	100	47.13614362	IWSO
20	25	17.2984031	IWSO
20	50	16.6781936	IWSO
20	75	19.30439174	IWSO
20	100	20.52854778	IWSO
30	25	12.33819653	IWSO
30	50	9.218480935	IWSO
30	75	10.63701763	IWSO
30	100	11.59896982	IWSO
40	25	9.119480038	IWSO
40	50	7.503438472	IWSO
40	75	9.851431669	IWSO
40	100	7.942859124	IWSO
50	25	7.182982738	IWSO
50	50	6.910463504	IWSO
50	75	7.49950594	IWSO
50	100	8.124688584	IWSO
				
			};
			
		\end{axis}
	\end{tikzpicture}
	\caption{$f_{23}$}
	\label{fig:com8}
\end{subfigure}
\caption{Comparative analysis of average fitness across benchmark functions against state-of-the-art methods} \label{fig:com}
\end{figure*}
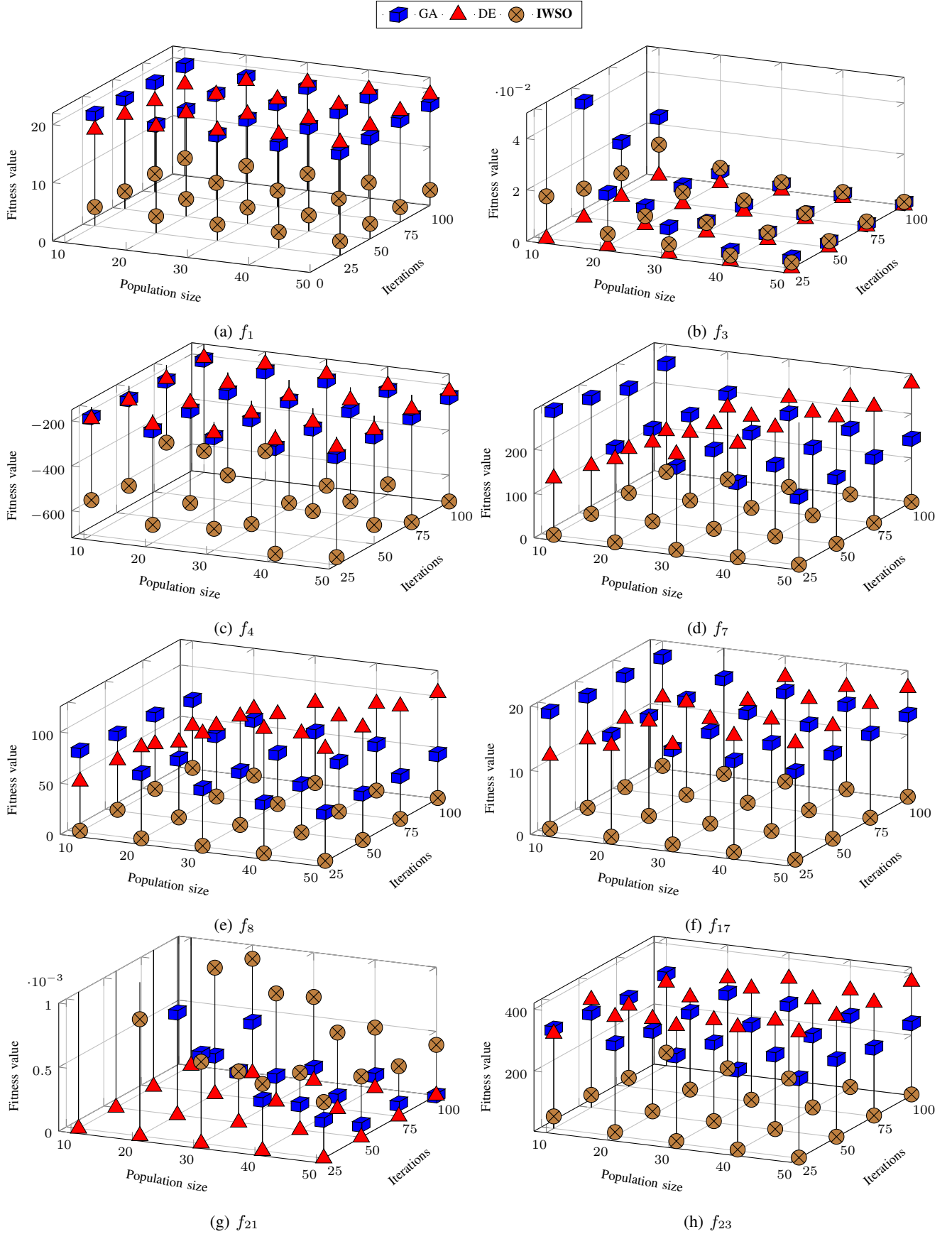 